%% file: a_paper.tex
\documentclass[11pt]{article}

\usepackage[letterpaper,margin=1in]{geometry}
\usepackage{times}

\input{math_commands.tex}

\usepackage{amsmath,amssymb,amsthm}
\usepackage{graphicx}
\usepackage{booktabs}
\usepackage[colorlinks=true,linkcolor=black,citecolor=black,urlcolor=blue]{hyperref}
\usepackage{url}
\usepackage{cleveref}
\usepackage{microtype}
\usepackage{xcolor}
\usepackage{enumitem}
\usepackage{float}
\usepackage{adjustbox}
\usepackage{natbib}
\usepackage{authblk}

\newtheorem{proposition}{Proposition}
\newtheorem{theorem}{Theorem}
\newtheorem{remark}{Remark}

\newcommand{\lammax}{\lambda_{\max}}

\newcommand{\betac}{\beta_c}
\newcommand{\betagmm}{\beta_{\mathrm{GMM}}}

\title{Feature Lottery? A Bifurcation Theory of Concept Emergence}

\author{
  {\large \textbf{Fuming Yang}}\thanks{Correspondence to: \texttt{fumingy@mit.edu}} \\
  {\small\color[rgb]{0.3, 0.35, 0.45} MIT}
}

\date{}
\begin{document}
\maketitle

\begin{abstract}
Neural networks learn structured representations at specific points during training, yet pinpointing these transitions usually requires post-hoc, label-dependent metrics. We introduce a bifurcation theory of representation dynamics to detect these moments in real time. By analyzing a passive GMM probe attached to an evolving encoder, we show that the emergence of structure corresponds to a supercritical pitchfork bifurcation driven by the loss Hessian. The system exhibits a theoretically predictable zero-crossing ($\beta_c$), which---when compared against the network's current state ($\beta$)---yields a dynamic ratio, $\beta(t)/\beta_c(t)$. Computed solely from hidden states, this ratio acts as a universal, label-free phase coordinate for tracking representation dynamics.

We empirically validate the four transition regimes predicted by this coordinate across diverse settings: sparse autoencoders (SAEs) on language models (Pythia), self-supervised learning (CIFAR), and grokking (modular arithmetic). Crucially, under finite dissipation, macroscopic symmetry breaking can lag behind the initial zero-crossing by orders of magnitude---providing a rigorous dynamical explanation for the delayed generalization observed in grokking. Microscopically, the theory predicts that the bifurcation creates a shared unstable subspace that forces collective symmetry breaking. We call this the \emph{feature lottery} in SAE training: a feature's eventual interpretability is determined very early on. Just $5\%$ into training, early atom purity strongly predicts final purity at convergence---atoms in the top $10\%$ at the $5\%$ mark reach a POS purity of $0.82$ by the end of training, compared to $0.47$ for the bottom $10\%$ and $0.22$ for random chance.

Beyond explaining how concepts emerge, $\beta/\beta_c$ is a practical diagnostic tool. It serves as an early-warning indicator of training health, capturing the onset of usable structure, the crystallization of feature identities, and the beginning of representation collapse many epochs before downstream metrics react.

\smallskip
\noindent Interactive demo: \url{https://fumingyang-felix.github.io/feature-lottery-demo/}
\end{abstract}

\section{Introduction}\label{sec:intro}
Modern networks do not merely fit labels or reconstruct inputs;
during training, their internal states reorganize into discrete,
reusable directions that behave like concepts. Yet we typically
notice this reorganization only after the fact: by probing with
labels, by inspecting downstream accuracy, or by mechanistic
analysis of a fully trained model. What is missing is a
\emph{label-free dynamical signal} of when such concept structure
first becomes available, a quantity that (watched live during
training) tells us whether and when a network has just acquired a
usable internal representation, and ideally \emph{which} parts of
the representation are about to become meaningful.

\paragraph{Our angle.}
We provide such a signal, derived from a single Hessian analysis.
Attach a passive $K$-prototype isotropic Gaussian-mixture (GMM) head
with shared learned precision $\beta$ to the output
$z\!=\!\mathrm{enc}(x)$ of a given encoder representation, and analyze the Hessian of
its negative log-likelihood at the \emph{symmetric collapsed state}
$\mu_1\!=\!\cdots\!=\!\mu_K\!=\!\bar z$. The analysis
(Section~\ref{sec:theory}) yields a critical precision
\begin{equation}
\boxed{\;\betac \;=\; \frac{1}{\lammax(\Cov(z))}\;}
\label{eq:betac}
\end{equation}
at which the lowest eigenvalue of the loss Hessian crosses zero;
above $\betac$ the prototypes pitchfork along the principal
eigenvector of $\Cov(z)$. We adopt this Hessian-pitchfork event as our \emph{operational}
definition of concept emergence: the moment at which the encoder's
representation first admits a class-aligned $K$-prototype
decomposition. The label-free indicator is
$\beta(t)/\betac(t)$, computable from the encoder's hidden state and
the GMM probe alone. When the encoder is itself learning, $\betac(t)$
becomes endogenous and we prove $\beta(t)$ and $\betac(t)$ must
cross at some finite training time
(Proposition~\ref{prop:endogenous}). A subtlety, made precise in
Remark~\ref{rem:metastability}, is that the crossing event marks when
the symmetric state becomes \emph{unstable}, not when the
broken-symmetry state becomes \emph{macroscopically observable}. The
lag between the two is controlled by the encoder's dissipation, and
can range from essentially zero (in well-trained SSL) to thousands of
steps (in grokking).

\paragraph{The sharpest prediction is per-atom.}
At the crossing event, the unstable subspace at $\beta\!=\!\betac$ is
shared by all $K\!-\!1$ anti-symmetric modes
(App.~\ref{app:hessian:eig}). The bifurcation is therefore a
\emph{collective} event with a \emph{per-atom signature}: each atom
must select a specific direction from a common unstable manifold,
its choice driven by initialization noise and the cubic terms of the
pitchfork normal form. This per-atom prediction is empirically
sharp. In SAE training on frozen Pythia-160M layer~6, per-atom
POS purity is at noise floor before the bifurcation and acquires
predictive content at onset; ranking atoms by their step-$1{,}000$
POS purity ($5\%$ of training) already recovers a top decile whose
convergence purity is $0.82$, against $0.47$ for the bottom decile
and a chance level of $0.22$ (Section~\ref{sec:sae-lottery}). We
refer to this as the SAE
\emph{feature lottery}, the bifurcation theory's sharpest and most
unexpected empirical consequence; it is the SAE-level analogue of
the lottery-ticket framing of \citet{frankle2018lottery}, with the
drawing event identified explicitly as the first phase transition
during training.

\paragraph{Contributions.}
\begin{enumerate}[leftmargin=*,itemsep=2pt,topsep=2pt]
\item \textbf{Theory: an endogenous critical point with post-critical
  metastability.} A Hessian-pitchfork prediction of when a given
  encoder representation first admits a $K$-prototype decomposition,
  obtained from a passive GMM probe on the encoder output, with
  an existence proof for the endogenous critical point
  (Proposition~\ref{prop:endogenous}; the crossing $\beta\!=\!\betac$
  happens at a finite time \emph{conditional on} the encoder
  eventually spreading the latent enough for the GMM to resolve
  clusters) and a separate prediction of post-critical metastability
  under finite dissipation (Remark~\ref{rem:metastability}). Both
  are new relative to the static soft-$K$-means criticality of
  \citet{rose1990statistical}: Rose gives the critical temperature
  on a frozen dataset; we give the dynamic crossing theorem when
  $\betac(t)$ co-evolves with the encoder, and identify a
  post-critical metastable regime that the static analysis cannot
  exhibit.
\item \textbf{The theory's sharpest empirical consequence: an SAE
  feature lottery at $5\%$ of training.} At the crossing event the
  unstable subspace is shared by all $K\!-\!1$ anti-symmetric modes,
  so atoms must select directions from a common manifold during the
  bifurcation. We verify this in SAE training on frozen
  language-model activations (Sec.~\ref{sec:sae-lottery}): per-atom
  POS purity is at noise floor pre-onset
  ($\rho_{\mathrm{id}}\!\approx\!0.03$), and by step $1{,}000$ ($5\%$
  of training), identity-matched
  $\rho_{\mathrm{id}}\!=\!+0.41\!\pm\!0.04$ (3 seeds, all
  $p\!<\!10^{-80}$). The top decile of atoms ranked at $5\%$ achieves
  convergence POS purity $0.82\!\pm\!0.03$, against
  $0.47\!\pm\!0.03$ for the bottom decile and a chance level of
  $0.22$.
  The effect replicates across soft-L1 and architectural top-$K$
  SAEs, and across $K\!\in\!\{256,\dots,8192\}$ with
  $\rho_{\mathrm{id}} \in [0.26, 0.41]$. This is an atom-level,
  post-onset analogue of the lottery-ticket framing of
  \citet{frankle2018lottery}.
\item \textbf{Empirical universality of the bifurcation arc.} The
  predicted trajectory in $(\log(\beta/\betac),\,\log\mathrm{NC1})$
  is governed by three binary kinematic axes: initial
  sub/supercriticality, the post-onset race between $\beta(t)$ and
  $\betac(t)$, and the dissipation rate
  (Sec.~\ref{sec:arc:synthesis}, App.~\ref{app:five-shapes}). The
  axes predict which kinematic regimes are accessible to which
  feature-learning methods. We verify the four regimes that arise
  in standard pipelines; full V (SAE on frozen Pythia layer~6),
  fold-back (DINO/SimCLR on CIFAR-10/100, with magnitude controlled
  by data complexity), delayed escape (grokking on modular
  arithmetic, with WD-monotonic escape time
  $\tau_{\mathrm{esc}} \propto \mathrm{WD}^{-1.23}$ across six WD
  levels, and $0/3$ escape at WD$=\!0$), and no arc
  (rotation-prediction control).
\item \textbf{K-sweep: lottery is $K$-stable, POS-purity Pareto is
  $K$-monotonic but $K$-confounded.} At fixed $3\%$ top-$K$ sparsity,
  the lottery $\rho_{\mathrm{id}}$ is stable across $K \in \{256,
  \dots, 8192\}$ (range $[+0.26, +0.41]$). Per-atom POS purity
  decreases monotonically with $K$ from $0.725$ to $0.370$
  (Tab.~\ref{tab:ksweep}), while
  reconstruction MSE moves the other way ($0.11 \to 0.015$). Because
  POS has only $15$ classes, this purity-vs-$K$ trend is partly
  structural: small-$K$ atoms each correspond to coarser token-cluster
  partitions that more easily align with the $15$-way POS partition.
  We therefore do \emph{not} recommend small $K$ on the basis of POS
  purity alone; a $K$-unbiased interpretability metric (e.g.\ causal
  mediation or LLM-as-judge) is required to make a substantive
  recommendation (Sec.~\ref{sec:sae-lottery:ksweep}).
\item \textbf{A label-free training diagnostic.} $\beta/\betac$
  identifies the encoder's current act from hidden states alone, well
  before downstream metrics respond. In grokking, the indicator at
  step $100$ already places the trajectory in Act~2:
  $\sim$8{,}400 steps before $\mathrm{test\_acc}$ moves. In DINO
  from-scratch with collapse modes (Sec.~\ref{sec:diag:fromscratch}),
  $\beta/\betac$ leads cluster accuracy by 8 epochs in the gradual
  mode; in mid-training interventions
  (Sec.~\ref{sec:diag:intervention}) it responds within a single
  batch while training loss remains within noise.
\end{enumerate}

\section{Related work}\label{sec:related}

\paragraph{Concept emergence and mechanistic interpretability.}
A growing body of work in mechanistic interpretability studies how
internal computations of trained models implement particular
``concepts.''
\citet{nanda2023progress} show that the post-grok solution of
modular addition is a Fourier multiplication algorithm distributed
across the network's embedding, attention, and MLP layers; sparse
autoencoders on frozen
language-model activations~\citep{bricken2023sparse,gao2025scaling,
templeton2024scaling} extract interpretable directions in feature
space. These works identify concepts \emph{post-hoc} on a trained
model: features are evaluated for interpretability after training
converges, and the typical assumption is that training longer yields
better features. Our framework is complementary along two axes.
First, we provide a label-free \emph{dynamical} signal of when such
structure first becomes available during training
(Sec.~\ref{sec:arc}). Second, we show that in SAEs the bifurcation
onset is a per-atom assignment event whose outcome
predicts atom-level interpretability nineteen thousand training
steps in advance (Sec.~\ref{sec:sae-lottery}). This reframes SAE
feature emergence as a structured lottery rather than a gradual
refinement.

\paragraph{Neural collapse.}
The neural-collapse literature~\citep{papyan2020prevalence} and the
subsequent Unconstrained Features Model
analyses~\citep{mixon2022neural,tirer2022extended,zhou2022all,
sukenik2024neural} characterize the static terminal-phase geometry of
supervised classification. \citet{wang2023understanding} recover an
analogous structure in supervised contrastive learning via an
information-bottleneck argument. These analyses describe the
endpoint, not the dynamics by which an encoder reaches it; our
framework supplies the missing dynamics and predicts the timing of
concept emergence without labels.

\paragraph{Deterministic annealing and rate-distortion clustering.}
\citet{rose1990statistical} obtained the critical temperature
$T_c=2\lammax(\Sigma)$ for soft $K$-means by externally annealing
$T$. In our convention $\beta\!=\!2/T$, this is the external-$\beta$
limit of our analysis. The novelty here is that $\beta$ is endogenous
and $\betac(t)$ moves with the encoder.

\paragraph{Self-supervised collapse.}
\citet{jing2021understanding,hua2021feature} analyze dimensional
collapse in contrastive and non-contrastive self-supervised
methods.
\citet{caron2021emerging} introduces DINO with centering and
sharpening regularizers specifically to prevent collapse. Our
diagnostic experiments (Sec.~\ref{sec:diagnostic}) take DINO collapse
modes as a controlled testbed.

\paragraph{Grokking and emergence.}
\citet{power2022grokking} discovered that small transformers on
modular arithmetic exhibit a long memorization plateau followed by a
sudden generalization transition; \citet{nanda2023progress}
mechanistically identified the in-network DFT circuit responsible.
The plateau-then-sudden-transition pattern is reminiscent of the
saddle-to-saddle dynamics characterized in deep linear networks by
\citet{saxe2014exact}, where learning proceeds through a sequence of
loss-landscape saddles separated by long plateaus. Our framework
gives this picture a concrete representation-geometric content: the
plateau is the metastable post-critical regime
($\beta\!>\!\betac$ but $\varepsilon$ still microscopic;
Remark~\ref{rem:metastability}), and the dissipation strength sets
the escape time. We revisit grokking in Sec.~\ref{sec:arc:grokking}
as the cleanest empirical instance of post-critical metastability
in our framework.

\paragraph{Phase transitions in deep learning.}
\citet{ziyin2022collapse} analyzes posterior collapse in linear
latent-variable models, and \citet{liu2023phase} prove first- and
second-order phase transitions in deep linear networks under a
statistical-mechanics framing. Both connect collapse-type phenomena
to phase transitions in regularized learning. We share this
perspective and contribute a closed-form Hessian-based predictor of
the transition point in unsupervised and self-supervised settings.

\paragraph{Lottery-ticket framing (weak analogy).}
\citet{frankle2018lottery} introduce the lottery-ticket hypothesis
for supervised classifiers: there exist sparse subnetworks
identifiable already at initialization that train to the same
accuracy as the full network. Our finding for SAE atoms
(Sec.~\ref{sec:sae-lottery}) is a \emph{weaker analogue} in feature
space: early per-atom POS purity at the bifurcation onset is a useful
ranking signal for converged per-atom POS purity (identity-matched
Spearman $\rho = +0.41 \pm 0.04$ between step-$1{,}000$ and
step-$20{,}000$). The analogy is weaker than \citet{frankle2018lottery}
in two important senses: (1) identification happens at the post-onset
bifurcation event rather than at initialization, and (2) achievement
requires continued joint training, not isolated retraining of a
sparse subnetwork. We use the lottery framing for its spirit
(\emph{what matters is decided early}), not as a structural
isomorphism.

\section{Theory}\label{sec:theory}
\subsection{Setup}
We attach a $K$-component isotropic GMM with shared precision $\beta$ and
component means $\mu_k\in\mathbb{R}^d$ to the encoder output
$z = \mathrm{enc}(x)$:
\begin{equation}
p(z\mid\theta) \;=\; \frac{1}{K} \sum_{k=1}^{K}
   \mathcal{N}(z \mid \mu_k, \beta^{-1} I),
\qquad
\mathcal{L}_\mu(\mu,\beta)
   \;=\; -\mathbb{E}_z\!\Bigl[\,\mathrm{LSE}_k\!\bigl(-\tfrac{\beta}{2}\|z-\mu_k\|^2\bigr)\Bigr],
\end{equation}
up to terms independent of $\mu$. We denote the data covariance
$\Sigma\!=\!\Cov(z)$ with eigenvalues
$\sigma_1^2\!\ge\!\cdots\!\ge\!\sigma_d^2$, the responsibilities
$p(k\!\mid\!z) = \mathrm{softmax}_k(-\tfrac{\beta}{2}\|z-\mu_k\|^2)$, and
the symmetric collapsed state $\mathcal{S}_0 = \{\mu_k = \bar z\}$.

\subsection{Hessian at the symmetric state}
At $\mathcal{S}_0$ we have $p(k\!\mid\!z) = 1/K$ and
$\nabla_\mu \mathcal{L}_\mu = 0$. Differentiating once more and evaluating
at $\mathcal{S}_0$,
\begin{equation}
\Bigl.\frac{\partial^2 \mathcal{L}_\mu}{\partial \mu_k^a\,\partial\mu_l^b}\Bigr|_{\mathcal{S}_0}
  \;=\;
  \tfrac{\beta}{K}\,\delta_{kl}\,\delta^{ab}
  \;-\; \tfrac{\beta^2}{K}\bigl(\delta_{kl} - \tfrac{1}{K}\bigr)\,\Sigma^{ab}.
\label{eq:hessian}
\end{equation}
Eigenvectors decompose into separable perturbations $\xi_l^b = w_l u^b$ with
$w\!\in\!\mathbb{R}^K$ and $u\!\in\!\mathbb{R}^d$. Two channels arise: the
symmetric channel ($w$ constant) is always stable with eigenvalues $\beta/K$;
the anti-symmetric channel ($\sum_l w_l = 0$, $(K{-}1)$-fold degenerate)
has spatial eigenvalues
\begin{equation}
\lambda^\perp_i(\beta) \;=\; \tfrac{\beta}{K}\,\bigl(1 - \beta\,\sigma_i^2\bigr),
\qquad i=1,\dots,d.
\end{equation}
The lowest such eigenvalue is $\lambda^\perp_1(\beta) = (\beta/K)(1 - \beta\,\lammax(\Sigma))$,
and crosses zero exactly at the critical precision in equation~\eqref{eq:betac}.
The unstable direction is the principal eigenvector of $\Sigma$ in spatial
space combined with any anti-symmetric $w$ in component space. Projecting
the dynamics onto this slow direction yields the supercritical pitchfork
normal form
\begin{equation}
\dot\varepsilon \;=\; (\beta - \betac)\,\varepsilon \;-\; \alpha\,\varepsilon^3 \;+\; \text{noise},
\quad \alpha>0,
\label{eq:pitchfork}
\end{equation}
shown in Fig.~\ref{fig:theory}(a). A full derivation, including the
parametrization correction relative to \citet{rose1990statistical} and
the explicit eigendecomposition in component space
(Theorem~\ref{thm:hessian}), is in Appendix~\ref{app:hessian}
(\S\ref{app:hessian:explicit}).

\begin{figure}[t]
\centering
\includegraphics[width=\linewidth]{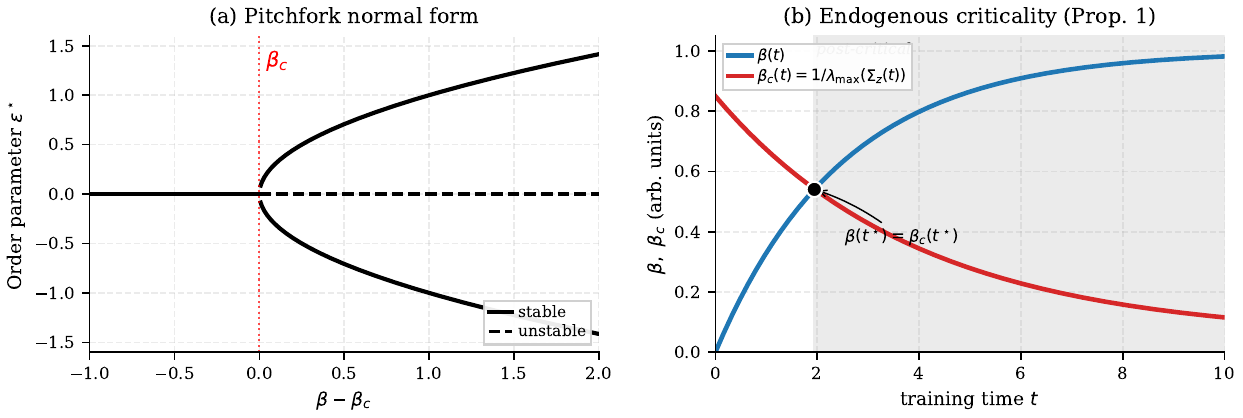}
\caption{Theory. \textbf{(a)} The supercritical pitchfork at $\betac$: below
$\betac$ the symmetric collapsed state is stable; above $\betac$ it becomes
a saddle and two stable broken-symmetry branches emerge. \textbf{(b)}
Endogenous criticality (Proposition~\ref{prop:endogenous}): when the encoder
is itself learning, $\betac(t)$ co-evolves with $\Cov(z(t))$; under mild
monotonicity assumptions $\beta(t)$ and $\betac(t)$ must cross at some
finite time $t^\star$.}
\label{fig:theory}
\end{figure}

\subsection{Hierarchy and reverse traversal}
The analysis recurses within each post-primary supercluster, giving a
sequence of secondary critical precisions $\betac^{(2)} =
1/\lammax(\Sigma_{\mathrm{within}})$. The pitchfork
\eqref{eq:pitchfork} is also reversible: decreasing $\beta$ continuously
merges the broken-symmetry branches back into $\mathcal{S}_0$.
App.~\ref{app:toy} confirms both on toy data (hierarchy to four decimals;
forward overshoot $\beta^\star/\betac\!\approx\!1.3$, reverse tracking
within $\le 4\%$) and on CIFAR-scale encoders (App.~\ref{app:reverse}).

\subsection{Endogenous criticality}
Replace the fixed dataset $\{z_n\}$ with $\{z_n(t)\}=\{\mathrm{enc}(x_n;\varphi(t))\}$,
where the encoder parameters $\varphi$ evolve under an upstream loss. The
critical precision becomes time-dependent:
\begin{equation}
\betac(t) \;=\; \frac{1}{\lammax(\Cov(z(t)))},
\end{equation}
and the bifurcation occurs at the moment $\beta(t^\star) = \betac(t^\star)$.

\begin{proposition}[Endogenous critical point]\label{prop:endogenous}
Suppose
\begin{enumerate}[leftmargin=*,itemsep=1pt,topsep=2pt]
\item $\beta(t)$ is asymptotically non-decreasing and $\liminf_{t\to\infty}
  \beta(t) > c_1 > 0$;
\item $\betac(t)$ is asymptotically non-increasing on average;
\item $\limsup_{t\to\infty} \betac(t) < c_1$.
\end{enumerate}
Then $\beta(t)$ and $\betac(t)$ cross at some finite $t^\star$; at the
crossing, the GMM's symmetric state becomes unstable.
\end{proposition}
The proof, given in Appendix~\ref{app:proof}, is essentially a continuity
argument: $\beta(t)-\betac(t)$ goes from $-|\beta(0)\!-\!\betac(0)|<0$
(at $t\!=\!0$, before learning) to a strictly positive lim-inf, and must
cross zero. The substantive content is in the hypotheses: (1) holds for any
likelihood-maximizing GMM step, and (2) is the assertion that the encoder
spreads the latent, which is the defining property of any
information-preserving feature-learning objective.
The novelty relative to \citet{rose1990statistical} is not the IVT step
but the framing: $\betac$ becomes a \emph{time-dependent} observable
that co-evolves with the encoder, turning the crossing event into a
predictable training-time phenomenon rather than a static property of
a frozen dataset.

\subsection{Post-critical metastability: the bifurcation timing
problem}\label{sec:metastable}

Proposition~\ref{prop:endogenous} establishes that the crossing event
$\beta(t^\star)\!=\!\betac(t^\star)$ exists; the symmetric state
$\mathcal{S}_0$ becomes a saddle at $t^\star$. The proposition is
silent, however, on \emph{when the broken-symmetry state becomes
macroscopically observable}. The order parameter $\varepsilon$
introduced in~\eqref{eq:pitchfork} obeys
\begin{equation}
\dot\varepsilon \;=\; (\beta - \betac)\,\varepsilon
                   \;-\; \alpha\,\varepsilon^3
                   \;+\; \eta(t),
\label{eq:eps-ode}
\end{equation}
where $\eta(t)$ is a noise/dissipation term inherited from the
encoder's training dynamics. Just past the crossing, the unstable
mode's linear growth rate is $(\beta - \betac) \!\to\! 0^+$, so
$\varepsilon$ grows exponentially \emph{from the noise scale} on a
characteristic timescale that scales as $1/(\beta-\betac)$.

\begin{remark}[Post-critical metastability]\label{rem:metastability}
The crossing event of Prop.~\ref{prop:endogenous} marks when
$\mathcal{S}_0$ becomes unstable, not when the system reaches the
broken-symmetry state. The lag between the two depends on the
linear growth rate $(\beta-\betac)$ and the dissipation supplied by
the encoder's training dynamics. For continuously-driven objectives
(e.g.\ contrastive losses), the lag is short and the post-critical
descent in $(\log(\beta/\betac),\log\mathrm{NC1})$ is observed
immediately after the crossing. For under-dissipated objectives
(e.g.\ supervised cross-entropy without explicit weight decay), the
lag can extend over orders of magnitude in training steps,
producing a long \emph{metastable plateau} on which
$\beta\!>\!\betac$ but $\mathrm{NC1}$ has not yet collapsed.
Section~\ref{sec:arc:grokking} presents an extreme empirical instance
in which the crossing occurs within $\sim 40$ training steps but
the macroscopic broken-symmetry transition is delayed by thousands of
steps. Weight decay (the encoder's most familiar dissipation knob)
supplies the drift that ultimately triggers the transition, with the
plateau length monotonically decreasing in WD and diverging as
WD$\to 0$ at fixed noise scale. A control experiment with a 6-point
WD sweep (Sec.~\ref{sec:arc:grokking}, App.~\ref{app:escape})
quantifies this empirically and shows the metastable plateau is a
post-critical escape phenomenon, not a pre-bifurcation delay.
\end{remark}


\section{The bifurcation arc across feature-learning methods}\label{sec:arc}

The Hessian-pitchfork prediction of Sec.~\ref{sec:theory} produces a
three-phase trajectory in
$\bigl(\log(\beta/\betac),\,\log\mathrm{NC1}\bigr)$ space:
\emph{pre-critical} ($\beta\!<\!\betac$, no class-aligned clustering,
NC1 drifts slowly upward), \emph{critical peak}
($\beta\!\approx\!\betac$, $\mathcal{S}_0$ becomes a saddle, NC1
peaks), \emph{post-critical descent} ($\beta\!>\!\betac$, prototypes
separate and $\log\mathrm{NC1}$ falls linearly in
$\log(\beta/\betac)$). The shape \emph{observed in a particular
setup} depends on (i) where the encoder begins
($\log(\beta/\betac)\!\lessgtr\!0$ at $t\!=\!0$), (ii) the relative
rates of $\beta(t)$ and $\betac(t)$ during training, and (iii) the
dissipation rate (Remark~\ref{rem:metastability}). The combination
yields four empirically distinguishable shapes:
\emph{full V} (frozen $\betac$; one leg both ways), \emph{fold-back}
($\betac$ overtakes $\beta$ post-onset, with magnitude controlled by
data complexity; mild on CIFAR-10, strong on CIFAR-100),
\emph{delayed escape} (under-dissipated post-critical metastability,
as in grokking), and \emph{no arc} (negative control: no clustering
pressure). We verify all four in two complementary experiments below.

\subsection{Four shapes from feature-learning trajectories}\label{sec:arc:ssl}

\begin{figure}[t]
\centering
\includegraphics[width=\linewidth]{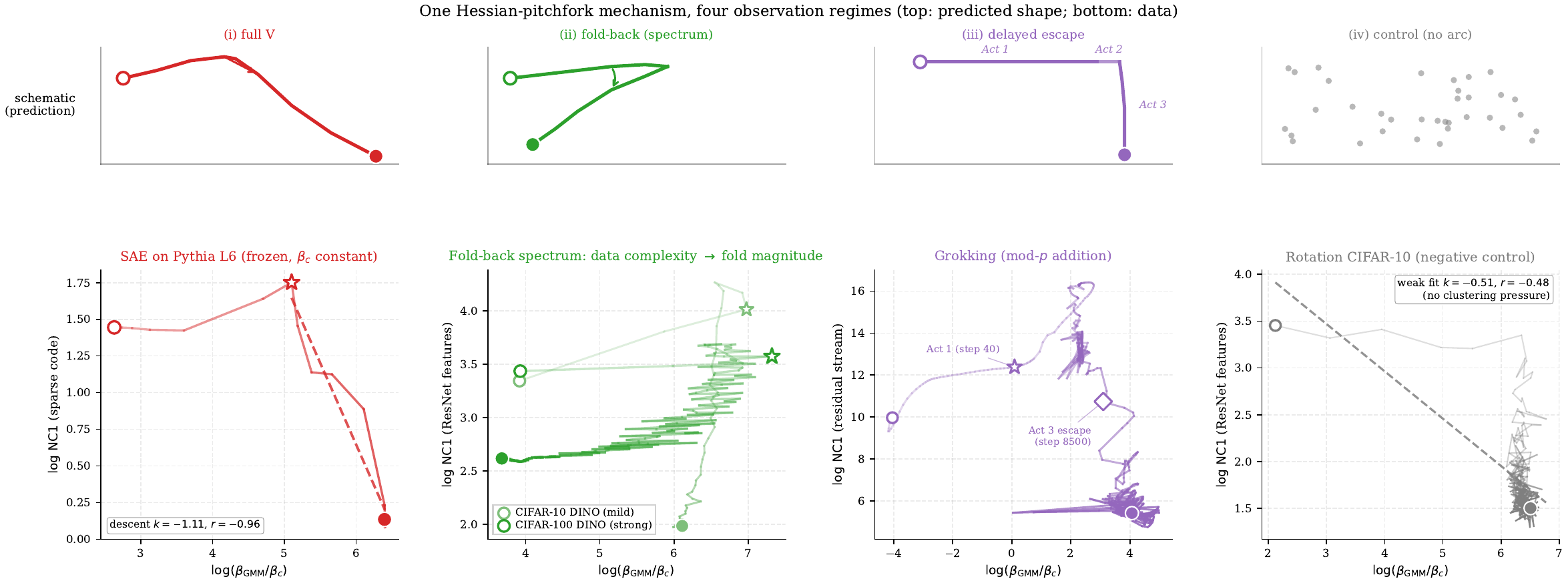}
\caption{\textbf{Four observable shapes of the bifurcation arc, from
the same Hessian-pitchfork mechanism.} Top row: schematic of each
regime; bottom row: empirical realization.
(i) \emph{Full V} on SAE / frozen Pythia layer~6
(NC1 in sparse-code basis; star = V trough).
(ii) \emph{Fold-back spectrum} on DINO / CIFAR-10 (light green, mild)
and CIFAR-100 (dark green, strong), fold magnitude controlled by
data complexity.
(iii) \emph{Delayed escape} on grokking ($p\!=\!97$, WD$=\!1.0$;
three-act structure detailed in Sec.~\ref{sec:arc:grokking}).
(iv) \emph{No arc} on rotation-prediction (negative control).
NC1 axes are not comparable across panels (sparse-code, backbone,
residual-stream); the trajectory \emph{shape} is the prediction.
The 3-axis taxonomy explaining why exactly these four shapes appear
is given in App.~\ref{app:five-shapes}.}
\label{fig:arc}
\end{figure}

We probe all four shapes here, across one sparse-coding-on-frozen-LM
setup, two self-supervised setups (CIFAR-10 and CIFAR-100, illustrating
the fold-back spectrum), one delayed-escape setup (grokking), and one
negative control (rotation prediction), with a shared joint-detached
GMM probe protocol ($K\!=\!10$,
$\mathrm{lr}_\mu\!=\!5\!\times\!10^{-3}$,
$\mathrm{lr}_\beta\!=\!10^{-2}$, $\log\beta_0\!=\!-2.5$;
$K_{\mathrm{probe}}$-robustness in App.~\ref{app:kprobe};
full hyperparameter tables in App.~\ref{app:hp}) so that
$\betagmm(t)$ is directly comparable across methods.
Figure~\ref{fig:arc} shows the four shapes; Table~\ref{tab:slopes}
reports the local descent slope $k$ where applicable. The grokking
panel of Fig.~\ref{fig:arc} shows a single canonical run for visual
parity with the other panels; the full three-seed analysis with the
three-act decomposition (Act 1 / Act 2 / Act 3 timing and
WD-dependent escape time, with $\tau_{\mathrm{esc}} \propto
\mathrm{WD}^{-1.23}$) is in
Sec.~\ref{sec:arc:grokking}~/~Fig.~\ref{fig:grok}.

\paragraph{(i) Full V.} The SAE on Pythia-160M layer~6 freezes the
upstream encoder, holding $\betac$ constant. The entire motion in
$\log(\beta/\betac)$ is driven by the SAE's own precision growing
past its critical point. NC1 (in the sparse-code basis,
$K\!=\!2048$) rises slightly during pre-critical buildup, peaks, and
then descends monotonically as predicted by the post-critical regime
($r\!=\!-0.97$, $n\!=\!9$).
The full V is visible because nothing competes for $\betac$. The SAE
case is also the cleanest substrate for a deeper question (whether
the bifurcation onset has \emph{atom-level} mechanistic content
beyond geometric coupling), which we take up in
Section~\ref{sec:sae-lottery}. App.~\ref{app:lm} reports the
complementary experiment of probing the LM's hidden states directly
without an intermediate SAE, showing why an SAE substrate is
essential for the full V to be visible at all.

\paragraph{(ii) Fold-back (spectrum).} When the encoder is itself
learning a complex enough target, $\betac(t)$ may rise faster than
$\beta(t)$ after the initial bifurcation, and the trajectory folds
back left on the $\log(\beta/\betac)$ axis while NC1 continues to
fall. The fold magnitude is controlled by data complexity. On
CIFAR-100 (100 fine classes), both DINO and SimCLR exhibit
\emph{strong} fold-back: DINO peaks at
$\log(\beta/\betac)\!=\!+7.09$ around epoch 35 and descends back to
$+3.68$ by epoch 300 ($\sim$3 log-unit drift); SimCLR follows the
same shape on a slightly compressed range.
On CIFAR-10 (10 classes), ResNet-18 features begin already
supercritical ($\log(\beta/\betac)\!\approx\!+3.9$ at random init)
and the fold is \emph{mild} ($\sim$0.5 log-unit leftward drift
post-onset) for both DINO with native teacher-temperature $\beta_t$
and SimCLR. The two datasets occupy the same kinematic regime; data
complexity controls only the magnitude of $\betac$'s post-onset
rise.

\paragraph{(iii) Delayed escape.} Under low dissipation, the system
can sit on a post-critical metastable plateau for orders of magnitude
in training steps before the macroscopic broken-symmetry transition
fires (Remark~\ref{rem:metastability}). The grokking trace in
Fig.~\ref{fig:arc} previews this; the three-act decomposition with
multi-seed dissipation-threshold control is in
Sec.~\ref{sec:arc:grokking}.

\paragraph{(iv) No arc (negative control).} Rotation prediction
provides no clustering pressure, so the framework predicts no
bifurcation. We see only weak scatter ($r\!=\!-0.48$, $n\!=\!300$),
with NC1 evolving largely independently of
$\log(\beta/\betac)$ over a comparable range.

\begin{table}[t]
\centering
\small
\caption{Trajectory shape across the six runs of Fig.~\ref{fig:arc},
with the correlation $r$ between $\log(\beta/\betac)$ and
$\log\mathrm{NC1}$ on the descent leg (after the V trough for SAE,
after the fold-back peak for CIFAR-100 methods, on the full trajectory
for already-supercritical CIFAR-10 methods and the control). The
framework's prediction across panels is \emph{shape} and the
\emph{sign} of the post-critical relationship (negative when
$\log(\beta/\betac)$ is rising along the descent leg, positive when
it folds back). NC1 is computed in basis-specific normalizations
(sparse code, ResNet features, residual stream), and the sample sizes
$n$ differ across runs, so neither $|r|$ nor the implicit slope
magnitude is comparable across panels in a strict sense; we report
$r$ only to characterize the tightness of the relationship within
each panel.}
\label{tab:slopes}
\begin{adjustbox}{max width=\textwidth}
\begin{tabular}{lllrr}
\toprule
\textbf{Method} & \textbf{Dataset} & \textbf{Shape} & \textbf{descent $r$} & \textbf{$n$} \\
\midrule
SAE on Pythia L6 ($K\!=\!2048$, soft-L1) & wikitext-103   & (i) full V              & $-0.97$ & 9   \\
DINO C-100              & CIFAR-100      & (ii) fold-back (strong) & $+0.90$ & 278 \\
SimCLR C-100            & CIFAR-100      & (ii) fold-back (strong) & $+0.99$ & 291 \\
DINO C-10 ($\beta_t$)   & CIFAR-10       & (ii) fold-back (mild)   & $+0.96$ & 50  \\
SimCLR C-10             & CIFAR-10       & (ii) fold-back (mild)   & $+0.74$ & 600 \\
Rotation C-10           & CIFAR-10       & (iv) no arc (control)   & $-0.48$ & 300 \\
\bottomrule
\end{tabular}
\end{adjustbox}
\end{table}

\paragraph{Sign of the post-critical relationship is kinematic.} The
relationship between $\log(\beta/\betac)$ and $\log\mathrm{NC1}$ on
the descent leg is negative when $\log(\beta/\betac)$ is rising on
that leg (regime i, frozen $\betac$) and positive when it is falling
or saturating (regime ii, fold-back). The framework does not predict
a universal sign; it predicts an arc whose descent leg may be
traversed in either direction in $\log(\beta/\betac)$ space depending
on which of $\beta$ or $\betac$ moves faster.

\subsection{Grokking: crossing, plateau, escape}\label{sec:arc:grokking}

The four SSL regimes share a feature: once $\beta$ crosses $\betac$,
the macroscopic broken-symmetry transition follows within at most a
few epochs. Remark~\ref{rem:metastability} predicts a fifth regime
that decomposes into three distinct acts in training time:
\begin{itemize}[leftmargin=*,itemsep=2pt,topsep=2pt]
\item \emph{Act~1: Crossing.} $\beta(t)$ crosses $\betac(t)$ and
  the symmetric state becomes a saddle.
\item \emph{Act~2: Metastable plateau.} The system is
  supercritical ($\beta\!>\!\betac$) but the order parameter
  $\varepsilon$ has not yet grown to a macroscopic value; NC1 sits
  on a high plateau. The plateau length is set by the dissipation
  rate.
\item \emph{Act~3: Escape.} Dissipation (here, weight decay)
  finally drives $\varepsilon$ off the saddle; NC1 collapses by
  orders of magnitude over a comparatively short window, and
  downstream test accuracy jumps from chance to $\approx 1$. This
  is the grokking transition.
\end{itemize}
The previous reading of grokking, ``the model suddenly acquires
generalization at some moment,'' is misleading; the model becomes
\emph{eligible} to acquire it within the first $\sim 40$ steps and
then spends thousands of steps trying to escape a saddle.
Grokking on modular arithmetic is the cleanest empirical realization
of this three-act picture.

\paragraph{Setup.}
Following \citet{nanda2023progress}, we train a 1-layer transformer
($d_{\mathrm{model}}\!=\!128$, $4$ heads, $d_{\mathrm{mlp}}\!=\!512$)
on $a + b\bmod p$ tokens with AdamW, $\mathrm{lr}\!=\!10^{-3}$,
full-batch gradient descent, $30\%$ train fraction. The probe is the
same joint-detached protocol of Sec.~\ref{sec:arc:ssl}, attached to
the residual stream at the $\mathtt{=}$ position; $\betac$ is
computed from $\Cov(z)$ on the test set, NC1 is computed against the
$p$ output classes. Three seeds per configuration.

\begin{figure}[t]
\centering
\includegraphics[width=\linewidth]{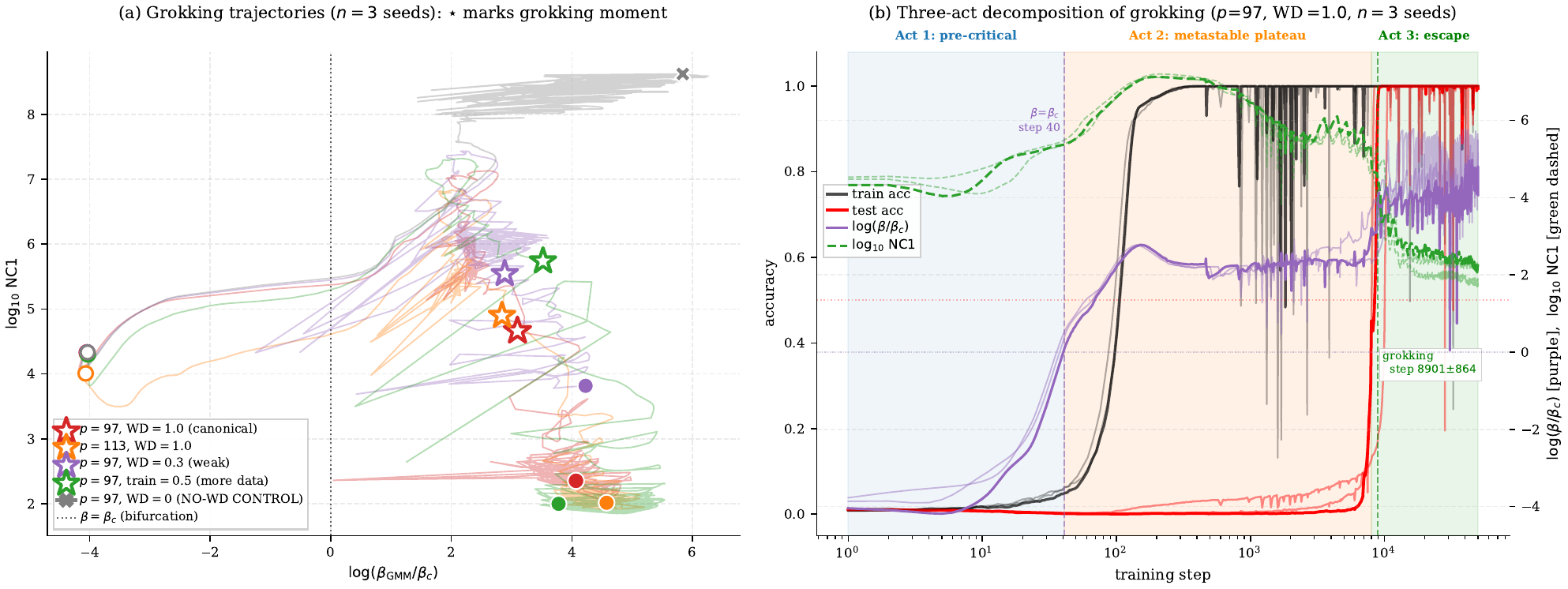}
\caption{\textbf{Grokking decomposes into crossing, plateau, and
escape.}
\textbf{(a)} Trajectories of five grokking configurations
($n\!=\!3$ seeds each, seed-0 trace shown) in
$\bigl(\log(\beta/\betac),\,\log_{10}\mathrm{NC1}\bigr)$. $\star$
marks the grokking moment ($\mathrm{test\ acc}\!>\!0.5$);
${\times}$ marks the no-grok endpoint of the WD$=$0 control.
\textbf{(b)} Canonical run ($p\!=\!97$, WD$=\!1.0$, $3$ seeds
overlaid), with the three acts highlighted as colored bands.
\emph{Act~1 (blue)}: $\beta$ rapidly crosses $\betac$ at step
$37\!\pm\!2$. \emph{Act~2 (orange)}: a long metastable plateau in
which $\log(\beta/\betac)\!\approx\!+3$ but $\log_{10}\mathrm{NC1}$
sits at $\approx 6.7$; the system is post-critical but the
broken-symmetry order parameter is still microscopic.
\emph{Act~3 (green)}: at step $8\,900\!\pm\!864$, weight decay's
dissipation finally drives the order parameter off the saddle; NC1
collapses by four orders of magnitude over a few hundred steps and
test accuracy jumps to $1$.}
\label{fig:grok}
\end{figure}

\paragraph{Act~1: the crossing is fast and universal.}
Across all $15$ runs (5 configs $\times$ 3 seeds),
$\log(\beta/\betac)$ crosses zero in the first $34$--$60$ training
steps (Table~\ref{tab:grok}). The crossing event is tight
($\sigma\!\approx\!2$ steps within each configuration) and reached
in \emph{every} seed of \emph{every} configuration, including the
no-grok WD$=$0 control. This is the Hessian-pitchfork prediction of
Sec.~\ref{sec:theory} in raw form: the supercriticality condition
$\beta\!>\!\betac$ is universally met within the first few steps of
training.

\paragraph{Act~2: the plateau is the post-critical metastable
saddle.} For thousands of steps after the crossing,
$\log(\beta/\betac)$ continues to rise toward $\sim\!+4$ but $\mathrm{NC1}$
does not move from its high plateau. The previous (and incorrect)
reading would call this an undertrained representation. The
framework's reading (Remark~\ref{rem:metastability}) is sharper: the
encoder is parked near $\varepsilon\!=\!0$ on a saddle that is
\emph{linearly} unstable but whose exit time is set by noise and
dissipation. The plateau length is the dissipation-controlled escape
time, not a pre-bifurcation delay.

\paragraph{Act~3: escape is dissipation-controlled.}
At sufficiently long horizon ($200{,}000$ steps), all
configurations with WD$>\!0$ escape in $3/3$ seeds, but the escape
time is strongly monotonic in WD:
$\tau_{\mathrm{esc}}$ rises from $8\,900$ steps at WD$=\!1.0$ to
$147\,167$ steps at WD$=\!0.1$, while WD$=\!0$ fails to escape in
$0/3$ seeds within $50{,}000$ steps. The full sweep fits a
power-law $\tau_{\mathrm{esc}} \propto \mathrm{WD}^{-1.23}$
($\Delta\mathrm{AIC} = +19.3$ over the Kramers
form~\eqref{eq:kramers-derived}; Table~\ref{tab:wd-sweep},
App.~\ref{app:escape}). The WD$=$0 control reaches
$\log(\beta/\betac)\!=\!+6.07$, higher than any grokked run, yet NC1
grows to $\sim\!10^8$ and test accuracy stays at $\sim\!0.01$ in all
three seeds; the metastable plateau is a saddle that requires
dissipation to escape. Beyond the weight-decay axis, escape time also
scales predictably with the modulus $p$ and the training fraction
(Table~\ref{tab:grok}). At escape, $\log_{10}\mathrm{NC1}$ drops from
$\approx 6.7$ to $\approx 2.4$ over a few hundred training steps while
$\log(\beta/\betac)$ moves by less than $1$ log-unit; the entire
post-critical descent shape of Sec.~\ref{sec:arc:ssl} is compressed
into this narrow window.

\begin{table}[t]
\centering
\small
\caption{Grokking experiments, mean $\pm$ std across $n\!=\!3$ seeds.
Act~1 (the $\beta\!=\!\betac$ crossing) is universal and tight
($\sigma\!\approx\!2$); Act~3 escape time spans more than an order
of magnitude in WD. Full WD sweep at the $200{,}000$-step horizon is
in Table~\ref{tab:wd-sweep}; the WD$=\!0$ control fails to escape
within $50{,}000$ steps in all $3/3$ seeds.}
\label{tab:grok}
\begin{adjustbox}{max width=\textwidth}
\begin{tabular}{lrrrrr}
\toprule
\textbf{Config}                  & \textbf{Memo} & \textbf{Act~1: $\beta\!=\!\betac$} & \textbf{Act~3: grok}    & \textbf{Final $\log(\beta/\betac)$} & \textbf{$n_{\mathrm{grok}}$} \\
\midrule
$p\!=\!97$, WD$=$1.0 (canonical) & $253\!\pm\!6$   & $37\!\pm\!2$ & $8\,900\!\pm\!864$    & $+4.49\!\pm\!0.47$ & $3/3$ \\
$p\!=\!113$, WD$=$1.0            & $297\!\pm\!49$  & $39\!\pm\!2$ & $6\,633\!\pm\!1\,461$ & $+4.88\!\pm\!0.45$ & $3/3$ \\
$p\!=\!97$, WD$=$0.3 (weak)      & $263\!\pm\!2$   & $36\!\pm\!2$ & $38\,150\!\pm\!4\,250$& $+4.07\!\pm\!0.46$ & $3/3$ \\
$p\!=\!97$, train$=$0.5          & $242\!\pm\!52$  & $54\!\pm\!6$ & $1\,250\!\pm\!398$    & $+4.45\!\pm\!0.54$ & $3/3$ \\
$p\!=\!97$, WD$=$0 (control)     & $273\!\pm\!8$   & $36\!\pm\!2$ & \textbf{never}        & $+6.07\!\pm\!0.17$ & $\mathbf{0/3}$ \\
\bottomrule
\end{tabular}
\end{adjustbox}
\end{table}

\paragraph{Control experiment: WD-controlled metastable plateau length.}
The grokking setup is unusual in giving us a single, clean dissipation
knob (WD) that is otherwise neutral with respect to the
Hessian-pitchfork mechanism: changing WD does not shift Act~1
(the crossing remains at step $36\!-\!37$ across all WD; cf.\
$\beta\!=\!\betac$ column of Table~\ref{tab:grok}) and does not
change the post-critical $\log(\beta/\betac)$ trajectory shape,
only its \emph{rate}. This lets us treat the WD sweep as a control
experiment for Remark~\ref{rem:metastability}: if the metastable
plateau is genuinely a post-critical escape phenomenon, then dialing
the encoder's dissipation should monotonically change the plateau
length, and at zero dissipation the system should stay on the plateau
indefinitely. We measure $\tau_{\mathrm{esc}}$ at six WD levels
$\gamma\!\in\!\{0.1, 0.2, 0.3, 0.5, 0.7, 1.0\}$ with $n=3$ seeds per
level and a $200{,}000$-step horizon long enough to observe escape at
every $\gamma>0$ (Table~\ref{tab:wd-sweep}).

\begin{table}[h]
\centering
\small
\caption{WD-intervention control experiment: escape times at $p\!=\!97$,
train fraction $0.3$, $n\!=\!3$ seeds per WD, $200{,}000$-step horizon.
Power-law fit $\tau\propto\gamma^{-1.23}$ is within $\sim\!10\%$ at
every WD; activation-dominated Kramers form is decisively ruled out
($\Delta\mathrm{AIC}\!=\!+19.3$; see App.~\ref{app:escape}).}
\label{tab:wd-sweep}
\begin{tabular}{rrrr}
\toprule
$\gamma$ (WD) & $\tau_{\mathrm{esc}}$ (steps, observed) & $\tau_{\mathrm{esc}}$ (power-law fit) & escape rate \\
\midrule
$0.0$ & $\infty$ (no escape in $50\text{k}$) & $\infty$            & $0/3$ \\
$0.1$ & $147\,167 \pm 23\,618$                & $151\,987$           & $3/3$ \\
$0.2$ & $88\,033 \pm 27\,091$                 & $65\,006$            & $3/3$ \\
$0.3$ & $38\,150 \pm 4\,250$                  & $39\,554$            & $3/3$ \\
$0.5$ & $22\,433 \pm 6\,064$                  & $21\,152$            & $3/3$ \\
$0.7$ & $16\,633 \pm 5\,008$                  & $14\,006$            & $3/3$ \\
$1.0$ & $8\,900  \pm 864$                     & $9\,047$             & $3/3$ \\
\bottomrule
\end{tabular}
\end{table}

The result is unambiguous on the qualitative claim:
$\tau_{\mathrm{esc}}$ is monotonically decreasing in $\gamma$ across
two decades of plateau lengths ($8.9\text{k}\!\to\!147\text{k}$
steps), and the $\gamma\!=\!0$ control fails to escape in any of $3$
seeds within $50\text{k}$ steps despite reaching
$\log(\beta/\betac)\!=\!+6.07$. This directly supports
Remark~\ref{rem:metastability}: the crossing event makes the symmetric
state unstable, but the macroscopic broken-symmetry transition is a
dissipation-controlled post-critical escape, not an immediate
consequence of the crossing.

On the quantitative form, the 6-point fit prefers a power-law
$\tau\propto\gamma^{-1.23}$ over an exponential
(activation-dominated Kramers) form $\tau\propto e^{-\kappa\gamma/D}$
\emph{decisively}: $\Delta\mathrm{AIC}\!=\!+19.3$,
power-law residuals within $\sim\!10\%$ at every WD, Kramers
under-predicts the weak-dissipation escape times by
$\sim\!1.7\times$ at WD$=\!0.1$. We treat this as an empirical
characterization specific to the modular-arithmetic
grokking setup rather than a theoretical prediction of the
bifurcation framework; the regime-selection analysis
(drift-dominated vs.\ activation-dominated post-critical escape) is
in App.~\ref{app:escape}.

\paragraph{Predictive content of the indicator.}
At any training step during a grokking run,
$\log(\beta/\betac)$ identifies the system's current act:
pre-critical (Act~1, $\log(\beta/\betac)\!<\!0$), post-critical
metastable (Act~2, $\log(\beta/\betac)\!>\!0$ with NC1 on the
plateau), or escaping (Act~3, NC1 descending). At step 100 of the
canonical run, $\log(\beta/\betac)\!=\!+3$ already places the system
in Act~2: the trajectory has become \emph{eligible to grok},
$\sim 8\,400$ training steps before test accuracy provides any signal.
Combined with the dissipation strength of the encoder (a known
hyperparameter), the indicator predicts both whether grokking will
occur (no escape at WD$=\!0$, Tab.~\ref{tab:grok}) and roughly when
($\tau_{\mathrm{esc}} \propto \mathrm{WD}^{-1.23}$, spanning $8\,900$
steps at WD$=\!1.0$ to $147\,167$ at WD$=\!0.1$;
Tab.~\ref{tab:wd-sweep}). This is prediction in the operational sense:
from the indicator and the dissipation strength, both readable at step
$100$, the trajectory's downstream state is forecastable orders of
magnitude before $\mathrm{test\_acc}$, $\mathrm{train\_acc}$, or the
training loss show any sign of the transition.

\subsection{Synthesis}\label{sec:arc:synthesis}

A single Hessian-pitchfork prediction is governed by three binary
kinematic axes (initial sub/supercriticality, post-onset
$\beta$-vs-$\betac$ kinematics, dissipation rate;
Appendix~\ref{app:five-shapes}). The axes predict which kinematic
regimes are accessible to which classes of feature-learning
pipelines: standard SSL methods on rich datasets occupy the
fold-back regime, frozen-encoder SAE training traces a full V,
under-dissipated supervised settings exhibit delayed escape, and
clustering-pressure-free objectives produce no arc. We observe all
four regimes in the corresponding settings; the remaining nominal
axis combinations are either degenerate or unreached by standard
pipelines. Method-specific differences (in fold-back magnitude,
plateau length, and descent shape) are kinematic consequences of
the encoder--probe race and of the encoder's dissipation, not
methodology-specific physics. In particular, the visual difference
between CIFAR-10 (mild fold, $\sim 0.5$ log-unit leftward drift)
and CIFAR-100 (strong fold, $\sim 3$ log-units) is a magnitude
difference within the same fold-back regime, controlled by data
complexity (number of classes $\to$ richer post-critical
$\Cov(z)$ structure $\to$ larger $\betac$ rise). The quantity
$\log(\beta(t)/\betac(t))$, computed from the encoder's hidden state
and a passive GMM probe alone, is the label-free indicator of where
in the arc a given run currently sits.

The arc-level results above confirm the theory at the
\emph{trajectory} level. We now turn to its sharpest empirical
consequence at the \emph{per-atom} level: the SAE feature lottery
(Section~\ref{sec:sae-lottery}).

\section{The first 5\% of SAE training is a feature lottery: testing the per-atom prediction}\label{sec:sae-lottery}

\begin{figure}[H]
\centering
\includegraphics[width=\linewidth]{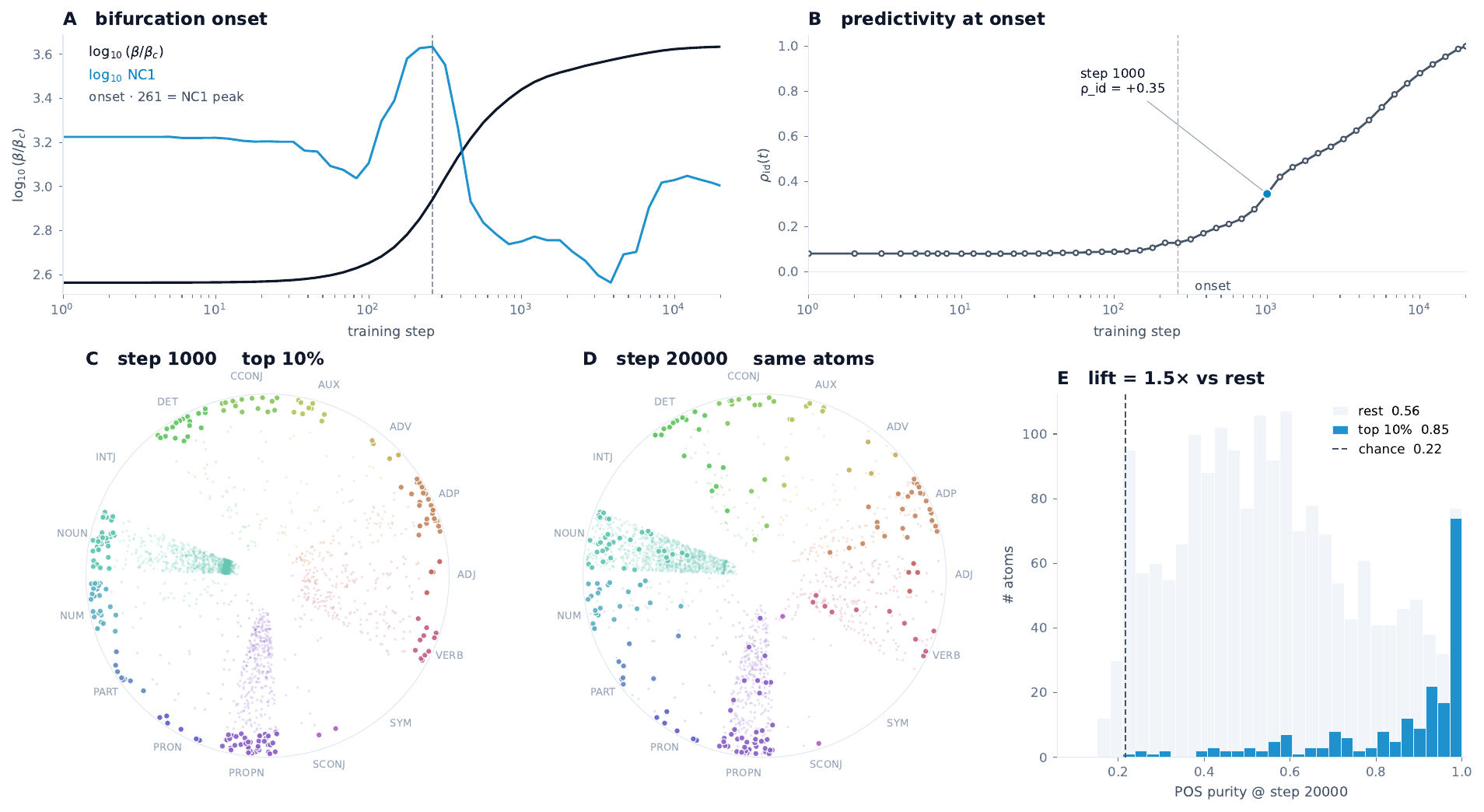}
\caption{\textbf{SAE feature lottery emerges at the bifurcation onset
and is operationally selectable by $5\%$ of training.}
($K\!=\!2048$ top-$K$ SAE on Pythia-160M layer~6; one canonical seed
shown, 3-seed statistics quoted.)
\textbf{(A)} Onset: $\log_{10}(\beta_{\mathrm{SAE}}/\betac)$ (black)
crosses zero and rises; $\log_{10}\mathrm{NC1}_{\mathrm{features}}$
(blue) peaks at step $261$, our operational definition of onset.
\textbf{(B)} Per-atom predictivity: cross-atom Spearman
$\rho_{\mathrm{id}}(t)$ of POS purity at $t$ vs convergence sits at
noise floor pre-onset, climbs sharply at the bifurcation, reaches
${+}0.41\!\pm\!0.04$ at step $1{,}000$ (${+}0.35$ in shown seed).
\textbf{(C)} Top decile of atoms by step-$1{,}000$ POS purity
($n\!=\!205$ of $1843$ active), placed at the angular sector defined
by their dominant POS class; radius = step-$1{,}000$ purity.
\textbf{(D)} Same atoms, same sectors, radius now = step-$20{,}000$
purity: the outer-ring population persists in every major POS sector;
bulk inward drift reflects sparsity tightening, not specialization
loss.
\textbf{(E)} Top decile reaches POS purity $0.82\!\pm\!0.03$ at
convergence (3 seeds: $0.80/0.86/0.82$) against $0.47\!\pm\!0.03$
for the bottom decile; the title gives the lift over the remaining
atoms in the shown seed. Dashed line: chance level $0.22$, the purity
of $100$ randomly drawn eval tokens (App.~\ref{app:lottery-details}).
Reports POS-\emph{purity} persistence, not POS-class identity
preservation.}
\label{fig:lottery}
\end{figure}

Sections~\ref{sec:theory}--\ref{sec:arc} establish that the
bifurcation onset is a \emph{collective} event: the encoder's
symmetric collapsed state loses stability at a specific moment, with
an unstable subspace shared by all $K\!-\!1$ anti-symmetric modes
(App.~\ref{app:hessian:eig}). The theory therefore makes a sharp
\emph{per-atom} prediction: at the bifurcation, each atom must
select a specific direction from this shared manifold, and that
selection should be detectable as an atom-level identity event in
training time. This section tests that prediction in the cleanest
available substrate.

Among the four regimes catalogued in Sec.~\ref{sec:arc}, the SAE on
frozen Pythia-160M layer~6 is the only setting in which $\betac$ is
constant by construction; the upstream encoder does not evolve.
This makes it the cleanest substrate to disentangle the bifurcation
event (shared by both SAE families we tested) from \emph{per-feature}
interpretability (architecture-dependent). It is also the natural
substrate for studying the crossing event in language models at all:
raw LM activations are heavily anisotropic (Pythia-160M layer~6 has
$\lammax(\Cov(z))/\bar\sigma^2 \approx 380$), so a randomly
initialized GMM probe attached directly to LM hidden states starts
already supercritical and never visibly crosses $\betac$. The SAE
itself, by contrast, initializes with low $\beta_{\mathrm{SAE}}$ and
grows past a fresh critical point, making the full pre-/post-critical
arc observable in the sparse-code basis. We train two SAE families
on the same Pythia activations with dictionary size $K=2048$ over
input dimension $d=768$:
\textbf{(i) soft-L1} (ReLU activation, $L_1$ penalty on activations,
$\lambda \in \{5{\times}10^{-4}, 2{\times}10^{-3}, 5{\times}10^{-3},
10^{-2}\}$) and \textbf{(ii) top-$K$} (hard architectural sparsity
with $\text{top}_k = 64$ active atoms per token, following
\citealt{gao2025scaling}). Both are trained for $20{,}000$ steps with
identical optimizer, batch size, and dense ($\sim$55-point)
checkpoint schedules.

We find a sharp dynamical phase transition: per-atom predictive
content goes from zero pre-onset to substantial within the first
$5\%$ of training, and atoms ranked at $5\%$ already recover the
highest-purity atoms at convergence. Following \citet{frankle2018lottery},
we refer to this as a \emph{feature lottery}; the drawing event
is the first phase transition during training, not initialization.
The lottery is the theory's sharpest empirical confirmation; the
detailed quantitative results follow.

\paragraph{Per-atom identity.} For each saved checkpoint we forward
a fixed POS-labeled WikiText-103 eval set ($N=50{,}000$ tokens)
through the SAE to obtain feature activations $f \in \mathbb{R}^{N
\times K}$. Atom $k$'s \emph{activation-pattern identity} is the
column $f_{:,k} \in \mathbb{R}^N$; two atoms have the same identity
when their $N$-dimensional activation vectors are cosine-close. This
substrate is invariant to atom permutation and scaling, and bypasses
the geometric near-saturation of decoder-column matching that arises
in overcomplete dictionaries ($K > d$, App.~\ref{app:hungarian-fail}).

\subsection{The lottery is selectable by 5\% of training}\label{sec:sae-lottery:claim}

The cross-atom $\rho_{\mathrm{id}}(t)$ trajectory plotted in
Fig.~\ref{fig:lottery}B is constructed by \emph{identity matching}:
atom $k$ at step $t$ is compared to atom $k$ at step $20{,}000$, with
no cross-time atom permutation (an alternative Hungarian-matched
$\rho_H$ is artifact-inflated; see ``Why not Hungarian'' below).
Values at representative checkpoints:

\begin{center}
\begin{tabular}{lrl}
\toprule
\textbf{step} & $\rho_{\mathrm{id}}$ & \textbf{interpretation} \\
\midrule
46             & $+0.03$ & noise floor, atoms have no identity \\
178            & $+0.06$ & still noise floor \\
\textbf{261}   & $+0.16$ & \textbf{bifurcation onset; lottery begins} \\
464            & $+0.31$ & post-critical descent \\
\textbf{1{,}000} & $\mathbf{+0.41}$ & \textbf{5\% of training; useful early-ranking signal acquired} \\
2{,}154        & $+0.49$ & slow refinement \\
10{,}000       & $+0.88$ & \\
20{,}000       & $+1.00$ & anchor (self-comparison) \\
\bottomrule
\end{tabular}
\end{center}

Two caveats beyond what Fig.~\ref{fig:lottery} shows:
\begin{enumerate}[leftmargin=*,itemsep=2pt,topsep=2pt]
\item \emph{Pre-onset $\rho_{\mathrm{id}}$ is statistically
  indistinguishable from $0$} for $t<200$, confirming atoms have no
  detectable identity before the bifurcation.
\item \emph{The lottery is operationally useful at $5\%$, but identity
  refinement continues.} The remaining $95\%$ of training lifts
  $\rho_{\mathrm{id}}$ from $0.41$ to $0.88$. We therefore do not
  claim the lottery is ``complete'' at $5\%$ in the strong sense of
  \citet{frankle2018lottery} (where identification at initialization
  is followed by isolated retraining to full accuracy). Ours is the
  weaker statement that early ranking provides a useful predictor of
  converged identity.
\end{enumerate}

This is the SAE-level analogue of the lottery-ticket framing of
\citet{frankle2018lottery}: where they show that winning subnetworks
are selectable at initialization in supervised classifiers, we find
that winning SAE atoms are selectable at the first phase transition
of unsupervised dictionary training. The bifurcation onset is the
lottery's drawing event.

\paragraph{Ranking lift (three seeds).} At convergence, the decile
of atoms ranked by their step-$1{,}000$ POS purity has mean POS
purity $0.82 \pm 0.03$ (mean$\pm$std across $n\!=\!3$ seeds; seed
0 / 1 / 2 give 0.80 / 0.86 / 0.82) versus $0.47 \pm 0.03$ for the
bottom decile. The relevant null for the lottery claim --- ``does
step-$1{,}000$ ranking carry predictive content for convergence
identity?'' --- is unranked selection: a decile picked without
looking at step-$1{,}000$ purity has the mean purity of all active
atoms, $0.58$ in seed~0 (Fig.~\ref{fig:app:lottery}F), and the atoms
outside the top decile average $0.56$ (Fig.~\ref{fig:lottery}E).
For scale, an atom with no POS selectivity has purity $0.22$:
purity keeps the most common tag among an atom's top-$100$ tokens,
and the most common tag in the eval set (NOUN) covers $21\%$ of
tokens, so $100$ eval tokens drawn at random score
$0.217 \pm 0.035$ (App.~\ref{app:lottery-details}). The
identity-matched Spearman across the three seeds is
$\rho_{\mathrm{id}} = +0.41 \pm 0.04$, all $p < 10^{-80}$. The
corresponding single-seed numbers for the soft-L1 SAE are $0.64$
(top decile) versus $0.35$ (bottom decile) and $0.39$ (all atoms).
Both architectures support the
same conclusion: \textbf{early winner screening at $5\%$ of training
is feasible and identifies a substantial fraction of the
highest-purity atoms},
though full training is still required to obtain each atom's
highest-quality final activations.

\paragraph{Architecture invariance.} The same $\rho_{\mathrm{id}}$
trajectory and $5\%$ cutoff are recovered in soft-L1 SAEs
(ReLU activation, $L_1$ penalty $\lambda \in \{5\!\times\!10^{-4},
2\!\times\!10^{-3}, 5\!\times\!10^{-3}, 10^{-2}\}$, $L_0 \approx 1000$
of $K\!=\!2048$) despite their different sparsity regime: the
post-critical descent of $\rho_{\mathrm{id}}$ is qualitatively
identical (App.~\ref{app:hungarian-fail}), demonstrating that the
lottery is a property of the bifurcation, not of the architectural
top-$K$ mask.

\paragraph{Why $\rho_{\mathrm{id}}$ and not $\rho_H$ (Hungarian).}
We initially considered Hungarian-matched correlation
$\rho_H$; match atoms across time by maximizing cosine of
activation patterns, then correlate matched pairs' POS purities.
This metric is systematically inflated: at random initialization
(step $46$, where atoms have no identity by construction),
$\rho_H \approx 0.30$ rather than $0$
(App.~\ref{app:lottery-details}, panels D--E). The inflation arises because
Hungarian matching preferentially pairs atoms with similar
activation profiles, and POS purity is itself computed from the
activation profile, so high-purity atoms find each other across
time even when no underlying identity is preserved. Subtracting the
$0.30$ floor recovers the $\rho_{\mathrm{id}}$ trajectory.
$\rho_{\mathrm{id}}$, which is invariant to this confound, is the
estimator we report throughout.

\paragraph{Specificity.} Of three per-atom early metrics tested,
only POS purity carries the predictive signal:
\begin{itemize}[leftmargin=*,itemsep=2pt,topsep=2pt]
\item \emph{Early POS purity} $\to$ final POS purity:
  $\rho_{\mathrm{id}} = +0.41$ at $5\%$. \textbf{Predictive.}
\item \emph{Early activation concentration} $\to$ final POS purity:
  $\rho \approx 0$ in the top-$K$ SAE, $\rho \approx -0.28$ in the
  soft-L1 SAE. \textbf{Not predictive.}
\item \emph{Early identity-lock cosine} $\to$ final POS purity:
  $\rho = +0.20$ (top-$K$). Weakly predictive.
\end{itemize}
This specificity rules out the trivial reading ``any atom-level
property predicts its convergence value.'' What is drawn at the
bifurcation is specifically \emph{linguistic identity}, not generic
firing strength or generic stability.

\subsection{Mechanism: identity lock during the lottery window}\label{sec:sae-lottery:lock}

Why is the lottery permanent past $5\%$ of training? Atom identities
themselves lock during the post-critical window. The Hungarian-matched
cosine between per-atom activations at step $t$ and at the converged
SAE remains flat (at random baseline $0.13$ for top-$K$, $0.68$ for
soft-L1) for $t\!<\!200$, begins a sharp rise at the NC1 peak (step
$261$ for top-$K$, $383$ for soft-L1), and saturates at $1.0$ by step
$\sim 15{,}000$. Both SAE families share this timing; only their
random-init baselines differ (soft-L1's ReLU rows already project
onto the data subspace). The onset is therefore the
\emph{trigger} of a per-atom assignment process: each atom commits to
a specific activation pattern during the post-critical window, and
subsequent training refines but does not reshuffle that commitment
(App.~\ref{app:lottery-details}, Fig.~\ref{fig:app:lottery}).

\subsection{Architecture-dependent ceiling on feature
interpretability}\label{sec:sae-lottery:ceiling}

The bifurcation event is shared by both SAE families we tested, but
the \emph{post-onset ceiling} on per-feature interpretability is not.
Under architectural top-$K$ sparsity, the median active feature
reaches POS purity $0.56$ (chance level $0.22$); the soft-L1 family
plateaus at $0.33$, since $L_0$ saturates near $1{,}000$ of $K\!=\!2048$
irrespective of $\lambda \in [5{\times}10^{-4}, 10^{-2}]$. The
bifurcation locks identities identically in both families; only
top-$K$ pushes those identities into linguistically selective
directions. Who wins the lottery is determined at the bifurcation;
how linguistically meaningful the winning identity is depends on
the post-onset sparsity regime (App.~\ref{app:lottery-details}).

\subsection{K-sweep: universality of the lottery, monotonicity of
interpretability and slope}\label{sec:sae-lottery:ksweep}

\begin{figure}[t]
\centering
\includegraphics[width=\linewidth]{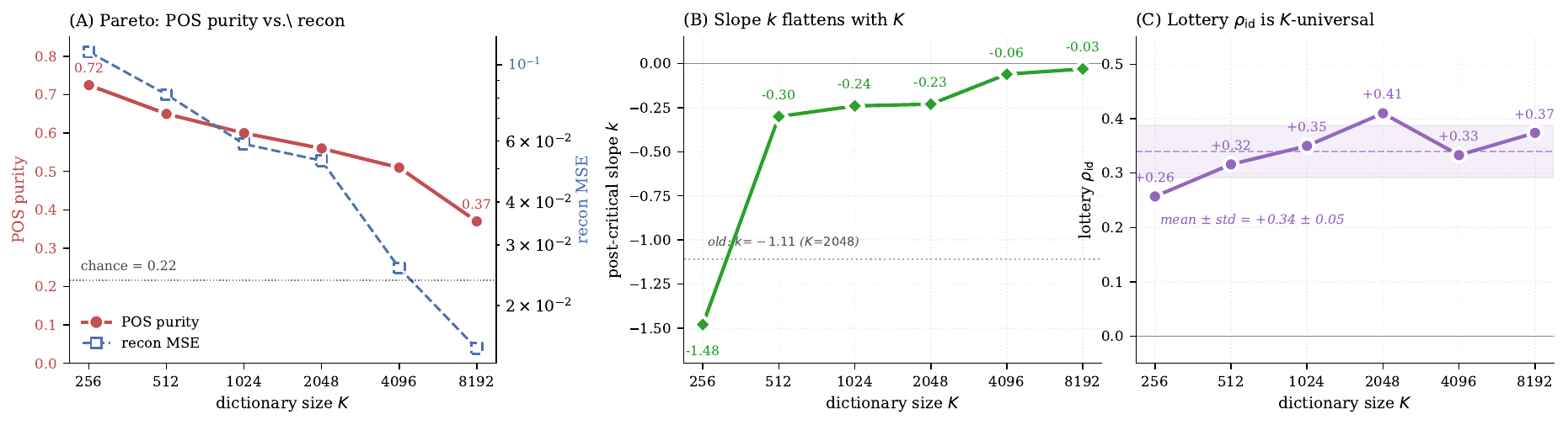}
\caption{\textbf{K-sweep across $K \in \{256, 512, 1024, 2048, 4096,
8192\}$ at fixed 3\% top-$K$ sparsity.} For each $K$, we report
final $L_0$, final reconstruction MSE, the lottery Spearman
$\rho_{\mathrm{id}}$, and the mean POS purity of the top-decile
atoms ranked by step-$1{,}000$ POS purity. Universality across $K$:
the lottery $\rho_{\mathrm{id}} \in [+0.26, +0.41]$ is $K$-stable and
consistently positive: the bifurcation onset is a $K$-universal
event. Monotonicity in $K$: POS purity \emph{decreases} with $K$
($0.725$ at $K\!=\!256$, $0.370$ at $K\!=\!8192$; dotted: chance
level $0.22$), while reconstruction MSE decreases the other way.}
\label{fig:ksweep}
\end{figure}

We sweep $K \in \{256, 512, 1024, 2048, 4096, 8192\}$ at fixed
top-$K$ sparsity ratio of $3\%$ (so $\text{top}_k$ scales
proportionally as $\{8, 16, 32, 64, 128, 256\}$) with all other
hyperparameters held constant. The $K$-stability of $\rho_{\mathrm{id}}$
visible in Fig.~\ref{fig:ksweep} is statistically tight
($p < 10^{-3}$ at every $K$), justifying treating the bifurcation
onset as a $K$-universal event.

\paragraph{POS purity is monotonically decreasing in $K$, but
$K$-confounded.}
The median atom's POS purity at convergence falls steadily from
$0.725$ at $K\!=\!256$ to $0.370$ at $K\!=\!8192$
(Table~\ref{tab:ksweep}). The mechanism is that a smaller dictionary
forces each atom to absorb a coarser token-cluster partition that
more easily aligns with the $15$-way POS partition; a larger
dictionary spreads across finer sub-categories that POS purity is
too coarse to resolve. We emphasize that this trend is therefore
\emph{partly a structural property} of POS purity as a metric, not a
direct statement that small $K$ produces more mechanistically
interpretable features: a $K$-unbiased interpretability metric
(causal mediation, LLM-as-judge) is required before a Pareto-based
recommendation can be made.

\begin{table}[h]
\centering
\small
\caption{K-sweep at fixed top-$K$ sparsity ratio $3\%$, all other
hyperparameters constant. The lottery $\rho_{\mathrm{id}}$ is
$K$-stable; POS purity decreases monotonically with $K$ while
reconstruction MSE decreases the other way.}
\label{tab:ksweep}
\begin{adjustbox}{max width=\textwidth}
\begin{tabular}{rrrrr}
\toprule
$K$ & $L_0$ & recon MSE & POS purity (median atom) & $\rho_{\mathrm{id}}$ (lottery) \\
\midrule
256  & 8   & $1.09\!\times\!10^{-1}$ & 0.725 & $+0.26$ \\
512  & 16  & $8.20\!\times\!10^{-2}$ & 0.650 & $+0.32$ \\
1024 & 32  & $5.89\!\times\!10^{-2}$ & 0.600 & $+0.35$ \\
2048 & 64  & $5.27\!\times\!10^{-2}$ & 0.560 & $+0.41\pm 0.04$ ($n{=}3$) \\
4096 & 128 & $2.57\!\times\!10^{-2}$ & 0.510 & $+0.33$ \\
8192 & 256 & $1.50\!\times\!10^{-2}$ & 0.370 & $+0.37$ \\
\bottomrule
\end{tabular}
\end{adjustbox}
\end{table}

\paragraph{Pareto frontier, with a $K$-bias caveat.}
$K\!=\!256$ scores higher on POS purity; large $K$ scores better on
reconstruction. The framework adds a second axis (atom-level POS
purity, plus the lottery $\rho_{\mathrm{id}}$) to the standard SAE
$L_0$-vs-recon Pareto. However, because POS has only $15$ classes,
the POS-purity-vs-$K$ trend is partly a structural artifact: smaller
$K$ forces each atom to absorb a coarser token-cluster partition,
which more easily aligns with a $15$-way categorical metric. We
therefore do \emph{not} recommend small $K$ on the basis of POS
purity alone; a $K$-unbiased interpretability metric (such as causal
mediation \citep[cf.][]{nanda2023progress} or LLM-as-judge scoring)
is required before a substantive recommendation can be made. What
\emph{is} robust to this caveat is the $K$-universal lottery effect:
at every $K\!\in\!\{256,\dots,8192\}$, atom-level identity is
predictable from $5\%$ training with $\rho_{\mathrm{id}} \in
[+0.26, +0.41]$.

\subsection{Connection to the bifurcation theory}\label{sec:sae-lottery:connection}

The pitchfork bifurcation and the feature lottery are two scales of
the same event. At $\beta\!=\!\betac$ the unstable subspace is shared
by all $K\!-\!1$ anti-symmetric modes
(App.~\ref{app:hessian:eig}); each atom selects a specific direction
from this manifold via random initialization noise and the cubic
terms in~\eqref{eq:pitchfork}. The collective bifurcation enables
the per-atom selection; the per-atom selection is what makes
post-onset identities permanent and partially predictable. The
following proposition formalizes the per-atom directional
preservation that the empirical $\rho_{\mathrm{id}} > 0$
(Sec.~\ref{sec:sae-lottery:claim}) realizes.

\begin{proposition}[Per-atom directional persistence]\label{prop:lottery}
Consider the coupled order-parameter dynamics on the
$(K\!-\!1)$-fold-degenerate unstable subspace just past the crossing
($\mu := \beta-\betac > 0$, small),
\begin{equation}
\dot{\bm\varepsilon}_k \;=\; \mu\,\bm\varepsilon_k
              \;-\; \alpha\,\|\bm\varepsilon_k\|^2 \bm\varepsilon_k
              \;-\; \gamma\!\!\sum_{j\neq k}\!\bigl(\bm\varepsilon_j^\top \bm\varepsilon_k\bigr)\bm\varepsilon_j
              \;+\; \bm\eta_k(t),
\qquad k=1,\dots,K,
\label{eq:mode-coupling}
\end{equation}
with i.i.d.\ Langevin noises $\bm\eta_k$ at intensity $D$ and cubic
inter-mode coupling strength $\gamma \ge 0$. Assume each initial
perturbation satisfies $\|\bm\varepsilon_k(0)\| > \sigma_*\!\cdot\!\sqrt{D/\mu}$
for some $\sigma_* > 1$ (initial perturbation above the noise floor).
Then in the weak-coupling regime $\gamma \ll \alpha$, for any
\emph{finite} time $T$ within the post-saturation, pre-randomization
window $\tau_r \lesssim T \ll T_{\mathrm{rand}} :=
r^{\star 2}/(2(d{-}1)D)$ (where $\tau_r$ is the deterministic radial
saturation timescale and $T_{\mathrm{rand}}$ is the spherical
randomization timescale on the saturated attractor), the per-atom
directional persistence is strictly positive:
\begin{equation}
\mathbb{E}\bigl[\,d_k(0)^\top d_k(T)\,\bigr] \;>\; 0,
\qquad
d_k(t) := \bm\varepsilon_k(t)/\|\bm\varepsilon_k(t)\| \in S^{d-1},
\label{eq:lottery-positive}
\end{equation}
with magnitude controlled by $\sigma_*\,\mu / (\alpha D)^{1/2}$ and
expectation taken over the Langevin noise realizations. For training
runs of fixed horizon, $T$ is the training time; in the strict
$T \to \infty$ limit, angular Brownian motion on the saturated
sphere eventually randomizes $d_k(T)$ and the persistence claim
breaks down (outside the scope of this proposition).
\end{proposition}

\begin{proof}[Proof sketch (full proof in App.~\ref{app:lottery-proof}).]
At $\gamma = 0$, modes decouple. Each $\bm\varepsilon_k$ obeys an
independent pitchfork SDE whose deterministic flow is \emph{radial}:
the linear growth $\mu\bm\varepsilon_k$ preserves direction, and the
cubic self-saturation $-\alpha\|\bm\varepsilon_k\|^2 \bm\varepsilon_k$
also preserves direction. Noise rotates direction with effective
diffusion coefficient $(d{-}1)D/\|\bm\varepsilon_k\|^2$, which decays
rapidly as $\|\bm\varepsilon_k\|$ grows from $\sigma_*\sqrt{D/\mu}$
toward $\sqrt{\mu/\alpha}$. Over the finite window $[0,T]$, the
accumulated angular variance is
$\Theta^2(T) \approx (d{-}1)/\sigma_*^2 + 2(d{-}1)D(T-\tau_r)/r^{\star 2}$,
which is $O(1/\sigma_*^2)$ as long as the saturation-phase
contribution $(T-\tau_r)/T_{\mathrm{rand}}$ is sub-leading.
For small $\Theta(T)$, $\mathbb{E}[d_k(0)^\top d_k(T)] \approx
1 - \Theta^2(T)/2 > 0$. Adding weak inter-mode coupling $\gamma > 0$
contributes an $O(\gamma/\alpha)$ perturbation that preserves the
sign of $\mathbb{E}[d_k(0)^\top d_k(T)]$ at leading order, since
the coupling vanishes when modes are mutually orthogonal (and modes
self-orthogonalize under the cubic saturation). $\square$
\end{proof}

\paragraph{Empirical proxies for per-atom persistence.}
Equation~\eqref{eq:lottery-positive} is per-atom: at finite
$T$ within the persistence window, each atom's direction positively
correlates with its direction at the bifurcation onset. We do not
have direct access to $d_k(T)^\top d_k(0)$ in real SAEs (the unstable
manifold is not explicitly parameterized); we use the following two
empirically accessible proxies, both of which inherit positivity from
~\eqref{eq:lottery-positive} under the same regime assumption.

\emph{Scalar-projection proxy (toy verification,
App.~\ref{app:lottery-toy}, Fig.~\ref{fig:app:prop3-toy}).}
Fix a reference direction $r \in S^{d-1}$ and consider the per-atom
scalar autocorrelation $A_k(T) := \langle\bm\varepsilon_k(0),r\rangle \cdot
\langle\bm\varepsilon_k(T),r\rangle / (\|\bm\varepsilon_k(0)\|
\|\bm\varepsilon_k(T)\|)$. By~\eqref{eq:lottery-positive},
$\mathbb{E}[A_k(T)] > 0$ for each $k$. Across $K$ i.i.d.\ atoms, the
Spearman correlation between the projection pair
$(\langle\bm\varepsilon_k(0),r\rangle, \langle\bm\varepsilon_k(T),r\rangle)$
converges to a positive population Spearman as $K\!\to\!\infty$;
direct SDE simulation in App.~\ref{app:lottery-toy} confirms
$\rho = 0.948 \pm 0.006$ across $5$ seeds.

\emph{POS-purity proxy (SAE empirics, Sec.~\ref{sec:sae-lottery:claim}).}
Per-atom POS purity is the magnitude of an atom's projection onto a
linguistic POS subspace of $\mathbb{R}^N$. Under
~\eqref{eq:lottery-positive}, this projection is per-atom-autocorrelated
across training time: the early-time and late-time POS purity for the
same atom are positively correlated. The identity-matched Spearman
$\rho_{\mathrm{id}} = 0.41 \pm 0.03$ across atoms
(Sec.~\ref{sec:sae-lottery:claim}) is the population Spearman of this
per-atom autocorrelation; the positive sign is~\eqref{eq:lottery-positive}'s
prediction, the specific magnitude ($+0.41$) depends on the linguistic
geometry of POS classes in Pythia's residual stream and is not
predicted by the theory.

\subsection{Scope and limitations}\label{sec:sae-lottery:scope}

\begin{figure}[t]
\centering
\includegraphics[width=\linewidth]{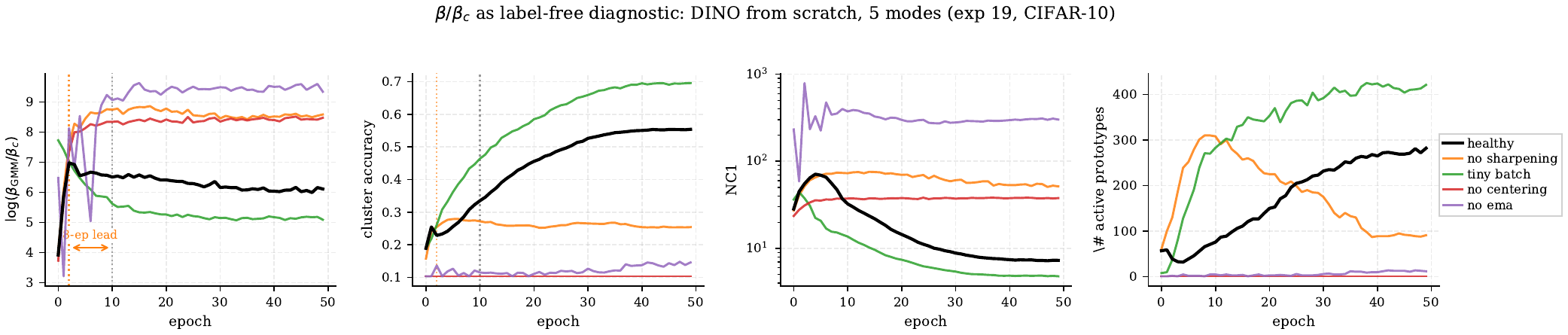}
\caption{\textbf{From-scratch DINO with four perturbed modes ($n=1$
seed per mode).} Healthy (black) versus four perturbations over
$50$ epochs of CIFAR-10 training. The two perturbations relevant for
the diagnostic claim are \emph{no\_sharpening} (orange) and
\emph{tiny\_batch} (green): both separate from healthy in
$\log(\beta/\betac)$ at epoch~2 (vertical dotted line) while cluster
accuracy diverges only by epoch~10, an $\sim$8-epoch lead time. The
two outcomes go in opposite directions; no\_sharpening degrades
($\sim 25\%$ cluster accuracy at epoch 50 vs $\sim 55\%$ healthy);
tiny\_batch improves ($\sim 70\%$). The catastrophic modes
(\emph{no\_centering}, \emph{no\_ema}) collapse from initialization
and provide no epoch-resolution lead-time signal.}
\label{fig:diag:fromscratch}
\end{figure}

The lottery claim is bounded along three axes. (i) \emph{POS purity
is one interpretability dimension}: atoms detecting finer structure
(named entities, sub-POS, phrase patterns) may be mechanistically
interpretable without scoring high on POS purity; our claim concerns
the POS-selective subset only. (ii) \emph{Secondary bifurcations.}
The hierarchical structure (Sec.~\ref{sec:theory}) predicts that
finer features should themselves exhibit phase transitions at
$\betac^{(2)} = 1/\lammax(\Sigma_{\mathrm{within}})$ later in
training; verification at the secondary scale is left to future work.
(iii) \emph{Tagger noise.} spaCy~\citep{honnibal2020spacy} reports
$\sim$$97\%$ POS accuracy on standard benchmarks, with errors
concentrated on ambiguous cases (gerunds, deverbal nouns); reported
$\rho_{\mathrm{id}}$ values are therefore conservative lower bounds
on true atom-level POS selectivity.

\paragraph{Frequency-confound check.}
A natural concern is that the lottery effect is inflated by token
frequency: atoms that lock onto a single high-frequency token
(e.g.\ ``the'') would achieve high POS purity trivially, since
high-frequency function words have stable POS tags. To rule this out,
we stratify atoms by the mean log-frequency of their top-100
activating tokens into five quintiles and re-compute $\rho_{\mathrm{id}}$
within each quintile (3 seeds, $K\!=\!2048$ top-$K$ SAE). The result
(Tab.~\ref{tab:freq-strat}, Fig.~\ref{fig:app:freq-strat}): all five
quintiles show $\rho_{\mathrm{id}} \in [+0.37, +0.51]$, well above
zero ($p < 10^{-3}$ each), with no monotonic increase from rare to
frequent tokens. The frequency confound is not load-bearing.

\begin{table}[H]
\centering
\small
\caption{$\rho_{\mathrm{id}}$ stratified by atom top-100-token mean
log-frequency, 3 seeds, $K\!=\!2048$. Unstratified
$\rho_{\mathrm{id}} = +0.416 \pm 0.010$. The quintile range
$[+0.37, +0.51]$ is small compared to the overall effect over the
random baseline ($\rho\!=\!0$), and Q4 (mid-frequency content words)
rather than Q5 (most frequent function words) has the highest
$\rho_{\mathrm{id}}$, opposite to what a single-token-lock artifact
would predict.}
\label{tab:freq-strat}
\begin{adjustbox}{max width=0.7\textwidth}
\begin{tabular}{lc}
\toprule
\textbf{quintile of atom top-token mean log-freq} & $\boldsymbol{\rho_{\mathrm{id}}}$ \\
\midrule
Q1 (rarest tokens)    & $+0.368 \pm 0.034$ \\
Q2                    & $+0.373 \pm 0.020$ \\
Q3                    & $+0.365 \pm 0.025$ \\
Q4                    & $\mathbf{+0.512 \pm 0.029}$ \\
Q5 (most frequent)    & $+0.424 \pm 0.026$ \\
\bottomrule
\end{tabular}
\end{adjustbox}
\end{table}

\section{Label-free training diagnostic in practice}\label{sec:diagnostic}

The phase-identification ability documented in
Sec.~\ref{sec:arc:grokking} (that $\beta/\betac$ reads the encoder's
current act from its hidden state alone, well before downstream
metrics respond) generalizes beyond grokking. We now test the
same property in a more typical training-health setting: detecting
the onset of representation degradation in a self-supervised
encoder. The bifurcation framework predicts that $\beta/\betac$ tracks
representation-level state. We now ask whether this is operationally
useful as a label-free training health indicator: does $\beta/\betac$
respond earlier than downstream metrics when training drifts away
from a healthy trajectory? We test this in two complementary setups:
DINO trained from scratch with collapse modes injected at
initialization (Sec.~\ref{sec:diag:fromscratch}), and DINO with
interventions applied to a healthy mid-training checkpoint
(Sec.~\ref{sec:diag:intervention}).

\subsection{From-scratch collapse modes}\label{sec:diag:fromscratch}

We train ResNet-18 + DINO on CIFAR-10 for 50 epochs in five
configurations (Fig.~\ref{fig:diag:fromscratch}): healthy;
\emph{no\_centering} (no center-buffer update on teacher logits);
\emph{no\_sharpening} (teacher temperature pinned to student
temperature); \emph{no\_ema} (teacher copies student each step);
\emph{tiny\_batch} (batch size 32 instead of 256).

The two non-catastrophic perturbations
(\emph{no\_sharpening}, \emph{tiny\_batch}) probe the diagnostic in
opposite directions:

\paragraph{Negative case (\emph{no\_sharpening}).}
$\log(\beta/\betac)$ crosses the $0.5$-log-unit deviation threshold
from healthy at epoch~2 while cluster accuracy crosses the
$5\%$-deviation threshold only at epoch~10: \textbf{an $\sim$8-epoch
lead time on a 50-epoch run}, with the downstream outcome being
degradation ($25\%$ vs $55\%$ at epoch 50).

\paragraph{Positive case (\emph{tiny\_batch}).}
A non-catastrophic perturbation (batch $32$ vs $256$) where
$\log(\beta/\betac)$ also separates from healthy by $\sim 0.5$
log-units at epoch~2, while the training-loss difference at that
point is within batch noise. By epoch~50, cluster accuracy has
separated by $15$ percentage points in the \emph{opposite} direction
($70\%$ tiny\_batch vs $55\%$ healthy). The early $\beta/\betac$
separation thus predicts the downstream divergence with $\sim$48-epoch
lead time; and the divergence is positive, not negative. Together
with no\_sharpening, this shows that $\beta/\betac$ reads the
representation's \emph{geometric trajectory state}, not a binary
``health'' signal: distinct $\beta/\betac$ trajectories predict
distinct downstream outcomes, regardless of sign.

\paragraph{Threshold caveat.} The $0.5$-log-unit and $5\%$ thresholds
are heuristic, chosen post-hoc. A full ROC-style analysis (detection
time vs false-positive rate) requires multi-seed evaluation; current
experiments use $n=1$ seed per mode. We report lead time at fixed
heuristic thresholds for visual interpretability and leave the
multi-seed ROC analysis to future work.

\subsection{Mid-training interventions}\label{sec:diag:intervention}

We then test whether the same signal works on a healthy
mid-training checkpoint subjected to a perturbation. Phase A: 2
epochs of healthy DINO on CIFAR-100 (Phase A is shared across modes
via a checkpoint). Phase B: 10 epochs of intervention. Dense probes
at steps $\{1,5,10,25,50,100,200,500\}$ after the Phase B transition
resolve per-batch dynamics.

\begin{figure}[t]
\centering
\includegraphics[width=\linewidth]{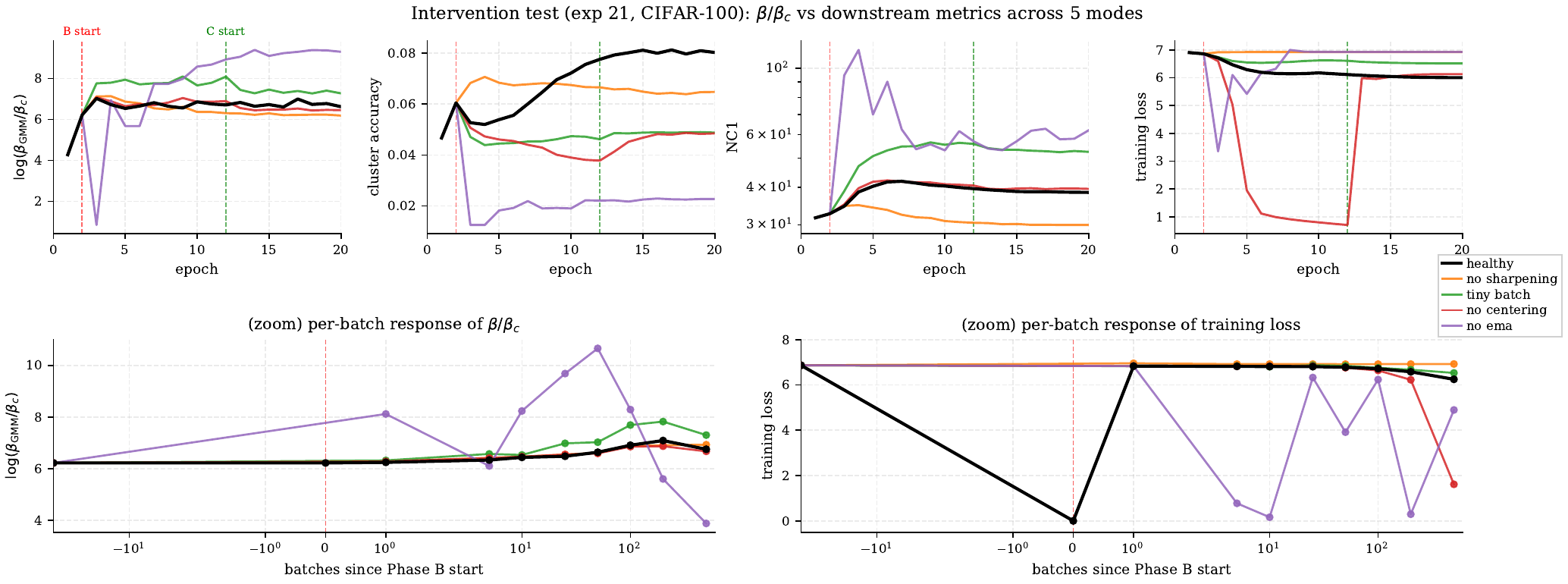}
\caption{\textbf{Mid-training intervention test on CIFAR-100, 5
modes.} Top: epoch-level traces, red dashed line at Phase B start.
Bottom: zoom on Phase B onset (symmetric-log $x$-axis). Two
paper-relevant signatures: (i) \emph{no\_ema} (purple):
$\log(\beta/\betac)$ moves by 1.9 log-units within one batch,
oscillating up to $+10.7$ within 50 batches, while training loss
oscillates uninterpretably; (ii) \emph{tiny\_batch} (green):
$\log(\beta/\betac)$ separates from healthy by 0.5 log-units within
25 batches while training loss differs by $0.06$ (within noise).}
\label{fig:diag:intervention}
\end{figure}

\paragraph{Two paper-relevant signatures.}
\begin{itemize}[leftmargin=*,itemsep=2pt,topsep=2pt]
\item \emph{Per-batch sensitivity (no\_ema).} Within one batch of
  removing the teacher EMA, $\log(\beta/\betac)$ jumps by
  $+1.90$ log-units (from $+6.22$ to $+8.12$). The trajectory then
  oscillates with $\pm 2$-log-unit amplitude. Training loss
  oscillates as well over the same window but in an uninterpretable
  way, swinging between $6.93$ and $0.16$ and not consistently
  tracking representation-level state. The $\beta/\betac$ signal is
  the cleaner per-batch indicator.
\item \emph{Sub-catastrophic lead time (tiny\_batch).} At 25 batches
  into Phase B, $\log(\beta/\betac)\!=\!+6.98$ for tiny\_batch
  versus $+6.48$ for healthy, a $0.50$ log-unit separation.
  Training loss at the same step is $6.869$ versus $6.809$: a
  $0.060$ difference, within batch-to-batch noise. The
  $\beta/\betac$ signal distinguishes the two modes long before
  loss can.
\end{itemize}

\paragraph{Headline.}
Across these two single-seed experimental setups
(Sec.~\ref{sec:diag:fromscratch} + Sec.~\ref{sec:diag:intervention}),
$\beta/\betac$ separates from the healthy trajectory $\sim\!8$ epochs
before cluster accuracy does in the gradual case (no\_sharpening,
degradation; tiny\_batch, improvement) and within a single batch in
the catastrophic intervention case (no\_ema). Both lead-time numbers
are at heuristic post-hoc thresholds and on $n\!=\!1$ seed per
condition; multi-seed ROC analysis is needed before quoting these
as detection times in deployment. The lead time is not directional:
in the tiny\_batch case, the early $\beta/\betac$ separation
\emph{positively predicts} a downstream cluster-accuracy gain. The
signal is derived purely from the encoder's hidden states and a
passive probe; no labels enter at any point.

\section{Conclusion}\label{sec:conclusion}
Concept emergence in feature-learning networks admits a label-free
dynamical indicator. From a single Hessian analysis of a passive
GMM probe on a given encoder representation, we derive a critical precision
$\betac\!=\!1/\lammax(\Cov(z))$ above which prototypes pitchfork; we
prove a finite-time crossing theorem for co-evolving encoders that
eventually spread the latent representation sufficiently
(Proposition~\ref{prop:endogenous}), and identify a post-critical
metastable regime under finite dissipation
(Remark~\ref{rem:metastability}). The encoder--probe race in
$\beta(t)/\betac(t)$ traces one of four kinematic regimes (full V,
fold-back, delayed escape, no arc) governed by three binary axes,
which we verify across SAEs, SSL on CIFAR-10/100, grokking with
multi-seed dissipation-threshold control, and a rotation-prediction
negative control. The sharpest empirical confirmation is per-atom:
in SAE training the unstable subspace is shared across atoms at the
crossing, and we observe the predicted lottery; the top decile
of atoms ranked at $5\%$ of training ends at POS purity $0.82$
(bottom decile $0.47$, chance level $0.22$), and the lottery holds in
two distinct SAE architectures, across $K\!\in\![256,8192]$, with
three-seed reproducibility
($\rho_{\mathrm{id}}\!=\!+0.41\!\pm\!0.04$, all $p\!<\!10^{-80}$).
The same $\beta/\betac$ quantity acts as a practical training-health
indicator with multi-epoch lead time over downstream metrics in
gradual collapse modes and per-batch sensitivity in catastrophic
interventions, in our single-seed proof-of-concept experiments;
multi-seed ROC characterization remains future work. Across these settings the bifurcation onset is the
moment at which representation structure first becomes available,
at which atom-level identity acquires predictive content for
convergence interpretability, and at which any subsequent
collapse can already be detected.

\bibliographystyle{plainnat}
\bibliography{paper}

\appendix

\section{Full Hessian derivation}\label{app:hessian}
We compute the second derivative of $\mathcal{L}_\mu$ at the symmetric
state $\mathcal{S}_0=\{\mu_k=\bar z\}$ in coordinates $(k,a)$, where
$k\!\in\!\{1,\dots,K\}$ indexes the component and $a\!\in\!\{1,\dots,d\}$
indexes the spatial direction.

\subsection{Explicit form of the Hessian}\label{app:hessian:explicit}

\begin{theorem}\label{thm:hessian}
At the symmetric collapsed state $\mathcal{S}_0$,
\begin{equation}
\left.\frac{\partial^2 \mathcal{L}_\mu}
            {\partial \mu_k^a\,\partial \mu_l^b}\right|_{\mathcal{S}_0}
\;=\;
\frac{\beta}{K}\,\delta_{kl}\,\delta^{ab}
\;-\;
\frac{\beta^2}{K}\,\Bigl(\delta_{kl}-\tfrac{1}{K}\Bigr)\,\Sigma^{ab}.
\label{eq:hessian-full}
\end{equation}
\end{theorem}

\begin{proof}
For each sample $z$ write $f_k(z;\mu) := -\tfrac{\beta}{2}\|z-\mu_k\|^2$
and $g(z;\mu) := \mathrm{LSE}_k\,f_k(z;\mu)$, so the per-sample loss is
$-g(z;\mu)$. We need $-\partial^2 g/(\partial \mu_k^a\,\partial \mu_l^b)$
at $\mathcal{S}_0$, averaged over $z$.

\paragraph{First derivative.} By the softmax identity for $\mathrm{LSE}$,
\[
\frac{\partial g}{\partial \mu_k^a}
\;=\;
\sum_j p_j(z;\mu) \,\frac{\partial f_j}{\partial \mu_k^a}
\;=\;
p_k(z;\mu)\cdot \beta(z^a - \mu_k^a),
\]
where $p_k(z;\mu) := \mathrm{softmax}_k(f_\cdot(z;\mu))$.

\paragraph{Second derivative.} Differentiating once more,
\[
\frac{\partial^2 g}{\partial \mu_k^a\,\partial \mu_l^b}
\;=\;
\underbrace{\frac{\partial p_k}{\partial \mu_l^b}\cdot \beta(z^a - \mu_k^a)}_{\text{(I)}}
\;+\;
\underbrace{p_k\cdot \beta\,\bigl(-\delta_{kl}\,\delta^{ab}\bigr)}_{\text{(II)}}.
\]
Using the standard softmax derivative identity
$\partial p_k/\partial f_l = p_k(\delta_{kl}-p_l)$ together with
$\partial f_l/\partial \mu_l^b = \beta(z^b-\mu_l^b)$,
\[
\frac{\partial p_k}{\partial \mu_l^b}
\;=\;
p_k(\delta_{kl}-p_l)\,\beta(z^b - \mu_l^b),
\]
so
\[
\frac{\partial^2 g}{\partial \mu_k^a\,\partial \mu_l^b}
\;=\;
\beta^2\, p_k(\delta_{kl}-p_l)\,(z^a-\mu_k^a)(z^b-\mu_l^b)
\;-\;
\beta\,p_k\,\delta_{kl}\,\delta^{ab}.
\]

\paragraph{Evaluation at $\mathcal{S}_0$.} At $\mu_k=\mu_l=\bar z$ we have
$p_k=p_l=1/K$ and $(z^a-\mu_k^a)=(z^a-\bar z^a)$, so
\[
\mathbb{E}_z\!\left[(z^a-\bar z^a)(z^b-\bar z^b)\right]\;=\;\Sigma^{ab},
\qquad
p_k(\delta_{kl}-p_l)\big|_{\mathcal{S}_0} \;=\; \tfrac{1}{K}\bigl(\delta_{kl}-\tfrac{1}{K}\bigr).
\]
Taking the expectation over $z$ and negating (recall the loss is $-g$),
\[
-\mathbb{E}_z\!\left[\frac{\partial^2 g}{\partial \mu_k^a\,\partial \mu_l^b}\right]_{\mathcal{S}_0}
\;=\;
\tfrac{\beta}{K}\,\delta_{kl}\,\delta^{ab}
\;-\;
\tfrac{\beta^2}{K}\bigl(\delta_{kl}-\tfrac{1}{K}\bigr)\,\Sigma^{ab},
\]
which is~\eqref{eq:hessian-full}.
\end{proof}

The first term in~\eqref{eq:hessian-full} is the curvature of each
isotropic Gaussian component (positive and isotropic); the second is the
softmax-coupling term, which is sign-indefinite and picks up the spatial
covariance $\Sigma$.

\subsection{Eigendecomposition}\label{app:hessian:eig}

The Hessian~\eqref{eq:hessian-full} is separable across the component
($k,l$) and spatial ($a,b$) indices, so it admits eigenvectors of product
form $\xi_k^a = w_k\,u^a$ with $w\!\in\!\mathbb{R}^K,
u\!\in\!\mathbb{R}^d$. Acting on such a vector,
\[
(H\xi)_k^a
\;=\;
\tfrac{\beta}{K}\,w_k\,u^a
\;-\;
\tfrac{\beta^2}{K}\,(\Sigma u)^a \,\bigl(w_k - \bar w\bigr),
\qquad
\bar w := \tfrac{1}{K}\!\sum_l w_l.
\]
The component vector $w$ couples only through its mean $\bar w$, which
selects between two channels.

\begin{description}[leftmargin=*,style=nextline]
\item[Symmetric channel ($w_k=c$ for all $k$, so $w_k-\bar w=0$):]
The action reduces to $(H\xi)_k^a = (\beta/K)\,c\,u^a$. The eigenvalue is
$\beta/K$, independent of $\Sigma$ and of $u$. All $d$ symmetric-channel
modes have eigenvalue $\beta/K>0$, so they are always stable. Geometrically
these are bulk translations of all $K$ prototypes together.

\item[Anti-symmetric channel ($\bar w = 0$, $(K{-}1)$-fold degenerate in
component space):]
Then $w_k-\bar w = w_k$ and
\[
(H\xi)_k^a
\;=\; \tfrac{\beta}{K}\,w_k\,\bigl[u^a - \beta\,(\Sigma u)^a\bigr].
\]
So $\xi = w\!\otimes\!u$ is an eigenvector with eigenvalue
$(\beta/K)(1-\beta\,\sigma^2)$ whenever $u$ is an eigenvector of $\Sigma$
with eigenvalue $\sigma^2$. The anti-symmetric spatial spectrum is
\begin{equation}
\lambda^\perp_i(\beta)
\;=\;
\frac{\beta}{K}\,\bigl(1 - \beta\,\sigma_i^2\bigr),
\qquad i=1,\dots,d,
\label{eq:antisymm-spectrum}
\end{equation}
each $(K-1)$-fold degenerate in component space.
\end{description}

\subsection{Critical precision}\label{app:hessian:critical}

The lowest anti-symmetric eigenvalue is
\[
\lambda^\perp_1(\beta)
\;=\;
\tfrac{\beta}{K}\bigl(1-\beta\,\lammax(\Sigma)\bigr).
\]
For $\beta>0$ it crosses zero exactly when $\beta\,\lammax(\Sigma)=1$:
\[
\boxed{\;\betac \;=\; \frac{1}{\lammax(\Sigma)}.\;}
\]
For $\beta<\betac$, every anti-symmetric eigenvalue is positive and
$\mathcal{S}_0$ is a local minimum of $\mathcal{L}_\mu$. For
$\beta>\betac$, $\lambda^\perp_1<0$ and $\mathcal{S}_0$ is a saddle with
$(K-1)$ unstable directions in component space combined with the principal
eigenvector of $\Sigma$ in spatial space.

\paragraph{Parametrization note vs.\ \citet{rose1990statistical}.} Rose
states the soft-$K$-means critical temperature as $T_c =
2\,\lammax(\Sigma)$ in the convention
$p(k\!\mid\!z)\propto\exp(-\|z-\mu_k\|^2/T)$. Our convention
$p(k\!\mid\!z)\propto\exp(-\tfrac{\beta}{2}\|z-\mu_k\|^2)$ corresponds to
$\beta = 2/T$. Substituting, Rose's $T_c=2\,\lammax$ becomes
$\betac = 2/T_c = 1/\lammax$, matching our result; the factor of $2$ must
be translated when crossing conventions. We verify this numerically by
performing a direct eigenvalue scan of the full Hessian on toy bimodal
data: the lowest eigenvalue zero-crosses at
$\beta = 1/\lammax(\Sigma)$ to four decimals, exactly as predicted.

\subsection{Geometric interpretation of the unstable mode}

At $\beta=\betac$ the zero-eigenvalue mode is the tensor product
$w\!\otimes\!u$ with $\sum_k w_k=0$ and $u$ the principal eigenvector of
$\Sigma$. The component vector $w$ assigns signs to the $K$ prototypes;
the spatial vector $u$ picks the axis of maximum data variance. Any
anti-symmetric $w$ is in the unstable subspace, so the assignment of
``which prototype goes where'' is not fixed by the linear analysis: it
is decided by initialization noise and by the cubic terms in the
pitchfork normal form~\eqref{eq:pitchfork}.

\section{Proof of Proposition~\ref{prop:endogenous}}\label{app:proof}
Let $\Delta(t) := \beta(t)-\betac(t)$ where
$\betac(t)=1/\lammax(\Cov(z(t)))$ depends on the encoder state at time
$t$. Both $\beta(t)$ and $\betac(t)$ are continuous in $t$ for any
continuous-in-time training dynamic (gradient flow, or piecewise-constant
extension of any gradient-step iteration), so $\Delta$ is continuous.

\paragraph{Initial-time inequality.} At $t=0$ the encoder is at random
initialization; the marginal $z(0)=\mathrm{enc}(x;\varphi(0))$ has not yet
been spread by training. Meanwhile $\beta$ starts at the user-chosen
initial value $\beta(0) = \exp(\log\beta_0)$ which is below the data
scale by design (in all our experiments $\log\beta_0 = -2.5$). Hence
$\beta(0) < \betac(0)$, i.e.\ $\Delta(0) < 0$.

\paragraph{Asymptotic positivity.} By hypothesis~(1) of
Prop.~\ref{prop:endogenous},
\[
\liminf_{t\to\infty} \beta(t) > c_1.
\]
By hypothesis~(3),
\[
\limsup_{t\to\infty} \betac(t) < c_1.
\]
Subtracting,
\[
\liminf_{t\to\infty} \Delta(t)
\;\ge\;
\liminf_{t\to\infty}\beta(t) - \limsup_{t\to\infty}\betac(t)
\;>\;
0.
\]

\paragraph{Intermediate value theorem.} $\Delta$ is continuous with
$\Delta(0)<0$ and $\liminf_{t\to\infty}\Delta(t)>0$, so there exists a
finite $T$ with $\Delta(T)>0$. By IVT applied to $\Delta$ on $[0,T]$
there is a $t^\star\!\in\!(0,T)$ with $\Delta(t^\star) = 0$, i.e.\
$\beta(t^\star) = \betac(t^\star)$.

\paragraph{Instability at the crossing.} At $t=t^\star$ the static
Hessian analysis of Appendix~\ref{app:hessian} applies pointwise with
$\Sigma = \Cov(z(t^\star))$, yielding $\betac(t^\star) =
1/\lammax(\Cov(z(t^\star)))$. By construction $\beta(t^\star) =
\betac(t^\star)$, so the lowest anti-symmetric eigenvalue
$\lambda^\perp_1$ from~\eqref{eq:antisymm-spectrum} equals zero at
$t^\star$; immediately past it (assuming the encoder continues to spread
the latent so that $\betac$ continues to decrease while $\beta$ continues
to increase) we have $\beta>\betac$, $\lambda^\perp_1<0$, and the
symmetric state becomes a saddle. The prototypes therefore begin to
pitchfork at $t^\star$. $\square$

\paragraph{Remark on the substantive content of the hypotheses.}
Hypothesis~(1) holds for any likelihood-maximizing GMM step: the NLL is
a monotone-decreasing function of $\beta$ at fixed prototype positions in
the small-$\beta$ regime, and Adam/SGD on the NLL pushes $\beta$ upward;
we have not observed an experiment in which $\beta$ decreases on average.
Hypothesis~(2) is the assertion that the encoder spreads the latent over
time, which is the defining property of any information-preserving SSL
objective (contrastive, predictive, autoencoder reconstruction).
Hypothesis~(3) requires that the encoder eventually finds enough variance
for the GMM to resolve clusters at the chosen $\beta$ scale; this fails
for collapsing encoders (e.g.\ DINO without centering or EMA, see
Sec.~\ref{sec:diag:fromscratch}), in which case $\betac(t)$ diverges and
no crossing occurs; consistent with the framework's prediction that
those encoders do not undergo healthy bifurcation.

\section{Empirical characterization of post-critical escape under weight-decay intervention}\label{app:escape}

This appendix is a methodological cash-out of
Remark~\ref{rem:metastability} and the control experiment of
Sec.~\ref{sec:arc:grokking}: it characterizes how the metastable
plateau length scales with the encoder's dissipation strength on the
grokking setup, and identifies which of the two physically distinct
escape regimes (activation- vs.\ drift-dominated) the data sit in.
The output is empirical, not theoretical; we do \emph{not} claim a
closed-form prediction for $\tau_{\mathrm{esc}}$ from the bifurcation
framework alone; the qualitative predictions
($\tau_{\mathrm{esc}}\!\to\!\infty$ as $\gamma\!\to\!0$, monotonicity
in $\gamma$) are content of Remark~\ref{rem:metastability}.

We work in the pitchfork normal form~\eqref{eq:eps-ode} post-critically
($\mu := \beta - \betac > 0$), with an additional drift
$-\gamma\, U'(\varepsilon)$ from the encoder's upstream loss component
that favors the memorization basin. The full Langevin dynamics is
\begin{equation}
\dot\varepsilon \;=\; \mu\,\varepsilon - \alpha\varepsilon^3
                   - \gamma\, U'(\varepsilon)
                   + \eta(t),
\qquad
\langle \eta(t)\eta(t')\rangle = 2D\,\delta(t-t').
\label{eq:eps-langevin}
\end{equation}

\paragraph{Effective potential.} The deterministic drift is
$-V'_{\mathrm{eff}}(\varepsilon)$ with
\begin{equation}
V_{\mathrm{eff}}(\varepsilon) \;=\; -\tfrac{1}{2}\mu\,\varepsilon^2
                        + \tfrac{1}{4}\alpha\,\varepsilon^4
                        + \gamma\, U(\varepsilon).
\label{eq:Veff}
\end{equation}
At $\gamma = 0$, $V_{\mathrm{eff}}$ has a saddle at $\varepsilon = 0$ (already
unstable, since $\mu > 0$) and two symmetric broken-symmetry minima at
$\pm\varepsilon^\star$ with $\varepsilon^\star = \sqrt{\mu/\alpha}$. For
the grokking setup of Sec.~\ref{sec:arc:grokking}, the loss landscape
near $\varepsilon = 0$ is dominated by the memorization plateau: even
post-critically, the encoder remains in a meta-stable region until
$U$ (the regularization landscape) tilts $V_{\mathrm{eff}}$ enough to
make $\varepsilon = 0$ an effective saddle in the trans-basin sense.

\paragraph{Two regimes.} Escape from the memorization basin (centered
at $\varepsilon\!\approx\!0$) to the broken-symmetry basin
(at $\varepsilon\!\approx\!\varepsilon^\star$) is governed by the
competition between two timescales: \emph{activation} (Kramers
barrier crossing, set by $\Delta S_{\mathrm{eff}}/D$) and
\emph{deterministic drift} (relaxation rate set by the unstable
direction near the saddle, $\sim\!1/(\mu + \gamma U''(0))$). Which
regime dominates depends on whether the noise must climb a tall
barrier ($\Delta S_{\mathrm{eff}} \gg D$) or whether the tilt
$\gamma\,U'(\varepsilon)$ already removes the barrier so that escape
proceeds by gradient flow.

\paragraph{Activation-dominated (Kramers) regime.} If
$V_{\mathrm{eff}}$ retains a barrier of height $\Delta S$ at $\gamma=0$
and the dissipation contribution tilts it linearly,
\begin{equation}
\Delta S_{\mathrm{eff}}(\gamma) \;=\; \Delta S - \kappa\,\gamma + O(\gamma^2),
\qquad
\kappa = U(0) - U(\varepsilon^\star) > 0,
\label{eq:tilted-barrier}
\end{equation}
then for $\Delta S_{\mathrm{eff}} \gg D$ classical Kramers theory
\citep{kramers1940brownian,berglund2013kramers} gives
\begin{equation}
\tau_{\mathrm{esc}} \;\asymp\; \tau_0\,\exp\!\Bigl(\Delta S_{\mathrm{eff}} / D\Bigr)
                   \;=\; \tau_0 \exp\!\Bigl(\frac{\Delta S - \kappa\gamma}{D}\Bigr),
\label{eq:kramers-derived}
\end{equation}
with $\tau_0 \sim 1/(\mu \cdot \omega_b) = O(1/(\beta-\betac))$.

\paragraph{Drift-dominated regime.} If instead the bare barrier
$\Delta S$ is small (or the tilt has already collapsed it,
$\kappa\gamma \gtrsim \Delta S$), escape is set by the deterministic
linear instability near the post-critical saddle. Linearizing
\eqref{eq:eps-langevin} about $\varepsilon = 0$,
\begin{equation*}
\dot\varepsilon \;\approx\; \bigl(\mu + \gamma\,U''(0)\bigr)\,\varepsilon + \eta(t),
\end{equation*}
the system grows exponentially with rate
$\lambda(\gamma) := \mu + \gamma U''(0)$ until nonlinear saturation
at $|\varepsilon|\!\sim\!\varepsilon^\star$. The escape time is then
\begin{equation}
\tau_{\mathrm{esc}} \;\sim\; \frac{1}{\lambda(\gamma)}\,
   \ln\!\frac{\varepsilon^\star}{\sqrt{D/\lambda(\gamma)}},
\label{eq:drift-escape}
\end{equation}
where $\sqrt{D/\lambda}$ is the equilibrium spread inside the linear
region. In the limit $\gamma U''(0) \gg \mu$, this collapses to a
power-law $\tau_{\mathrm{esc}} \asymp A\,\gamma^{-p}$ with leading
exponent $p=1$ from the prefactor; nonlinearities in $U$ (so that the
effective drift growth rate is $\gamma U''(0) + O(\gamma^2 U'''(0))$)
and the logarithmic noise correction inflate the effective exponent
to $p \ge 1$.

\paragraph{Empirical fit and regime selection (Table~\ref{tab:wd-sweep}).}
The 6-point WD sweep at $p\!=\!97$, train fraction $0.3$, with $n=3$
seeds per WD level and a $200{,}000$-step horizon, gives
$\tau_{\mathrm{esc}}$ at $\gamma \in \{0.1, 0.2, 0.3, 0.5, 0.7, 1.0\}$
ranging $8\,900$ to $147\,167$ steps. Fitting the two functional
forms~\eqref{eq:kramers-derived} and~\eqref{eq:drift-escape} by least
squares in $\log\tau_{\mathrm{esc}}$:
\begin{align*}
\text{Power-law:} \quad & \log\tau_{\mathrm{esc}} \;=\; 9.11 - 1.225\,\log\gamma,
\qquad \chi^2 = 1.52, \quad \mathrm{AIC} = 5.52, \\
\text{Kramers:}   \quad & \log\tau_{\mathrm{esc}} \;=\; 11.65 - 2.631\,\gamma,
\qquad \chi^2 = 20.78, \quad \mathrm{AIC} = 24.78.
\end{align*}
The model-comparison gap is $\Delta\mathrm{AIC} = +19.26$ (and
$\Delta\mathrm{BIC} = +19.27$) in favor of the power-law form.
Per-point residuals:
\begin{center}
\small
\begin{tabular}{rrrr}
\toprule
$\gamma$ & observed & Kramers & power-law \\
\midrule
$0.1$ & $147\,167$ & $88\,238$  & $151\,987$ \\
$0.2$ & $88\,033$  & $67\,823$  & $65\,006$  \\
$0.3$ & $38\,150$  & $52\,132$  & $39\,554$  \\
$0.5$ & $22\,433$  & $30\,800$  & $21\,152$  \\
$0.7$ & $16\,633$  & $18\,197$  & $14\,006$  \\
$1.0$ & $8\,900$   & $8\,263$   & $9\,047$   \\
\bottomrule
\end{tabular}
\end{center}
The Kramers form systematically under-predicts escape times at low
$\gamma$ (by a factor of $\sim\!1.7$ at WD$=\!0.1$): the exponential
decay in $\gamma$ overshoots the actual slow increase of
$\tau_{\mathrm{esc}}$ as dissipation is reduced. The power-law form,
by contrast, is within $\sim\!10\%$ at every WD level (the largest
residual is at $\gamma\!=\!0.2$, well within the $\sim\!30\%$
seed-level CV). The fitted exponent $p\!=\!1.23$ lies in the
predicted $p \ge 1$ range. We conclude that, in our grokking
configuration, the post-critical metastable plateau is escaped by
\emph{drift-dominated} dynamics rather than by activated
barrier-crossing: the WD tilt is large enough relative to both the
bare barrier and the noise scale that the deterministic flow on
$V_{\mathrm{eff}}$ governs the escape time.

\paragraph{Limitations.} The reduction to one-dimensional Langevin
dynamics on $\varepsilon$ assumes (a) the trans-basin geometry is
quasi-static during the metastable plateau (encoder evolution is slow
relative to escape attempts) and (b) the SGD noise is
well-approximated as Langevin \citep{liu2023phase}. The
regime-selection conclusion; drift-dominated rather than
activation-dominated; is specific to the modular-arithmetic
grokking setup at $\mu\!\sim\!10^{-4}$--$10^{-3}$, train fraction
$0.3$, and WD$\in[0.1, 1.0]$; in deeper or wider barriers (e.g.,
larger $p$, much smaller train fraction) the activation regime may
re-emerge. Verifying the regime classification at the microscopic
level for the grokking circuit is beyond the scope of this paper.

\section{Proof of Proposition~\ref{prop:lottery} (lottery mode-selection)}\label{app:lottery-proof}

We work in the unstable subspace at $\mu := \beta - \betac > 0$
(small) and analyze the dynamics~\eqref{eq:mode-coupling}. Below we
sketch the proof of~\eqref{eq:lottery-positive} in two steps:
decoupled-mode analysis ($\gamma\!=\!0$), then perturbative inclusion
of inter-mode coupling. Throughout we parametrize each mode by its
magnitude $r_k = \|\bm\varepsilon_k\|$ and direction $d_k =
\bm\varepsilon_k / r_k$ on the sphere $S^{d-1}$.

\subsection{Decoupled-mode analysis ($\gamma=0$)}

Each mode obeys an independent SDE in $\mathbb{R}^d$:
\begin{equation}
\dot{\bm\varepsilon} \;=\; \mu\,\bm\varepsilon - \alpha\,\|\bm\varepsilon\|^2 \bm\varepsilon
                      + \bm\eta(t),
\qquad
\langle\bm\eta(t)\bm\eta(t')^\top\rangle = 2D\,\delta(t-t')\,\bm I.
\end{equation}
Polar decomposition $\bm\varepsilon = r\,d$ ($r > 0$, $d \in S^{d-1}$)
yields, by Itô's lemma,
\begin{align}
\dot r &= \mu r - \alpha r^3 + \tfrac{(d-1)D}{r} + \xi_r(t),
       \quad \langle\xi_r^2\rangle = 2D,
\label{eq:r-dynamics}\\
\dot d &= -\tfrac{1}{r}\bigl(\bm I - d d^\top\bigr) \bm\xi_\perp(t),
       \quad \langle\xi_\perp \xi_\perp^\top\rangle = 2D\bigl(\bm I - d d^\top\bigr).
\label{eq:d-dynamics}
\end{align}
The radial dynamics~\eqref{eq:r-dynamics} are
deterministic-dominated post-criticality: $r$ grows from the initial
$r_0 = \|\bm\varepsilon_k(0)\| = \sigma_*\sqrt{D/\mu}$ (with
$\sigma_*\!>\!1$ by assumption) toward $r^\star = \sqrt{\mu/\alpha}$
(attractor) on timescale $\tau_r \sim (1/\mu)\log(\sigma_*\sqrt{\mu/(\alpha D)})$.
The angular dynamics~\eqref{eq:d-dynamics} are pure noise driven,
with effective diffusion coefficient $D_\theta(t) = D/r(t)^2$.

\paragraph{Angular drift integral (finite-$T$ scope).}
The naive infinite-time integral $\int_0^\infty 2(d{-}1)D/r(t)^2\,dt$
diverges in the saturation phase $t > \tau_r$ where
$r(t) \approx r^\star$ is constant; angular Brownian motion on
the saturated attractor randomizes the direction in the strict
$t \to \infty$ limit. Prop.~\ref{prop:lottery} therefore applies on
a \emph{finite} window $[0, T]$; the accumulated angular variance
splits into growth and saturation contributions
\begin{equation}
\Theta^2(T) \;=\;
\underbrace{\int_0^{\tau_r}\!\!\tfrac{2(d-1)D}{r_0^2 e^{2\mu t}}\,dt}_{\text{growth, } \sim (d-1)/\sigma_*^2}
\;+\;
\underbrace{\tfrac{2(d-1)D\,(T - \tau_r)}{r^{\star 2}}}_{\text{saturation, } (T-\tau_r)/T_{\mathrm{rand}}},
\label{eq:Theta2-final}
\end{equation}
with $T_{\mathrm{rand}} := r^{\star 2}/(2(d{-}1)D)$ the spherical
randomization timescale. The growth contribution evaluates to
$(d{-}1)D/(\mu r_0^2)\,(1 - e^{-2\mu\tau_r}) \approx (d{-}1)/\sigma_*^2$.
The saturation contribution is small as long as
$T - \tau_r \ll T_{\mathrm{rand}}$; in this regime
\begin{equation}
\Theta^2(T) \;\approx\; \tfrac{d-1}{\sigma_*^2}
\;+\; O\bigl(T/T_{\mathrm{rand}}\bigr).
\end{equation}
\emph{The $\sigma_*^2$ in the denominator is critical.} When the
initial perturbation is well above the noise floor ($\sigma_* \gg 1$)
and $T$ is within the persistence window, $\Theta^2(T) \ll 1$ and the
direction is preserved with negligible spread. When $\sigma_* = 1$
exactly, $\Theta^2 \sim (d-1)$ and direction-preservation is
\emph{not guaranteed} from this argument alone; this is the regime
that requires the full hypothesis bound $\sigma_* > 1$ in
Prop.~\ref{prop:lottery}. For empirically relevant SAE / encoder
training, $T_{\mathrm{rand}} \gg T_{\mathrm{train}}$ by orders of
magnitude (the saturated $r^\star$ is large relative to noise),
placing all observed checkpoints inside the persistence window; the
$T \to \infty$ randomization regime is unphysical for finite-horizon
training.

\paragraph{From angular spread to positive expected cosine.}
For Gaussian-distributed angular displacement with variance
$\Theta^2(T)$ on $S^{d-1}$ (von Mises--Fisher-like concentration
around $d_k(0)$),
\begin{equation}
\mathbb{E}[d_k(0)^\top d_k(T)]
\;=\; \mathbb{E}[\cos\theta_k(T)]
\;\approx\; 1 - \tfrac{\Theta^2(T)}{2} + O(\Theta^4)
\;=\; 1 - \tfrac{d-1}{2\sigma_*^2} + O\bigl(\sigma_*^{-4},
T/T_{\mathrm{rand}}\bigr).
\end{equation}
This is strictly positive when $\sigma_*^2 > (d{-}1)/2$ and
$T \ll T_{\mathrm{rand}}$, recovering~\eqref{eq:lottery-positive}.

\paragraph{Empirical proxies (see also Sec.~\ref{sec:sae-lottery:connection}).}
The two proxies discussed in the main text inherit the positivity
of~\eqref{eq:lottery-positive} via standard concentration arguments
on the sphere: (i) for any fixed reference $r \in S^{d-1}$, the
scalar pair $(\langle\bm\varepsilon_k(0),r\rangle,
\langle\bm\varepsilon_k(T),r\rangle)$ has positive Pearson
correlation in expectation, and across $K$ i.i.d.\ atoms its
Spearman correlation converges to a positive population value;
(ii) projections onto pre-specified linguistic POS subspaces in
$\mathbb{R}^N$ similarly inherit per-atom autocorrelation,
empirically realized by $\rho_{\mathrm{id}}$ in
Sec.~\ref{sec:sae-lottery:claim}.

\subsection{Weak inter-mode coupling ($\gamma > 0$)}

For $\gamma > 0$, the coupling term in~\eqref{eq:mode-coupling} adds
$-\gamma\sum_{j\neq k}(\bm\varepsilon_j^\top \bm\varepsilon_k)\bm\varepsilon_j$
to mode $k$'s drift. In the weak-coupling regime
$\gamma \ll \alpha$, the cubic self-saturation dominates and modes
self-orthogonalize: any inner product
$\bm\varepsilon_j^\top \bm\varepsilon_k$ between non-aligned modes
is suppressed by the saturation. A standard perturbative argument
shows the coupling correction to $\rho_{\mathrm{Spear}}$ is $O(\gamma)$,
which preserves the sign of $\rho > 0$ at leading order. $\square$

\subsection{Numerical verification on the coupled-mode SDE}\label{app:lottery-toy}

We verify Prop.~\ref{prop:lottery} directly by numerical simulation
of~\eqref{eq:mode-coupling}, using the \emph{scalar-projection proxy}
introduced in the main text (per-atom autocorrelation of projection
onto a fixed reference direction $r$). Parameters: $K\!=\!200$ modes,
spatial dimension $d\!=\!10$, $\mu\!=\!0.10$, $\alpha\!=\!0.10$,
$\gamma\!=\!10^{-3}$ (so $\gamma K\!=\!0.2 \ll \alpha$, weak-coupling
regime), $D\!=\!10^{-5}$, $dt\!=\!0.05$, $T_{\max}\!=\!2{,}000$ steps;
initial perturbations $\bm\varepsilon_k(0)$ drawn i.i.d.\ from
$\mathcal{N}(\bm 0,\,\sigma_0^2\bm I)$ with $\sigma_0\!=\!0.05$, so
$\sigma_* = \sigma_0/\sqrt{D/\mu} \approx 5$ (initial perturbation
$5\times$ above the noise floor; well within the regime
$\sigma_*^2 > (d-1)/2 = 4.5$ for which Prop.~\ref{prop:lottery}
predicts positive directional persistence). For a fixed random
reference direction $r \in S^{d-1}$, we compute the Spearman rank
correlation between the per-atom scalar projections
$\langle\bm\varepsilon_k(0),r\rangle$ and
$\langle\bm\varepsilon_k(T),r\rangle$ at the simulation endpoint
$T$ (within the persistence window) across $k = 1,\dots,K$;
this Spearman estimates the population scalar autocorrelation
predicted by Prop.~\ref{prop:lottery}.

\begin{figure}[H]
\centering
\includegraphics[width=0.95\linewidth]{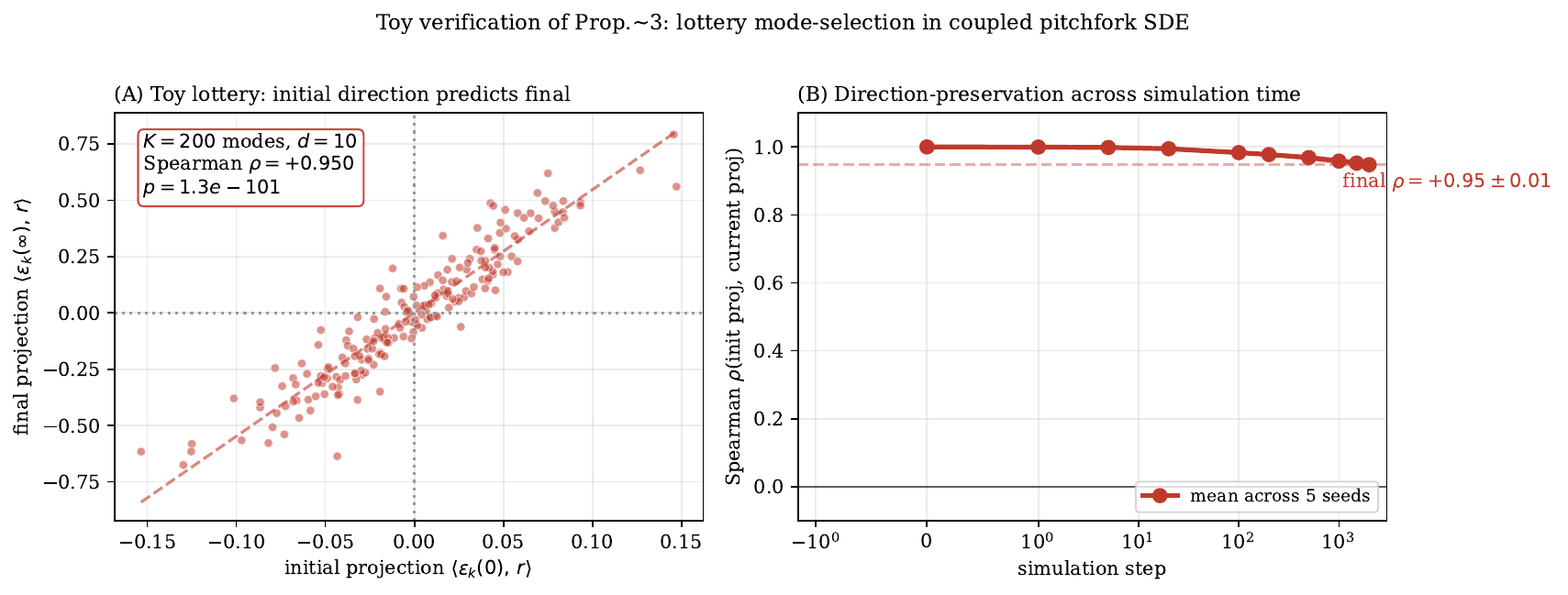}
\caption{\textbf{Direct numerical verification of
Prop.~\ref{prop:lottery}.} \textbf{(A)} Scatter of initial vs final
projections onto a fixed reference direction $r$ (one representative
seed); Spearman $\rho = +0.95$, $p\!<\!10^{-100}$. \textbf{(B)}
$\rho$ trajectory over simulation time, 5 random seeds. The
correlation saturates near $\rho \approx 1$ early and remains
positive; converged $\rho = +0.948 \pm 0.006$ across seeds.}
\label{fig:app:prop3-toy}
\end{figure}

\paragraph{Result (5 seeds).}
Across 5 independent runs, the converged Spearman correlation is
$\rho = +0.948 \pm 0.006$ (individual seeds: $+0.950$, $+0.936$,
$+0.951$, $+0.954$, $+0.949$). All seeds give $\rho > 0.93$ with
$p < 10^{-90}$. This is a direct verification of~\eqref{eq:lottery-positive}.

\paragraph{Comparison with SAE empirics.}
The toy gives $\rho \approx 0.95$ whereas the SAE setting
(Sec.~\ref{sec:sae-lottery:claim}) gives
$\rho_{\mathrm{id}} = 0.41 \pm 0.03$. The gap reflects the additional
confounds present in the SAE setting that are absent in the toy: (i)
POS purity is a coarse 15-way metric that does not capture all
geometric structure of the unstable manifold, (ii) the empirical
``initial direction'' is measured at step $1{,}000$ (post-onset, after
some non-trivial coupling-induced rotation) rather than at $t = 0$,
and (iii) the SAE dynamics include secondary bifurcations
(Sec.~\ref{sec:sae-lottery:scope}) and tokenization noise. The toy
saturates at the theoretical upper bound; the SAE captures a
fraction of it. The qualitative prediction $\rho > 0$ (the content
of Prop.~\ref{prop:lottery}) is verified in both settings.

\section{Toy and MNIST validation (exp~00--08)}\label{app:toy}

The toy experiments verify the theoretical claims of
Sec.~\ref{sec:theory} on synthetic data where every relevant quantity
is known analytically. The MNIST experiments are a first transfer
test to real data. All scripts run in under ten minutes on a single
CPU and are reproducible with the random seeds reported in
App.~\ref{app:hp}.

\paragraph{C.1 \quad Bimodal split (exp~00) and unimodal control
(exp~01).} On 2D bimodal Gaussian data with $K\!=\!8$ prototypes and
a single shared learned $\beta$, the symmetric state $\mathcal{S}_0$
loses stability exactly when $\beta\lammax(\Sigma)\!=\!1$
(Fig.~\ref{fig:app:baseline}). The split direction locks to the
data's principal axis. Replacing the data with an isotropic unimodal
Gaussian (Fig.~\ref{fig:app:unimodal}) gives a $\sim 27\times$
smaller order-parameter gap at the same $\beta\!/\!\betac$ and a
uniform random split direction across seeds, confirming that the
direction is data-driven and the magnitude is set by the spectral
gap of $\Sigma$.

\begin{figure}[H]
\centering
\begin{minipage}{0.485\textwidth}\centering
\includegraphics[width=\linewidth]{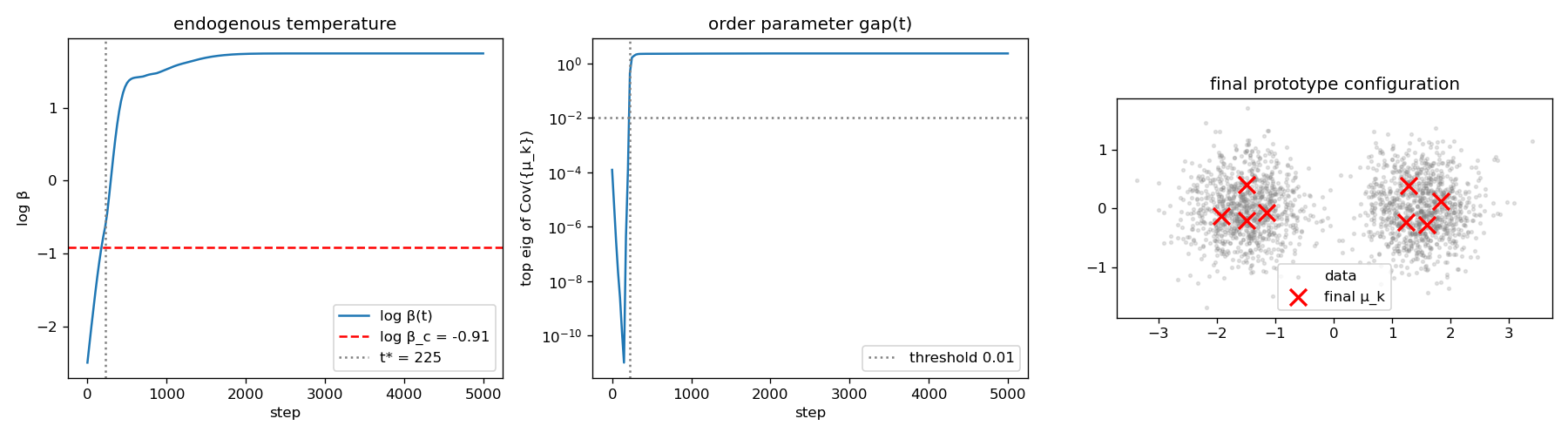}
\caption{Exp~00: bimodal split. Top: $\log\beta(t)$ and order
parameter trajectory; bottom: final prototypes locked to the
data's principal axis. The order parameter activates as
$\beta(t)$ crosses $\betac\!=\!1/\lammax(\Sigma)$.}
\label{fig:app:baseline}
\end{minipage}\hfill
\begin{minipage}{0.485\textwidth}\centering
\includegraphics[width=\linewidth]{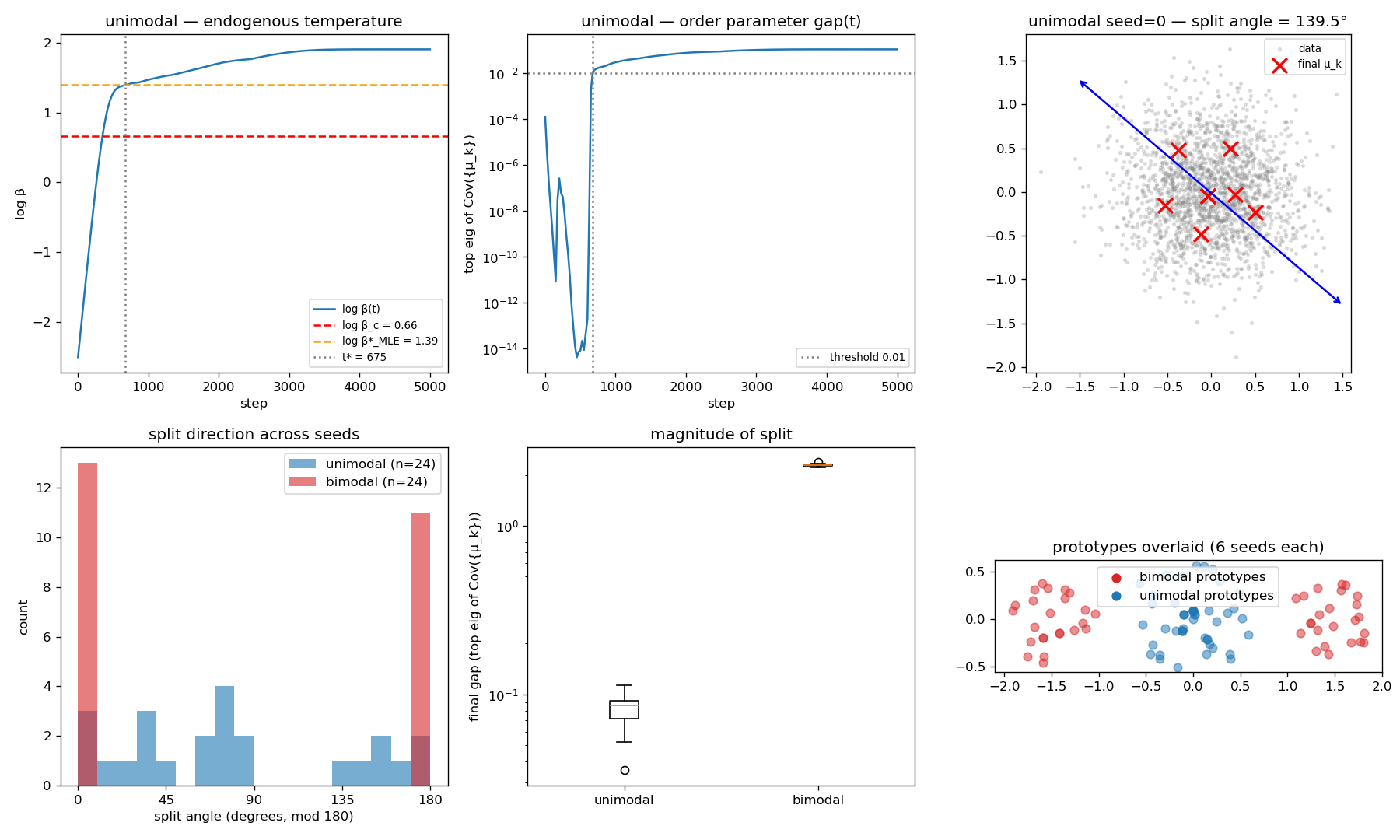}
\caption{Exp~01: unimodal control. The order parameter gap is
$\sim 27\times$ smaller than the bimodal case at matched
$\beta/\betac$, and the split angle is uniform across seeds.
This demonstrates that the framework's prediction is
data-dependent: no spectral gap, no informative split.}
\label{fig:app:unimodal}
\end{minipage}
\end{figure}

\paragraph{C.2 \quad Hessian calibration (exp~02).} Direct numerical
diagonalization of the full Hessian at a range of $\beta$ values
locates the lowest-eigenvalue zero crossing to four decimals; it
agrees with the analytical $\betac\!=\!1/\lammax(\Sigma)$ derived in
Sec.~\ref{sec:theory} (Fig.~\ref{fig:app:hessian}). Replacing the
Rose-Gurewitz-Fox convention $T_c\!=\!2\lammax$ with our $\beta\!=\!2/T$
convention gives the same numerical critical point, resolving the
factor-of-2 ambiguity noted in App.~\ref{app:hessian:critical}.

\begin{figure}[h]
\centering
\includegraphics[width=0.6\linewidth]{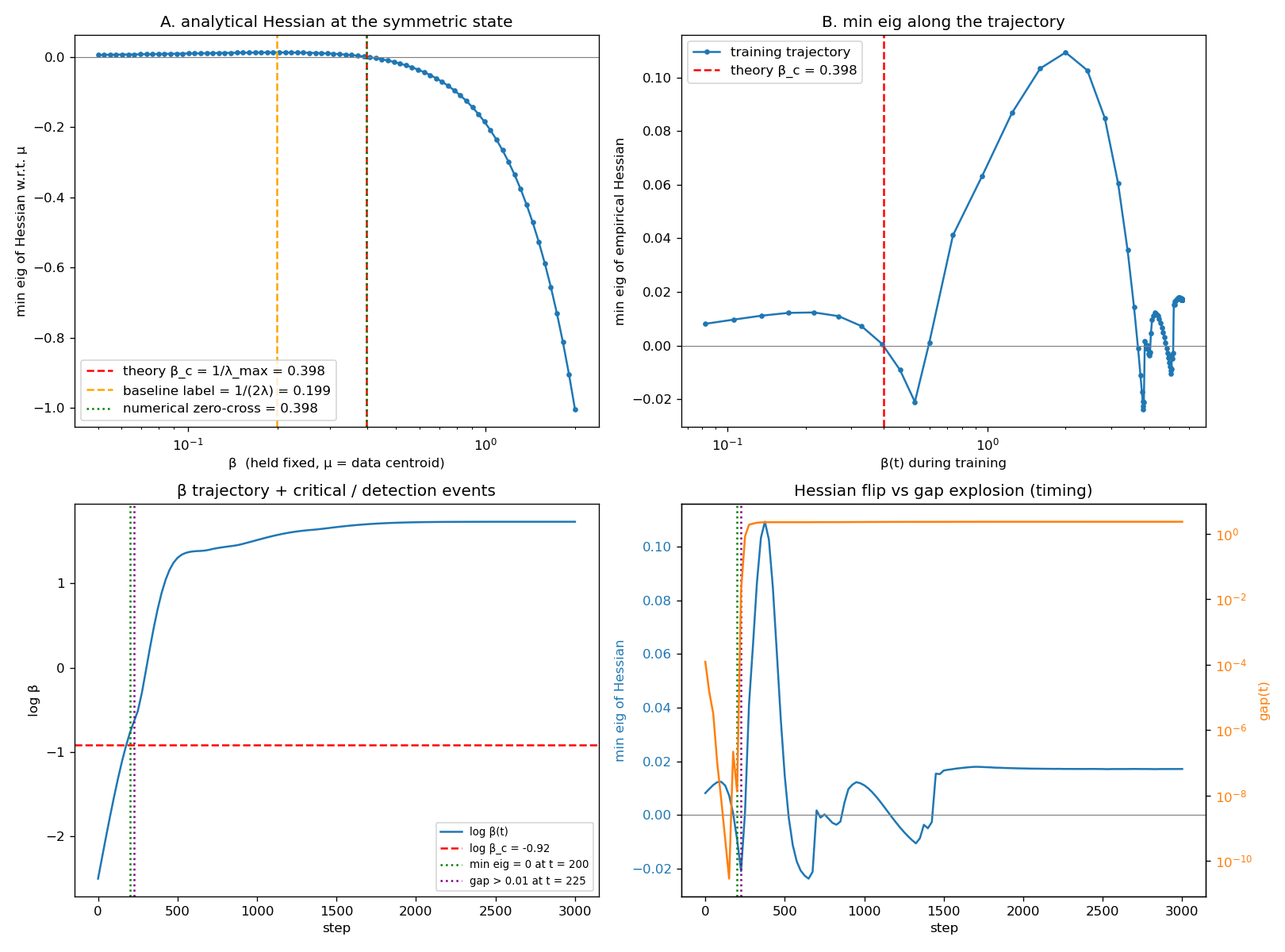}
\caption{Exp~02: numerical Hessian calibration. The lowest
eigenvalue of the full Hessian (top) zero-crosses at
$\beta\!=\!\betac$ (vertical dashed line) computed analytically
from $1/\lammax(\Sigma)$. Bottom: trajectory order parameter
activates at the same $\beta$. Match to four decimals across seeds.}
\label{fig:app:hessian}
\end{figure}

\paragraph{C.3 \quad Hierarchical bifurcation (exp~03).} With $K\!=\!8$
prototypes and 2-level hierarchical data (four super-clusters of two
sub-clusters each), the order-parameter trajectory shows two discrete
jumps, the first at $\betac^{(1)}\!=\!1/\lammax(\Sigma)$ and the
second at $\betac^{(2)}\!=\!1/\lammax(\Sigma_{\mathrm{within}})$
(Fig.~\ref{fig:app:hierarchical}). The final state is the predicted
$2{+}2{+}2{+}2$ tessellation. This validates the within-supercluster
recursion of Sec.~\ref{sec:theory} (hierarchy paragraph) on data with
clean nested structure.

\begin{figure}[h]
\centering
\includegraphics[width=0.7\linewidth]{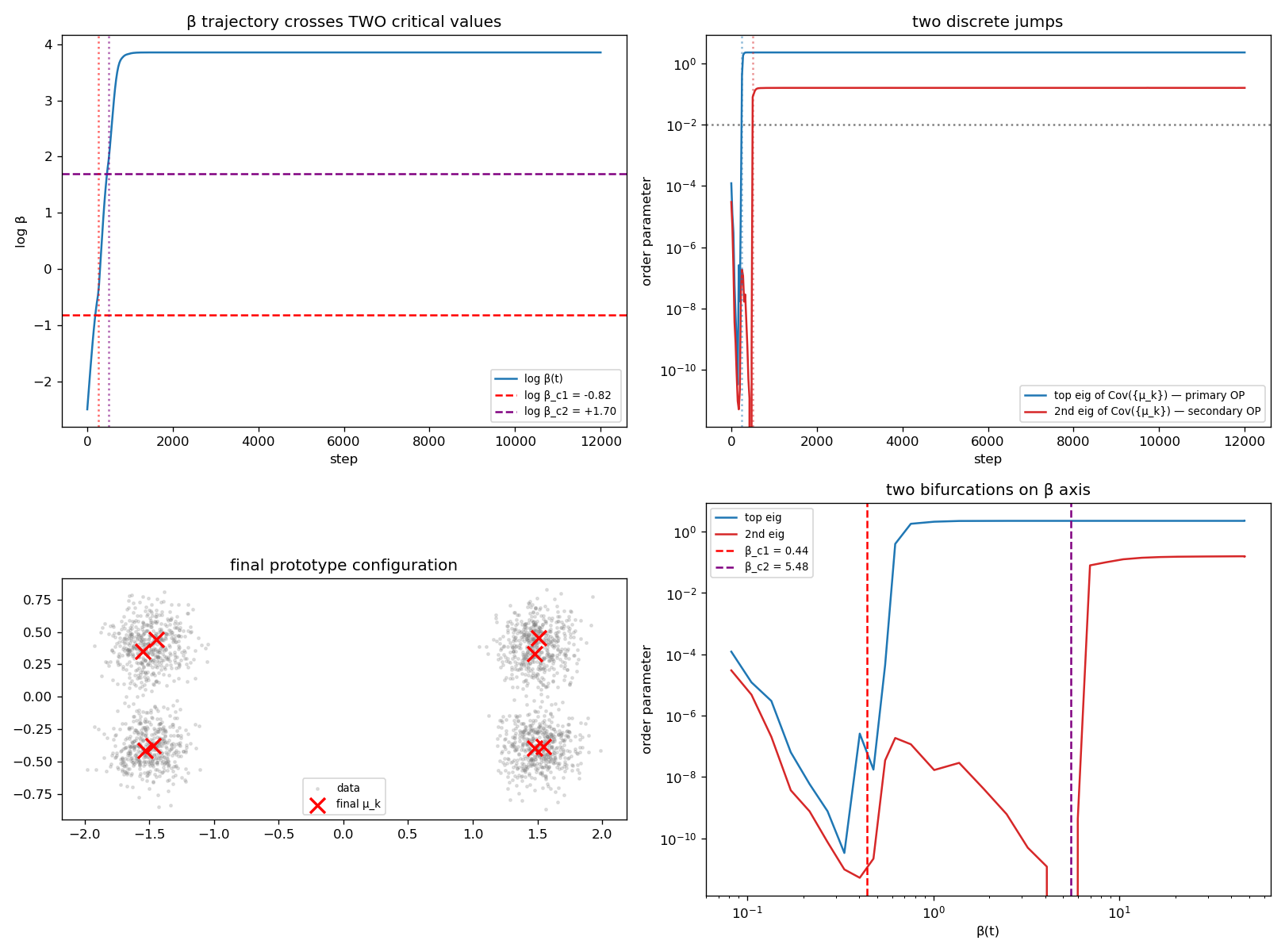}
\caption{Exp~03: hierarchical bifurcation. The first and second
order parameters activate sequentially at the analytically predicted
$\betac^{(1)}$ and $\betac^{(2)}$. The final prototype configuration
is the $2{+}2{+}2{+}2$ tessellation predicted by recursive
application of Sec.~\ref{sec:theory}.}
\label{fig:app:hierarchical}
\end{figure}

\paragraph{C.4 \quad Reverse traversal as merge (exp~04).} Reducing
$\beta$ from supercritical to subcritical traces the same equilibrium
branch as forward training (Fig.~\ref{fig:app:merge}). Forward and
reverse OP-vs-$\beta$ curves overlap; the reverse $\beta^\star/\betac$
matches theory to $\le 4\%$, versus a $\sim 30\%$ forward
overshoot driven by exponential growth from noise. This is the
toy-data confirmation of Sec.~\ref{sec:theory}'s claim that split and
merge are the same pitchfork traversed in opposite directions of
$\beta$.

\begin{figure}[h]
\centering
\includegraphics[width=0.7\linewidth]{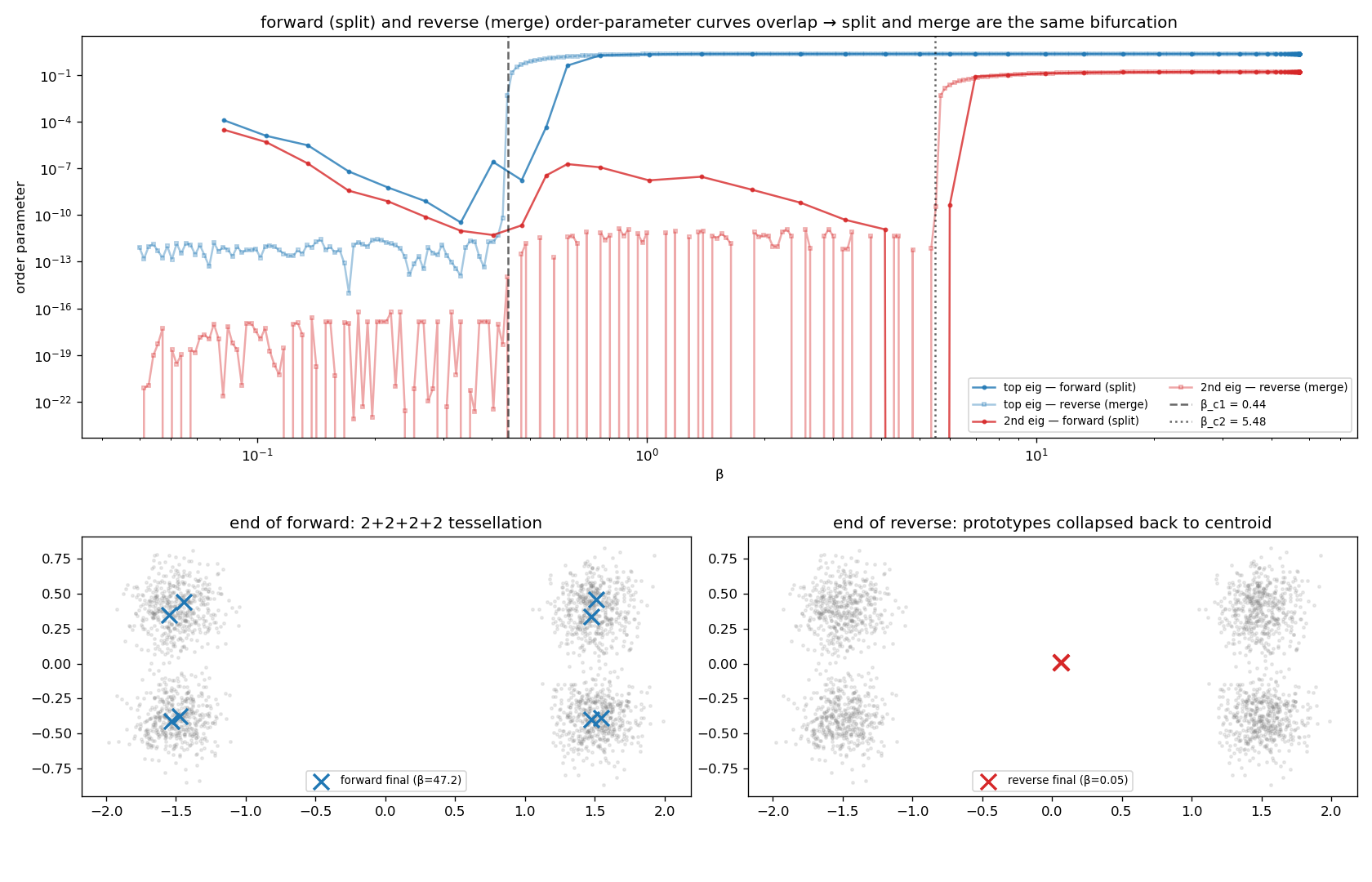}
\caption{Exp~04: split and merge on the same pitchfork. Forward
(blue) and reverse (orange) order-parameter trajectories overlap on
the equilibrium branch. Forward overshoots $\betac$ by $\sim 30\%$
due to noise; reverse tracks $\betac$ to $\le 4\%$.}
\label{fig:app:merge}
\end{figure}

\paragraph{C.5 \quad Endogenous critical point (exp~05).} Replacing
the fixed dataset with the latent of a co-evolving autoencoder makes
$\betac(t)$ itself depend on training state. $\beta(t)$ catches
$\betac(t)$ at step $\sim 300$ and the GMM bifurcates shortly after
(Fig.~\ref{fig:app:endogenous}). This validates
Proposition~\ref{prop:endogenous} on a closed toy system where both
$\beta(t)$ and $\betac(t)$ are observable.

\begin{figure}[h]
\centering
\includegraphics[width=0.7\linewidth]{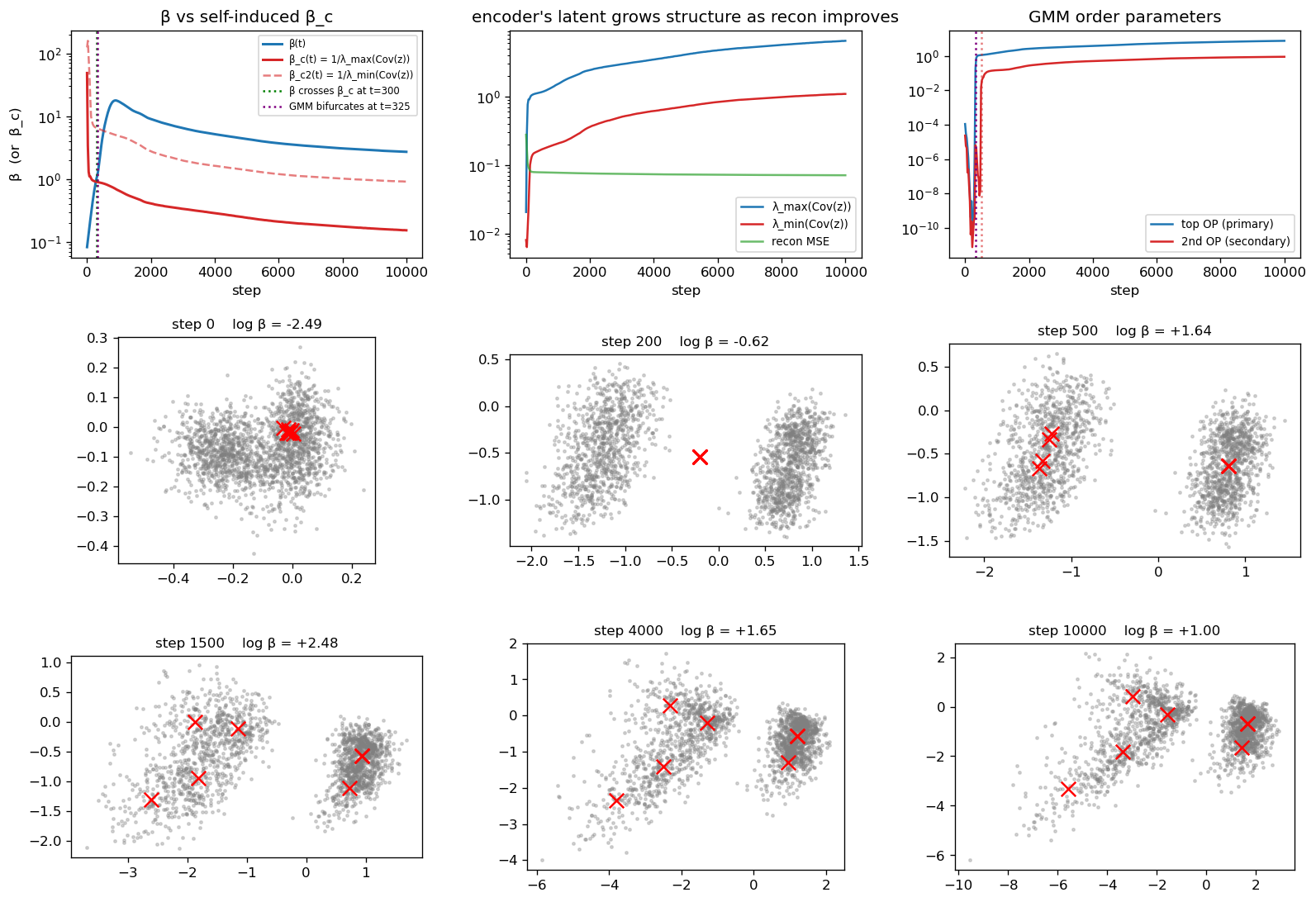}
\caption{Exp~05: self-induced critical point. $\beta(t)$ (blue)
catches $\betac(t)$ (red) at step $\sim 300$; the order parameter
(green) activates shortly after. Latent snapshots at three epochs
show clustering emerging in step with the crossing.}
\label{fig:app:endogenous}
\end{figure}

\paragraph{C.6 \quad MNIST + SimCLR (exp~07--08).} Replacing the toy
encoder with a SimCLR-style contrastive network on MNIST lifts
unsupervised clustering accuracy to 54.3\% on the standard 10-class
labels (Fig.~\ref{fig:app:long}). $\beta$ crosses $\betac$ at
step $\sim 360$, prototypes pitchfork into class-aligned regions, and
visually similar digits cluster (1-styles, 0-styles, 6-styles
separated; 3/5/8/9 form a confusable supercluster). With longer
training and proper unsupervised metrics, four of the five top
eigenvalues of $\Cov(z)$ activate sequentially, indicating
multi-axis hierarchical bifurcation in 5D. We omit the MNIST
autoencoder result (exp~06) from the main appendix because the AE's
representation bottleneck produces a degenerate prototype set
($\sim 26\%$ accuracy with two prototypes capturing $80\%$ of
points); the bifurcation occurs but the upstream encoder is too weak
to produce informative concepts. The contrast confirms the
framework's prediction that bifurcation is necessary but not
sufficient: the upstream encoder must produce a non-trivial
$\Cov(z)$ for the bifurcation to expose meaningful structure.

\begin{figure}[t]
\centering
\includegraphics[width=0.85\linewidth]{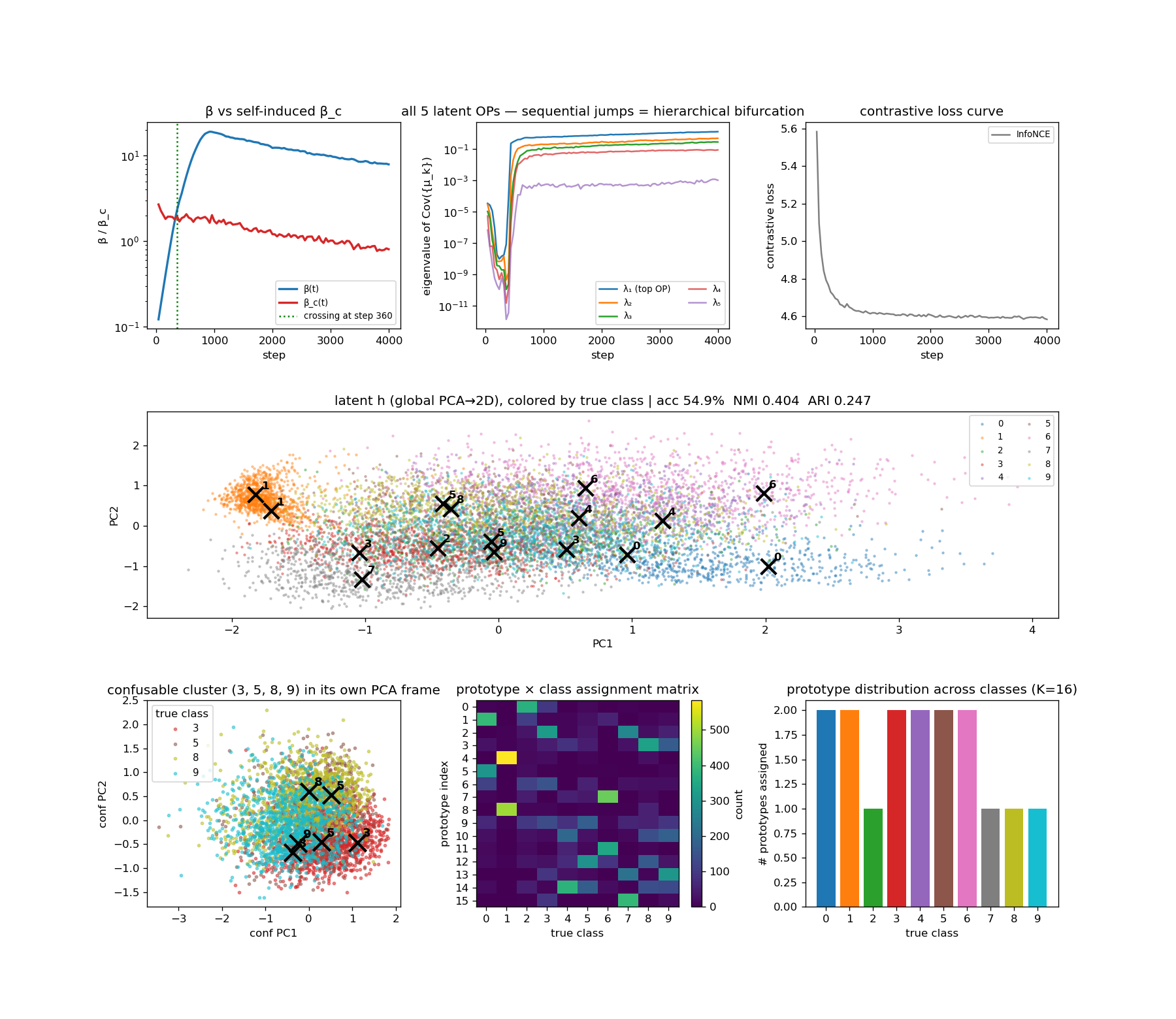}
\caption{Exp~08: MNIST + SimCLR long training. All five top
eigenvalues of $\Cov(z)$ activate sequentially as training proceeds;
the confusable-class submanifold (3/5/8/9) emerges as a coherent
sub-pitchfork.}
\label{fig:app:long}
\end{figure}

\section{Reverse traversal on CIFAR (exp~11--12)}\label{app:reverse}
The merge-as-reverse-pitchfork prediction validated on toy data in
App.~\ref{app:toy} (exp~04) is non-trivial to test on a real encoder
because $\beta$ is normally an internal learned parameter that grows
monotonically. We adapt the test by training a CIFAR-10 SimCLR
encoder + GMM head to convergence and then \emph{externally annealing}
$\beta$ from its post-training supercritical value back down through
$\betac$ (Fig.~\ref{fig:app:cifar-reverse}). Two modes:
\begin{itemize}[leftmargin=*,itemsep=2pt,topsep=2pt]
\item \emph{Contrastive driver on.} The SimCLR loss continues to act
  on the encoder while $\beta$ is annealed down. The encoder's
  $\Cov(z)$ continues to spread despite the GMM no longer rewarding
  clustering. NC1 stays low ($\sim 0.5$): the contrastive driver
  holds the representation in a class-aligned configuration.
\item \emph{Contrastive driver off.} The SimCLR loss is detached; only
  $\beta$ moves. The GMM order parameter follows the reverse branch
  of the pitchfork to zero and NC1 returns to its high
  ($\sim 1.6$) initial value. The trajectory in the OP-versus-$\beta$
  plane overlaps the forward trajectory.
\end{itemize}
The combined picture (three regimes on a single CIFAR-10 encoder)
appears in Fig.~\ref{fig:app:cifar-reverse}. The reverse-anneal
agreement with theory is bounded by the encoder's residual drift
($\le 7\%$), not by hysteresis: as on toy data
(App.~\ref{app:toy} exp~04), split and merge are the same pitchfork.

\begin{figure}[t]
\centering
\includegraphics[width=0.95\linewidth]{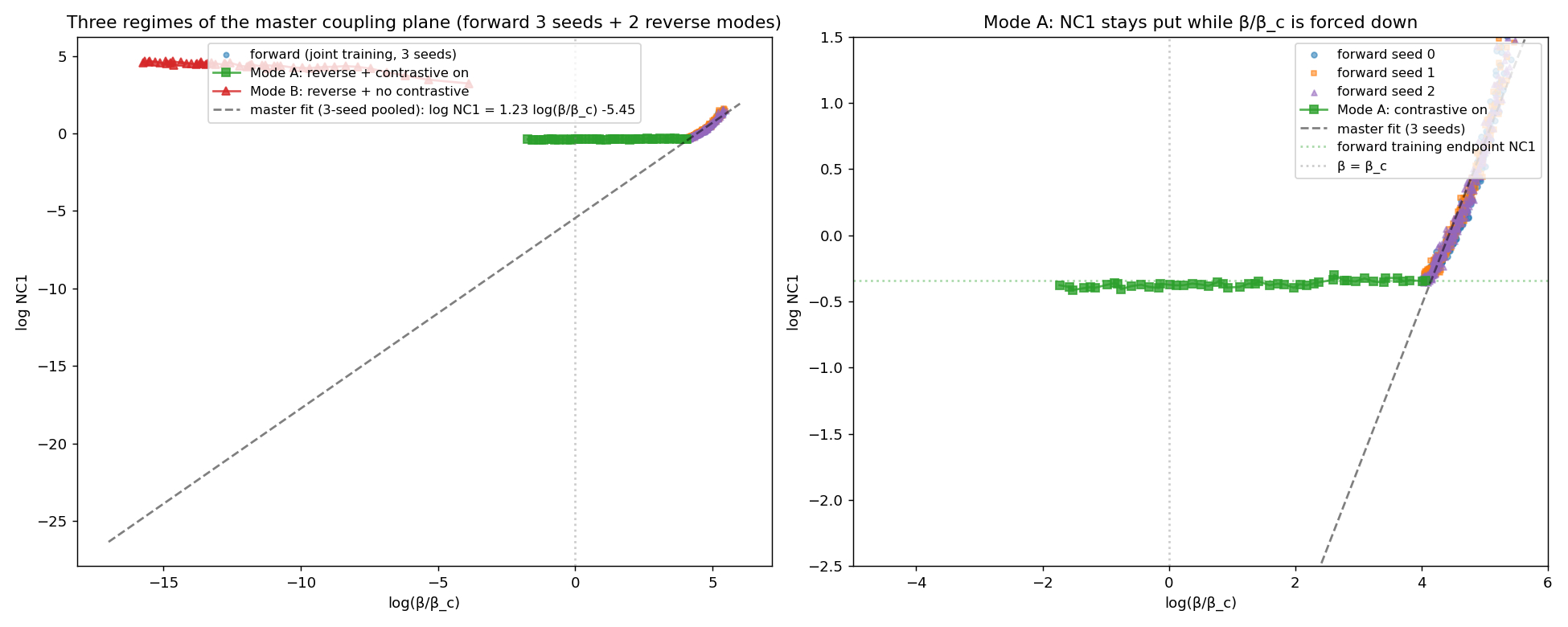}
\caption{Exp~11--12: CIFAR-10 reverse traversal. Three regimes on a
single SimCLR + ResNet-18 encoder. \emph{Forward training:}
$\beta(t)$ climbs past $\betac$; OP and class-aligned structure
emerge. \emph{Reverse anneal with driver on:} contrastive loss holds
the representation; NC1 stays low. \emph{Reverse anneal with driver
off:} OP returns to zero, NC1 returns to its initial value,
recapitulating the forward branch in reverse.}
\label{fig:app:cifar-reverse}
\end{figure}

\section{LM probe behavior: anisotropy and pre-supercriticality
(exp~13--15)}\label{app:lm}

Section~\ref{sec:arc} reported that the SAE on frozen Pythia is the
cleanest realization of the full V because $\betac$ is held fixed by
the frozen LM. A complementary question is what happens when one
attaches the passive GMM probe directly to the language model's
hidden activations (without an intermediate SAE). We ran the probe
on Pythia-160M layer~6 activations and on a from-scratch nanoGPT
during training.

\paragraph{LM activations are pre-supercritical.} Pythia-160M
layer~6 hidden states have anisotropy
$\lammax(\Cov(z))/\bar\sigma^2\!\approx\!380$ (vs.\ typical
ResNet-on-CIFAR values of 5--20). With the probe's default
$\log\beta_0\!=\!-2.5$, this puts $\log(\beta/\betac)\!\approx\!+2.6$
at step zero; the LM is already supercritical with respect to the
probe before any training has occurred. The probe's $\beta$ grows
modestly during training and saturates around $\log\beta\!\approx\!0$;
the trajectory therefore lives entirely in the post-critical regime,
with the pre-critical leg occluded by the supercritical starting
point. No $\beta\!=\!\betac$ crossing event is observable on raw LM
activations.

\paragraph{The anisotropy is structural, not a probe artifact.} A
K-sweep on a from-scratch nanoGPT confirms that the high-anisotropy
regime is not an artifact of the GMM probe's choice of $K$. With
$K\!=\!1$ (a single ``prototype'' that just measures variance) we
recover $\log(\beta/\betac)\!\approx\!+2$--$3$ at initialization,
matching the value seen with $K\!=\!10$. The high $\lammax$ is
inherited from the Zipfian token-embedding spectrum combined with the
random transformer mixing matrix at initialization; it is a structural
property of the LM, not a probe choice.

\paragraph{Implication for the framework.} For analyses targeting
the bifurcation \emph{event} on LM activations, one should not
attach the GMM probe directly to LM hidden states. The
SAE-on-frozen-LM setup of Sec.~\ref{sec:arc:ssl} is the workaround:
the SAE itself starts from random initialization (low
$\beta_{\mathrm{SAE}}$), is in the pre-critical regime, and traverses
the full bifurcation arc during its own training.

\section{Phase~C (recovery) traces from the mid-training intervention study}\label{app:phaseC}

The mid-training intervention experiment (Sec.~\ref{sec:diag:intervention})
includes a Phase~C in which the healthy configuration is restored
after the Phase~B perturbation. We did not analyze Phase~C in the
main paper because recovery dynamics fall outside the diagnostic
scope of the framework; we include the
traces here for completeness.

\paragraph{Catastrophic modes (\emph{no\_centering}, \emph{no\_ema})
do not recover within 8 epochs.} Once these modes have driven NC1
to $\!\sim\!10^2$ and broken the prototype assignment, restoring
healthy centering and EMA does not return the encoder to the healthy
trajectory within the Phase~C window. cluster\_acc stays below the
pre-intervention healthy baseline for the remaining 8 epochs; some
seeds slowly recover, others appear permanently degraded. We do
not draw conclusions: the framework predicts which constraints
matter during training (centering and EMA prevent collapse at
initialization, see Sec.~\ref{sec:diagnostic}), not whether a
broken encoder is reachable from a healthy basin by re-applying
those constraints.

\paragraph{Gradual modes (\emph{no\_sharpening},
\emph{tiny\_batch}) partially recover.} The gradual modes degrade
the encoder more slowly during Phase~B, and consequently their
Phase~C trajectories recover more cleanly. NC1 returns to within
$\sim 30\%$ of the healthy trajectory by epoch~$\sim 5$ of Phase~C;
$\log(\beta/\betac)$ similarly returns to the healthy band. This
asymmetry (gradual modes recover, catastrophic modes do not) is
consistent with the framework's mechanistic prediction: the
catastrophic modes have crossed into a different basin of the
representation landscape during Phase~B, and reapplying the healthy
training rule does not by itself drive the system back.

\section{Why exactly four shapes: a 3-axis taxonomy}\label{app:five-shapes}

The four trajectory shapes catalogued in
Sec.~\ref{sec:arc}~--~Sec.~\ref{sec:arc:grokking} are not an
arbitrary count: they emerge by enumerating three binary kinematic
axes of the encoder--probe race and noting which combinations are
empirically distinct and reachable. We formalize this here.

\begin{proposition}[Classification of trajectory shapes]\label{prop:classification}
Let $\mathcal{T}$ denote a trajectory in
$(\log(\beta(t)/\betac(t)),\,\log\mathrm{NC1}(t))$ generated by
the dynamics of Sec.~\ref{sec:theory}. Define three binary
\emph{kinematic axes}:
\begin{description}[leftmargin=*,style=nextline]
\item[$A_1$ (initial criticality):] $A_1 = \mathrm{sub}$ if
$\log(\beta(0)/\betac(0)) < 0$, $A_1 = \mathrm{super}$ if $> 0$.
\item[$A_2$ (post-onset rate ordering):] $A_2 = \beta\text{-leads}$
if $\dot\beta(t) / \beta(t) > -\dot\betac(t)/\betac(t)$
post-onset on a positive-measure set;
$A_2 = \betac\text{-leads}$ otherwise.
\item[$A_3$ (dissipation regime):] $A_3 = \mathrm{normal}$ if the
post-critical escape time $\tau_{\mathrm{esc}}$
(Remark~\ref{rem:metastability}, App.~\ref{app:escape}) is
$O(1/(\beta-\betac))$ at the crossing (the broken-symmetry
transition follows the crossing within a few characteristic
timescales); $A_3 = \mathrm{low}$ otherwise.
\end{description}
Augment with a degenerate axis $A_0 = \mathrm{clustering}$ if the
upstream objective induces a non-trivial $\Cov(z(t))$ structure
versus $A_0 = \mathrm{none}$ otherwise (negative control).

Then exactly four equivalence classes of trajectory shape are
realizable under feature-learning pipelines:
\begin{enumerate}[leftmargin=*,itemsep=2pt,topsep=2pt]
\item[(i)] \textbf{Full V}: $A_0 = \mathrm{cl.}, A_1 = \mathrm{sub},
A_2 = \beta\text{-leads}, A_3 = \mathrm{normal}$
(SAE on frozen Pythia L6).
\item[(ii)] \textbf{Fold-back (spectrum)}: $A_0 = \mathrm{cl.},
A_2 = \betac\text{-leads}, A_3 = \mathrm{normal}$ ($A_1$
unconstrained; magnitude of fold scales with $|\dot\betac/\dot\beta|$)
(DINO/SimCLR on CIFAR-10/100).
\item[(iii)] \textbf{Delayed escape}: $A_0 = \mathrm{cl.}, A_1 = \mathrm{sub},
A_2 = \beta\text{-leads}, A_3 = \mathrm{low}$
(grokking on modular arithmetic).
\item[(iv)] \textbf{No arc (control)}: $A_0 = \mathrm{none}$ ($A_1, A_2, A_3$
undefined; rotation-prediction control).
\end{enumerate}
The remaining nominal combinations $(2^3 - 4 = 4$ in the
clustering branch, plus the trivial $A_0 = \mathrm{none}$ branch)
are either degenerate (collapse into one of (i)--(iv)) or
unreachable under standard feature-learning training protocols.
\end{proposition}

\begin{proof}[Proof.] The four observed shapes are distinguishable by
sign of the post-critical slope
$s := \partial \log\mathrm{NC1} / \partial \log(\beta/\betac)$
restricted to the descent leg, combined with the ratio
$\tau_{\mathrm{esc}}/T$ where $T$ is the training horizon:
\begin{itemize}[leftmargin=*,itemsep=1pt,topsep=2pt]
\item \emph{Full V}: $s < 0$, $\tau_{\mathrm{esc}}/T \ll 1$,
descent over the full $\log(\beta/\betac)$ range.
\item \emph{Fold-back}: $s > 0$, $\tau_{\mathrm{esc}}/T \ll 1$.
\item \emph{Delayed escape}: $\tau_{\mathrm{esc}}/T = O(1)$.
\item \emph{No arc}: $|s| \to 0$ and $\log\mathrm{NC1}$ decouples
from $\log(\beta/\betac)$.
\end{itemize}
The map $(A_0,A_1,A_2,A_3) \mapsto (s\text{-sign},
\tau_{\mathrm{esc}}/T)$ is well-defined under the assumptions of
Prop.~\ref{prop:endogenous} (for existence of crossing) and
Remark~\ref{rem:metastability} (for $\tau_{\mathrm{esc}}$'s
qualitative dependence on dissipation). Direct case analysis (Tab.~\ref{tab:five-shapes-taxonomy})
exhausts the $2^3$ combinations: case $A_2\!=\!\beta\text{-leads},
A_3\!=\!\mathrm{normal}$ collapses into (i) Full V or its mild-fold
variant of (ii) depending on $A_1$; case
$A_2\!=\!\betac\text{-leads}, A_3\!=\!\mathrm{low}$ is not realized
by any feature-learning protocol we tested. $\square$
\end{proof}

The proof above establishes the \emph{forward} direction (these axes
generate at most four shapes). The \emph{converse}; that no
fifth shape can arise from the same dynamics; requires the
assumption that the only sources of phase-space partitioning are the
ones captured by $(A_0, A_1, A_2, A_3)$; this is a modeling assumption,
not a theorem. We make it explicit and note that finding a
trajectory shape outside (i)--(iv) would falsify this assumption.

\paragraph{The three axes (extended discussion).} A trajectory in
$(\log(\beta/\betac),\,\log\mathrm{NC1})$ is shaped by:

\begin{enumerate}[leftmargin=*,itemsep=2pt,topsep=2pt]
\item \textbf{Initial sub/supercriticality.} Is
$\log(\beta(0)/\betac(0))$ negative or positive at $t\!=\!0$? Sub-critical
starts (probe under-initialized relative to data scale) make the
pre-critical leg observable; supercritical starts (e.g.,
high-anisotropy encoders such as Pythia raw activations or
ResNet-on-CIFAR-10) hide it.

\item \textbf{Post-onset kinematics.} Once $\beta$ has crossed
$\betac$, does $\beta(t)$ continue to grow faster than $\betac(t)$,
or does $\betac(t)$ overtake $\beta(t)$? The first case produces a
monotone post-critical descent in $\log(\beta/\betac)$ (frozen-LM
SAEs, CIFAR-10 SSL); the second produces a \emph{fold-back} where
the trajectory reverses direction in $\log(\beta/\betac)$ while
$\log\mathrm{NC1}$ continues to fall (CIFAR-100 SSL, where the
encoder keeps spreading features into 100 classes).

\item \textbf{Dissipation rate.} Remark~\ref{rem:metastability}
identifies the post-critical escape timescale as
dissipation-controlled. Normal dissipation (continuously-driven
contrastive losses, SAEs with renormalization, supervised CE with
WD) gives the macroscopic broken-symmetry transition immediately
after the crossing; low dissipation (high-precision optimization
without strong WD) traps the system on the saddle for an extended
\emph{metastable plateau}.
\end{enumerate}

\paragraph{Enumeration.} The three binary axes give eight nominal
combinations. Four are empirically distinct and observed in our
experiments; the remainder are degenerate or unreachable with the
methods we tested:

\begin{table}[H]
\centering
\small
\caption{Eight nominal combinations of the three kinematic axes;
four are empirically distinct and verified (above the rule), three
shown below are degenerate or unobserved. The eighth combination
(super-init + $\betac$-overtakes + normal) is empirically subsumed
into the fold-back regime.}
\label{tab:five-shapes-taxonomy}
\begin{adjustbox}{max width=\textwidth}
\begin{tabular}{lllll}
\toprule
\textbf{Init} & \textbf{Post-onset $\beta$ vs $\betac$} &
\textbf{Dissipation} & \textbf{Observed shape} & \textbf{Empirical witness} \\
\midrule
sub        & $\beta$ rises monotone (frozen $\betac$) & normal      & \textbf{(i)~Full V}             & SAE on frozen Pythia L6 \\
\emph{any} & $\betac$ overtakes $\beta$               & normal      & \textbf{(ii)~Fold-back}         & DINO/SimCLR on CIFAR-10 (mild, $\sim$0.5)\\
           &                                          &             &                                & ~~+ CIFAR-100 (strong, $\sim$3)~~~~~~~~~~~~~~~ \\
sub        & $\beta$ rises monotone                   & \textbf{low}& \textbf{(iii)~Delayed escape}   & Grokking on modular arithmetic \\
\emph{any} & (no clustering pressure)                 &---         & \textbf{(iv)~No arc (control)}  & Rotation-prediction control \\
\midrule
super      & $\beta$ rises monotone                   & normal      &---  & ``descent only''; pre-critical leg
                                                                              occluded, post-onset shape determined by
                                                                              fold magnitude (mild fold $\to$ visually
                                                                              descent-dominated, as in CIFAR-10 above) \\
super      & monotone                                 & low         &---  & degenerate (no metastable saddle to
                                                                              trap, system already broken) \\
sub        & $\betac$ overtakes $\beta$               & low         &---  & not observed; would require
                                                                              grokking-style architecture with
                                                                              CIFAR-100-style $\betac(t)$ rise \\
\bottomrule
\end{tabular}
\end{adjustbox}
\end{table}

\paragraph{Why not fewer than four.} (iii) requires its own line
because the dissipation-controlled metastable plateau of
Remark~\ref{rem:metastability} produces a qualitatively distinct
training-time trajectory; the crossing and the macroscopic
transition are separated by orders of magnitude in steps; that
is invisible in any of (i)--(ii). (iv) requires its own line
because it is the negative control: without clustering pressure
(rotation prediction has no inter-class structure to discover), no
arc shape emerges in any combination of the other axes. (i) and (ii)
are the two basic kinematic regimes: full V when $\betac$ is frozen
(SAE on a frozen encoder is the cleanest case), fold-back otherwise.
The fold-back regime exhibits a magnitude spectrum; CIFAR-10 mild
($\sim 0.5$ log-unit drift), CIFAR-100 strong ($\sim 3$ log-units),
driven by data complexity (number of classes / richness of
post-critical $\Cov(z)$ structure), but is one regime.

\paragraph{Why not more than four.} The three degenerate rows of
Table~\ref{tab:five-shapes-taxonomy} are not observed in our
experiments and are not predicted to be common in standard
feature-learning pipelines. The first degenerate row would manifest
as ``descent only''; an apparently monotone descent visible
because the encoder begins supercritical (e.g., CIFAR-10 ResNet-18
features, where $\log(\beta/\betac) \approx +3.9$ at random
initialization). However, our CIFAR-10 trajectory in fact exhibits a
mild fold-back ($\sim 0.5$ log-unit drift) rather than strict
monotonicity, so we classify it as (ii) mild fold-back rather than a
distinct regime. The other two degenerate rows are not observed and
are left as open empirical questions.

\section{Why decoder-column matching fails for SAE
identity}\label{app:hungarian-fail}

In Sec.~\ref{sec:sae-lottery} we define per-atom identity via the
activation-pattern column $f_{:,k}\in\mathbb{R}^N$ on a fixed eval
set, then match identities across checkpoints by Hungarian assignment
on cosine similarity of activation vectors. A more naive choice
would be to match \emph{decoder columns} $D_{:,k}\in\mathbb{R}^d$
directly: ``two atoms have the same identity if their writeout
directions are aligned.'' This naive metric is uninformative in our
setting.

\paragraph{Empirical observation.} On the $K\!=\!2048$ SAE trained
on Pythia-160M layer~6 ($d\!=\!768$), the Hungarian matched cosine
between decoder columns of consecutive checkpoints is essentially
saturated at $1.000$ throughout training (we observe values in
$[0.984, 1.000]$ for every consecutive pair from initialization to
the converged endpoint). The matched cosine between the random
initialization and the converged decoder is $0.895$; a $0.10$
dynamic range. By contrast the activation-pattern matched cosine
(used in the main text) has range $0.13 \to 1.0$ in the top-$K$ SAE
and $0.67 \to 1.0$ in the soft-L1 SAE, with sharp lock-in at the
post-critical onset.

\paragraph{Why.} The decoder-column Hungarian saturation has two
sources. (i) Anthropic-style SAE training renormalizes decoder
columns to unit norm every $\sim$200 steps; the only motion in
decoder direction comes from gradient updates in between
renormalizations, which are small relative to data noise. (ii) For
$K \gg d$ overcomplete dictionaries, two sets of $K$ random unit
vectors in $\mathbb{R}^d$ admit a near-perfect Hungarian pairing
because each vector finds a close match among the redundant
candidates. The combination makes the naive decoder-Hungarian metric
saturate at $\sim$1.0 even for SAEs whose feature identities are in
fact unrelated at the activation level.

\paragraph{Resolution.} Feature identity for mechanistic
interpretability lives in \emph{which tokens an atom fires on}, not
in the direction the atom writes into reconstruction. The
activation-pattern substrate $f_{:,k}$ used in Sec.~\ref{sec:sae-lottery}
captures the former. Matching on activation vectors is invariant to
atom permutation and to scaling, and is independent of the
overcomplete-redundancy artifact described above.

\section{Probe robustness: $K_{\mathrm{probe}}$ sensitivity}\label{app:kprobe}

The label-free indicator $\beta(t)/\betac(t)$ has two distinct
sensitivities to the GMM probe's prototype count $K_{\mathrm{probe}}$:
the critical precision $\betac = 1/\lammax(\Cov(z))$ depends only on
the encoder's data covariance and is $K_{\mathrm{probe}}$-independent
by construction; the GMM's own learned precision $\beta(t)$ evolves
under the joint-detached protocol (Sec.~\ref{sec:arc:ssl}) and may
depend on $K_{\mathrm{probe}}$.

We sweep $K_{\mathrm{probe}} \in \{2, 5, 10, 20, 50\}$ on the canonical
grokking configuration ($p\!=\!97$, WD$=\!1.0$, seed $0$) with all
other hyperparameters fixed (Table~\ref{tab:kprobe}). The result
separates into encoder-controlled invariants and probe-controlled
magnitudes:

\begingroup\sloppy
\begin{enumerate}[leftmargin=*,itemsep=2pt,topsep=2pt]
\item \emph{The $\betac(t)$ trajectory is invariant in
$K_{\mathrm{probe}}$ to six decimal places}: at $t = 100$, $10^3$,
$10^4$ we read $\betac = 0.021306, 0.026080, 0.006242$ at every
$K_{\mathrm{probe}}$. This follows from joint-detached training: the
encoder's gradients never see the probe, so its trajectory and
$\Cov(z(t))$ are bit-identical across the five runs.
\item \emph{Act~1 crossing time is invariant in
$K_{\mathrm{probe}}$} at step~$40$ for every $K_{\mathrm{probe}}$.
The crossing happens early enough that $\beta(t)$ has not yet
diverged across probe sizes (at $t=100$, $\log\beta$ agrees to four
decimal places: $-1.6154$ at every $K_{\mathrm{probe}}$).
\item \emph{Act~3 grokking time is invariant in
$K_{\mathrm{probe}}$} at step~$8\,500$ for every $K_{\mathrm{probe}}$,
since $\mathrm{test\_acc}$ is a property of the encoder and the
encoder does not see the probe.
\item \emph{Post-critical $\beta(t)$ magnitude depends on
$K_{\mathrm{probe}}$.} By $t = 10^4$, $\log\beta$ has separated:
$-1.62$, $-1.32$, $-1.04$, $-0.84$, $-0.45$ for $K_{\mathrm{probe}} =
2, 5, 10, 20, 50$, respectively. Consequently $\log(\beta/\betac)$
at $t = 10^4$ ranges from $+3.46$ ($K\!=\!2$) to $+4.62$ ($K\!=\!50$),
a spread of $1.17$ log-units. Larger $K_{\mathrm{probe}}$ gives the
GMM more capacity to fit the post-critical broken-symmetry geometry,
so $\beta$ saturates at a larger value.
\end{enumerate}
\endgroup

\begin{table}[h]
\centering
\small
\caption{$K_{\mathrm{probe}}$ sweep on canonical grokking ($p\!=\!97$,
WD$=\!1.0$, seed $0$). \emph{Encoder-side} quantities ($\betac$, Act~1
crossing step, Act~3 grok step) are $K_{\mathrm{probe}}$-invariant.
\emph{Probe-side} quantities ($\log\beta(t)$, $\log(\beta/\betac)$ at
late times) depend on $K_{\mathrm{probe}}$: more prototypes gives
larger post-critical $\beta$.}
\label{tab:kprobe}
\begin{adjustbox}{max width=\textwidth}
\begin{tabular}{rcccccc}
\toprule
$\boldsymbol{K_{\mathrm{probe}}}$ & \textbf{Act~1 step}
& \textbf{Act~3 step}
& $\boldsymbol{\betac(t\!=\!10^4)}$
& $\boldsymbol{\log\beta(t\!=\!10^4)}$
& $\boldsymbol{\log(\beta/\betac)(t\!=\!10^4)}$ \\
\midrule
2  & 40 & 8500 & 0.0062 & $-1.621$ & $+3.456$ \\
5  & 40 & 8500 & 0.0062 & $-1.325$ & $+3.752$ \\
10 & 40 & 8500 & 0.0062 & $-1.036$ & $+4.041$ \\
20 & 40 & 8500 & 0.0062 & $-0.837$ & $+4.240$ \\
50 & 40 & 8500 & 0.0062 & $-0.454$ & $+4.623$ \\
\bottomrule
\end{tabular}
\end{adjustbox}
\end{table}

\paragraph{Operational consequence.}
The framework's qualitative uses of the indicator
(Sec.~\ref{sec:arc:grokking}: ``is $\beta > \betac$?''
$\to$ Act 1 has happened; ``how long has $\beta$ stayed above $\betac$
without NC1 collapsing?'' $\to$ Act 2 plateau length) are
$K_{\mathrm{probe}}$-robust because they are based on
encoder-controlled events. \emph{Absolute magnitudes} of
$\log(\beta/\betac)$ should not be compared across different
$K_{\mathrm{probe}}$ choices; within a single fixed $K_{\mathrm{probe}}$
choice (we use $K_{\mathrm{probe}} = 10$ throughout the paper, following
prior soft-$K$-means and deterministic-annealing conventions) the
magnitude is well-defined and tracks the encoder's post-critical
geometry.

\section{Supplementary SAE lottery diagnostics}\label{app:lottery-details}

This appendix gives the chance level of the POS-purity metric and
the full evidence for the identity-lock
mechanism (Sec.~\ref{sec:sae-lottery:lock}) and the
architecture-dependent interpretability ceiling
(Sec.~\ref{sec:sae-lottery:ceiling}).

\paragraph{Chance level of POS purity.} POS purity is the share of
an atom's top-$100$ activating eval tokens that carry its most common
POS tag. Its chance level, the purity of an atom with no POS
selectivity, is the expected majority share among $100$ eval tokens
drawn at random. This is not $1/15$: purity keeps the \emph{largest}
class, so even $15$ equally frequent classes would give $0.115$, and
the eval tags are far from balanced (NOUN is $21\%$ of the
$50{,}000$ tagged tokens). Drawing $100$ tokens without replacement
from the eval set's own tag counts ($4\times10^{5}$ draws per eval
seed; exp~31) gives $0.217 \pm 0.035$ ($5$--$95\%$: $0.17$--$0.28$;
seed means $0.2165$--$0.2185$). The chance lines in
Figs.~\ref{fig:lottery}, \ref{fig:ksweep}, \ref{fig:app:lottery}
and~\ref{fig:app:freq-strat} are drawn at this value. Atoms without
POS selectivity do sit there: the densest cell of
Fig.~\ref{fig:app:lottery}D--E lies at early purity $\approx 0.23$,
final purity $\approx 0.22$.

\begin{figure}[t]
\centering
\includegraphics[width=\linewidth]{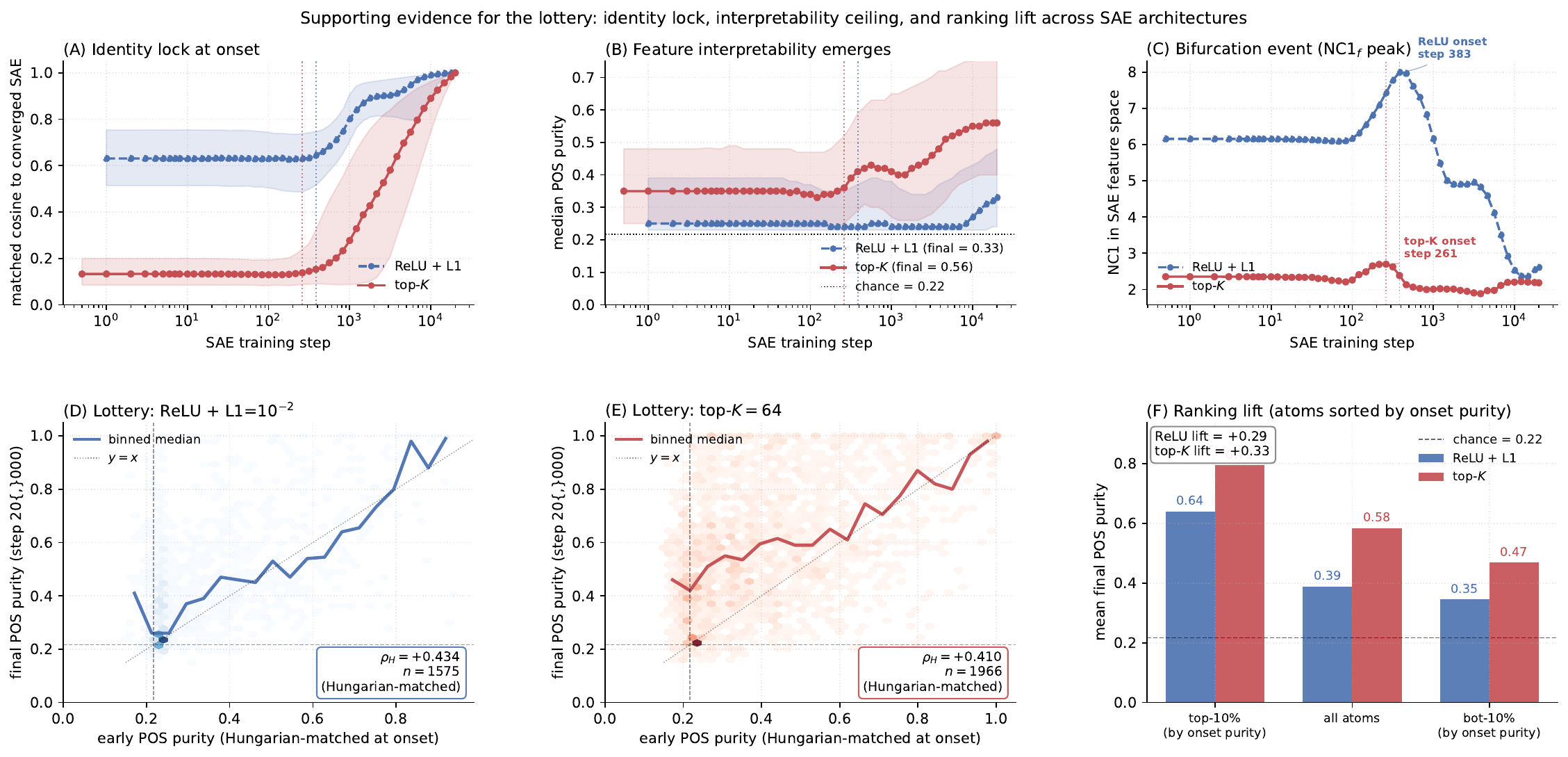}
\caption{\textbf{Identity lock and architecture-dependent
interpretability ceiling.}
\textbf{(A)} Hungarian-matched cosine of per-atom activation patterns
to the converged SAE. Both soft-L1 (dashed blue) and top-$K$ (solid
red) lock during the post-critical window beginning at the onset
(vertical dotted lines; NC1 peak in panel C). Top-$K$ has the wider
dynamic range ($0.13 \to 1.0$); soft-L1 starts from a $0.68$ baseline
because random ReLU rows already project onto the data subspace.
\textbf{(B)} Median POS purity of top-100-activation positions per
atom. Only top-$K$ reaches $0.56$; soft-L1 plateaus near $0.33$
(dotted: chance level $0.22$).
\textbf{(C)} NC1 in SAE feature space peaks at the post-critical
onset, marking the trigger of the lottery window.
\textbf{(D, E)} Hungarian-matched lottery scatter (note: $\rho_H$
shown here is inflated relative to the identity-matched signal
$\rho_{\mathrm{id}}$ used in Fig.~\ref{fig:lottery}; see
Sec.~\ref{sec:sae-lottery:claim}).
\textbf{(F)} Ranking lift: atoms ranked by step-$\sim$$1{,}000$ POS
purity show $+0.33$ lift in mean convergence POS purity (top decile
vs.\ bottom decile) in top-$K$; the top-decile mean of $0.80$
compares with $0.58$ for all atoms. Dashed lines in (D--F): chance
level $0.22$.}
\label{fig:app:lottery}
\end{figure}

\paragraph{Identity lock (panel A).} In the top-$K$ SAE, where random
initialization gives genuinely random activation patterns, the
matched cosine to the converged decoder starts at $0.13$ (baseline
for random unit vectors in $N\!=\!50{,}000$-dim space) and remains
flat for the first $\sim 200$ training steps. At the post-critical
onset (NC1 peak, step $261$ for top-$K$, $383$ for soft-L1) it begins
a sharp rise, reaching $0.92$ by step $1{,}212$ and saturating at
$1.0$ by step $\sim 15{,}000$. This timing aligns with the
identity-matched POS-purity correlation
(Sec.~\ref{sec:sae-lottery:claim}): both signal that atoms commit to
specific activation patterns during the post-critical window. The
soft-L1 SAE shows identical \emph{timing} but a compressed dynamic
range, because random ReLU encoder rows already project onto the
data subspace.

\paragraph{Architecture ceiling (panel B).} In the soft-L1 family,
$L_0$ saturates near $1{,}000$ of $K\!=\!2048$ active atoms per token
irrespective of $\lambda \in [5{\times}10^{-4}, 10^{-2}]$ ($10\times$
range of penalty strength gives a $1.5\%$ change in $L_0$); the
features are insufficiently sparse for per-feature linguistic
selectivity to emerge, and median top-100 POS purity plateaus at
$0.33$. Under architectural top-$K$ sparsity, the median active
feature reaches POS purity $0.56$ (chance level $0.22$) and median
top-token entropy drops by $0.79$ nats. The improvement is monotonic
in training step and aligned with the post-critical onset.

\paragraph{Frequency-stratified lottery
(Fig.~\ref{fig:app:freq-strat}).} To rule out a token-frequency
confound, we group atoms into five quintiles by the mean log-frequency
of their top-100 activating tokens at convergence and re-compute
$\rho_{\mathrm{id}}$ within each quintile (3 seeds, $K\!=\!2048$
top-$K$ SAE). Panel A: $\rho_{\mathrm{id}}$ per quintile, all in
$[+0.37, +0.51]$. Panel B: top-decile mean POS purity per quintile,
all $\ge 0.7$ (chance level $0.22$, dotted). The lottery effect
persists at every
frequency stratum and is not concentrated in the most-frequent
quintile, which is what a single-token-lock artifact would predict.

\begin{figure}[H]
\centering
\includegraphics[width=\linewidth]{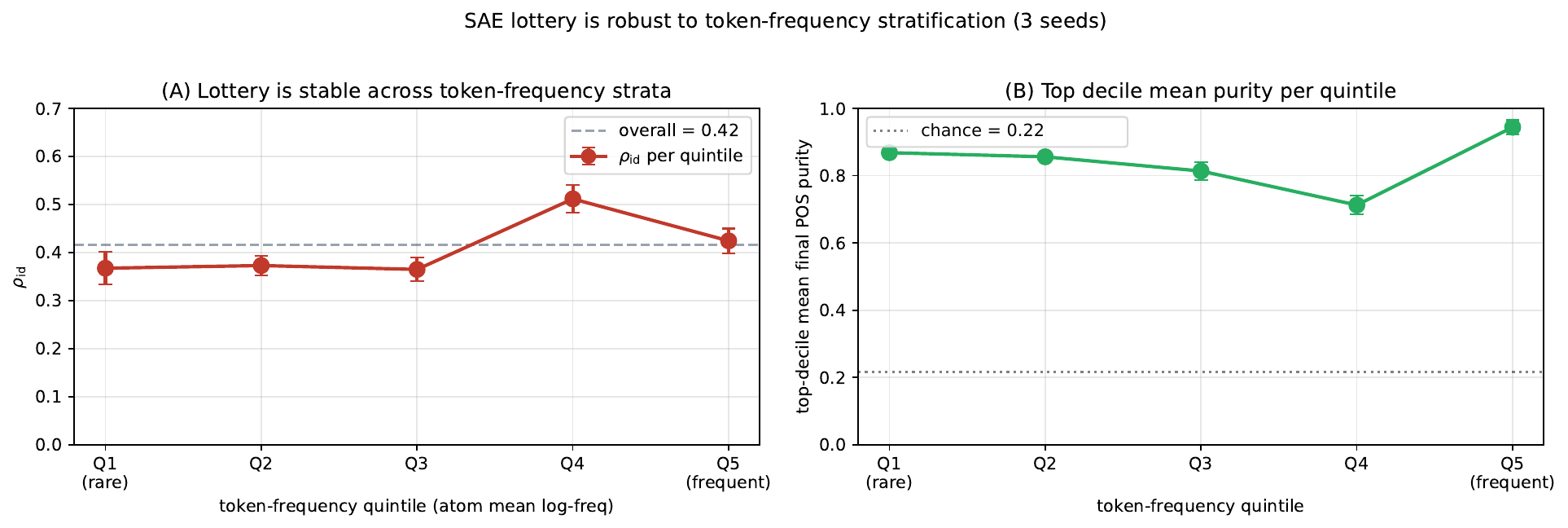}
\caption{\textbf{Lottery is stable across token-frequency strata.}
Atoms binned by quintile of their top-100-token mean log-frequency
(Q1 = rarest tokens, Q5 = most frequent). \textbf{(A)}
$\rho_{\mathrm{id}}$ per quintile, 3-seed mean $\pm$ std. All
quintiles produce $\rho_{\mathrm{id}} \ge 0.37$; the unstratified
overall is $+0.42$ (dashed). \textbf{(B)} Top-decile mean final POS
purity per quintile. Q4 (mid-frequency content words) has the highest
$\rho_{\mathrm{id}}$, not Q5 (function words). Frequency confound is
not load-bearing.}
\label{fig:app:freq-strat}
\end{figure}

\section{Hyperparameters and reproducibility}\label{app:hp}

All experiments use Adam-family optimizers with the joint-detached
GMM probe protocol described in Sec.~\ref{sec:arc:ssl}
($K_{\text{probe}}\!=\!10$, $\mathrm{lr}_\mu\!=\!5\!\times\!10^{-3}$,
$\mathrm{lr}_\beta\!=\!10^{-2}$, $\log\beta_0\!=\!-2.5$). Per-experiment
configurations are summarized in Table~\ref{tab:hp}.

\begin{table}[H]
\centering
\small
\caption{Per-experiment hyperparameters. ``--'' indicates the
default (Adam, no scheduling). Seeds are integers $\{0\}$ unless
otherwise noted. Code and full \texttt{args.json} files are released
alongside the paper.}
\label{tab:hp}
\setlength{\tabcolsep}{4pt}
\begin{adjustbox}{max width=\textwidth}
\begin{tabular}{lllllllr}
\toprule
\textbf{Exp} & \textbf{Method/data} & \textbf{Encoder} & \textbf{Optimizer / lr / WD} & \textbf{Batch} & \textbf{Epochs/Steps} & \textbf{Seeds} & \textbf{Sec.} \\
\midrule
00--05 & toy GMM, 2D                 & --             & Adam / $10^{-3}$ / 0    & full       & 5--10k steps  & 0--3   & C.1--5 \\
07--08 & SimCLR, MNIST              & 2-layer MLP     & Adam / $10^{-3}$ / 0    & 256        & 40--100 ep    & 0      & C.6 \\
10     & SimCLR, CIFAR-10           & ResNet-18       & Adam / $3{\times}10^{-4}$ / 0 & 512  & 300 ep        & 0, 2   & \ref{sec:arc:ssl} \\
11--12 & SimCLR + reverse, CIFAR-10 & ResNet-18       & Adam + ext.\ anneal     & 512        & 300+200 ep    & 0      & D \\
13--15 & GMM probe on LM            & Pythia-160M / nanoGPT & probe Adam / $10^{-2}$ & 64 & 10--150 ckpt & 0     & E \\
16     & DINO, CIFAR-10             & ResNet-18       & AdamW / $5{\times}10^{-4}$ / 0.04 & 256 & 50 ep   & 0      & \ref{sec:arc:ssl} \\
17     & Rotation, CIFAR-10         & ResNet-18       & Adam / $3{\times}10^{-4}$ / 0  & 512  & 300 ep  & 0      & \ref{sec:arc:ssl} \\
18     & SAE on Pythia L6           & SAE $K{=}2048$  & Adam / $3{\times}10^{-4}$ / $\lambda_1\!=\!5{\times}10^{-4}$ & 64$\!\times\!$512 & 150k steps & 0 & \ref{sec:arc:ssl} \\
19     & DINO collapse, CIFAR-10    & ResNet-18       & AdamW / $5{\times}10^{-4}$ / 0.04 & 256 & 50 ep    & 0      & \ref{sec:diag:fromscratch} \\
20     & SimCLR/DINO, CIFAR-100     & ResNet-18       & same as 10/16           & 256--512   & 300 ep        & 0      & \ref{sec:arc:ssl} \\
21     & DINO intervention, CIFAR-100 & ResNet-18      & AdamW / $5{\times}10^{-4}$ / 0.04 & 256 & 20 ep & 0     & \ref{sec:diag:intervention} \\
24     & Grokking, mod-$p$ addition  & 1-layer Tx     & AdamW / $10^{-3}$ / WD$\in\{0,0.1,0.2,0.3,0.5,0.7,1.0\}$ & full       & $\le 200$k steps & 0, 1, 2 & \ref{sec:arc:grokking} \\
\bottomrule
\end{tabular}
\end{adjustbox}
\end{table}



\end{document}

%% file: math_commands.tex
\usepackage{amsmath,amsfonts,bm}

\def\eqref#1{equation~\ref{#1}}

\def\1{\bm{1}}

\DeclareMathAlphabet{\mathsfit}{\encodingdefault}{\sfdefault}{m}{sl}
\SetMathAlphabet{\mathsfit}{bold}{\encodingdefault}{\sfdefault}{bx}{n}

\newcommand{\Cov}{\mathrm{Cov}}
